\documentclass[11pt]{article}
\usepackage[dvipsnames]{xcolor}

\usepackage{amsthm,amsmath,color,natbib, tikz}

\usepackage{setspace}
\usepackage{graphicx}
\usepackage{tabularx,booktabs}
\usepackage{xcolor}
\usepackage[bottom]{footmisc}

\newtheorem{theorem}{Theorem}
\newtheorem{example}{Example}

\newtheorem{remark}{Remark}
\newtheorem{lemma}{Lemma}
\newtheorem{proposition}{Proposition}
\newtheorem{Corollary}{Corollary}

\newtheorem{simulation}{Simulation}
\newtheorem{innercustomthm}{Assumption}

\newenvironment{customassumption}[1]
  {\renewcommand\theinnercustomthm{#1}\innercustomthm}
  {\endinnercustomthm}

\newcommand{\p}{{\rm I}\kern-0.18em{\rm P}}
\newcommand{\1}{{\rm 1}\kern-0.24em{\rm I}}
\newcommand{\E}{{\rm I}\kern-0.18em{\rm E}}
\newcommand{\R}{{\rm I}\kern-0.18em{\rm R}}
\newcommand\numberthis{\addtocounter{equation}{1}\tag{\theequation}}

\newcommand{\scolor}[1]{\textcolor{red}{#1}}

\DeclareMathOperator*{\argmin}{arg\,min}

\usepackage{amsthm,amsmath,amsfonts,amssymb,natbib,mathtools,mathrsfs,framed,multirow}
\usepackage[noend]{algpseudocode}
\usepackage{ccaption}
\usepackage{longtable}
\usepackage{appendix}

\usepackage[ruled,linesnumbered]{algorithm2e}

\usepackage{enumerate}
\usepackage{verbatim}
\usepackage{subfigure}
\usepackage{bbm}
\usepackage{threeparttable}
\usepackage[colorlinks,linkcolor=blue,citecolor=blue,urlcolor=blue]{hyperref}
\usepackage{xr}
\usepackage{multirow}
\usepackage{graphicx}

\usepackage{JASA_manu}
\usepackage{authblk}

\makeatother

\title{\Large \textsc{\bf{Asymmetric error control under imperfect supervision: a label-noise-adjusted Neyman-Pearson umbrella algorithm }}}

\begin{document}

%\author[1]{}
\author[1]{Shunan Yao}
\author[2]{Bradley Rava}
\author[2,3]{Xin Tong}
\author[2,3]{Gareth James}
\affil[1]{\footnotesize Department of Mathematics, Dana and David Dornsife College of Letters, Arts and Sciences, University of Southern California.}
\affil[2]{\footnotesize Department of Data Sciences and Operations, Marshall School of Business, University of Southern California. }
\affil[3]{To whom correspondence should be addressed. xint@marshall.usc.edu, gareth.james@marshall.usc.edu}

\maketitle

\begin{abstract}
    Label noise in data has long been an important problem in supervised learning applications as it affects the effectiveness of many widely used classification methods. Recently, important real-world applications, such as medical diagnosis and cybersecurity, have generated renewed interest in 
    %call for solutions to address label noise problems under 
    the Neyman-Pearson (NP) classification paradigm, which constrains the more severe type of error (e.g., the type I error) under a preferred level while minimizing the other (e.g., the type II error). However, there has been little research on the NP paradigm under label noise. It is somewhat surprising that even when common NP classifiers ignore the label noise in the training stage, they are still able to control the type I error with high probability. However, the price they pay is excessive conservativeness of the type I error and hence a significant drop in power (i.e., $1 - $ type II error).  Assuming that domain experts provide lower bounds on the corruption severity, we propose the first theory-backed algorithm that adapts most state-of-the-art classification methods to the training label noise under the NP paradigm. The resulting classifiers not only control the type I error with high probability under the desired level but also improve power. 
    
    
\end{abstract}

%\newpage 
\section{Introduction}

%\textcolor{red}{Shunan, this is my overall suggestion: as the paper spans more than 50 pages, do not edit a place only locally.  As the first author of this paper, you should frequently read through the whole thing. Then your edits will be more coherent. Also, at this point, you should be read again reviewers questions and make sure that the updates have responded to their concerns, even through some responses haven't been filled out yet. Moreover, please read all my comments first before you address any particular problem. In this way, you have a feeling about how much work is needed overall. I understand that the real data might further consume sometime, but you can work on other things while the code is running.  }

Most classification methods assume a perfectly labeled training dataset. Yet, it is estimated that in real-world databases around five percent of labels are incorrect \citep{Orr1998, Reman1998}.  Labeling errors might come from insufficient guidance to human coders, poor data quality, or human mistakes in decisions, among others \citep{Brazdil.Konolige.1990, hickey1996noise, Brodley.Fridel.1999}. Specifically, in the medical field, a 2011 survey of more than $6{,}000$ physicians found that half said they encountered diagnostic errors at least once a month \citep{MacDonald2011}. The existence of labeling errors in training data is often referred to as \textit{label noise}, \textit{imperfect labels} or \textit{imperfect supervision}. It belongs to a more general \textit{data corruption} problem, which refers to ``anything which obscures the relationship between description and class" \citep{hickey1996noise}.

The study of label noise in supervised learning has been a vibrant field in academia. \textit{On the empirical front}, researchers have found that  some statistical learning methods such as quadratic discriminant analysis \citep{lachenbruch1979note} and k-NN \citep{okamoto1997average}, can be greatly affected by label noise and have accuracy seriously reduced, while other methods, such as linear discriminant analysis \citep{lachenbruch1966discriminant}, are more label noise tolerant. Moreover, one can modify AdaBoost \citep{cao2012noise}, perceptron algorithm \citep{khardon2007noise} and neural networks \citep{sukhbaatar2014learning}, so that they are more tolerant to  label noise. Data cleansing techniques were also developed, such as in \cite{guyon1996discovering} and \cite{brodley1999identifying}. \textit{On the theoretical front}, \cite{natarajan2013learning} provided a guarantee for risk minimization in the setting of convex surrogates. \cite{manwani2013noise} proved label noise tolerance of risk minimization for certain types of loss functions, and \cite{ghosh2015making} extended the result by considering more loss types. \cite{liu2015classification} proposed learning methods with importance-reweighting which can minimize the risk.  \cite{blanchard_flaska_handy_pozzi_scott_2016} studied intensely the \textit{class-conditional corruption model}, a model that many works on label noise are based on. In particular, theoretical results about parameter estimation and consistency of classifiers under this model were presented in their work. Most recently, \cite{cannings2020classification} derived innovative theory of excess risk for general classifiers.

%\textcolor{red}{As requested by the reviewers, some more details on the literature."more details about what is known, and what approaches are taken in these works would better set the contributions of this paper into context"} \textcolor{purple}{The reviewers may questions are about the theory papers part.  So we just need to add summary to these ones.  But after numerical section is done.  This part will need to improve. (1) Natarajan 2013 employs a CCN model which is different from ours. And Their theory assumes known rates.  (2)   }

In many classification settings, one type of error may have far worse consequences than the other. For example, a biomedical diagnosis/prognosis that misidentifies a benign tumor as malignant will cause distress and potentially unnecessary medical procedures, but the alternative, where a malignant tumor is classified as benign, will have far worse outcomes. Other related predictive applications include cybersecurity and finance.  Despite great advances in the label-noise classification literature, to our knowledge, no classifier has been constructed to deal with this asymmetry in error importance under label noise so as to control the level of the more severe error type.
%controlled on the population level at some desirable target.}     %these methods are designed to minimize the overall error rate. However, this approach implicitly assumes that all errors are equally bad, 

In this paper, we concentrate on the classification setting involving both mislabeled outcomes and error importance asymmetry. The Neyman-Pearson (NP) paradigm \citep{cannon2002learning, scott2005neyman}, which controls the false-negative rate (FNR, a.k.a., type I error\footnote{Note that type I error in our work is defined to be the conditional probability of misclassifying a 0 instance as class 1. Moreover, we code the more severe class as class 0. In the disease diagnosis example, the disease class would be class 0. }) under some desired level while minimizing the false-positive rate (FPR, a.k.a., type II error), provides a natural approach to this problem. However, to the best of our knowledge, there has been no work that studies how label noise issues affect the control of the more severe FNR.
%false-negative rate (FNR, a.k.a., type I error\footnote{Note that type I error in our work is defined to be the conditional probability of misclassifying a 0 instance as class 1. Moreover, we code the more severe class as class 0. In the disease diagnosis example, the disease class would be class 0. }) in biomedical diagnosis/prognosis.  Given prevalent inaccuracies of historical diagnosis in medical records and the importance of FNR control, we are motivated to design algorithms that train on corrupted labels but correctly control the population-level FNR (regarding the true labels) with high probability. Critical predictive applications in other fields, such as cybersecurity and finance, also call for such a solution.  
%
%To construct classifiers that controls the FNR, the Neyman-Pearson (NP) paradigm \citep{cannon2002learning, scott2005neyman}, which controls FNR under some desired level while minimizing false-positive rate (FPR, a.k.a., type II error), is the natural formulation. 
We show that if one trains a standard NP classifier on corrupted labels (e.g., the NP umbrella algorithm \citep{tong2018neyman}), then the actual achieved FNR is far below the control target, resulting in a very high, and undesirable, FPR. 

%This problem motivates us to devise a new label-noise-adjusted umbrella algorithm that corrects for the labeling errors to produce a lower FPR while still controlling the FNR. The construction of such an algorithm is challenging because we must identify the optimal correction level without any training data from the uncorrupted distribution. To address this challenge, we employ a common class-conditional noise model and derive the population-level difference between type I error regarding the true labels and that regarding the corrupted ones. Based on this difference, we propose a sample-based correction term that, even without observing any uncorrupted labels, can correctly adjust the NP umbrella algorithm to significantly reduce the FPR while still controlling the FNR.  

%%

This problem motivates us to devise a new label-noise-adjusted umbrella algorithm that corrects for the labeling errors to produce a lower FPR while still controlling the FNR. The construction of such an algorithm is challenging because we must identify the optimal correction level without any training data from the uncorrupted distribution. To address this challenge, we employ a common class-conditional noise model and derive the population-level difference between the type I errors of the true and corrupted labels. Based on this difference, we propose a sample-based correction term that, even without observing any uncorrupted labels, can correctly adjust the NP umbrella algorithm to significantly reduce the FPR while still controlling the FNR.  

Our approach has several advantages. First, it is the first theory-backed methodology in the label noise setting to control population-level type I error (i.e., FNR) regarding the true labels. Concretely, we can show analytically that the new algorithm produces classifiers that have a high probability of controlling the FNR below the desired threshold with a FPR lower than that provided by the original NP umbrella algorithm.    Second, when there are no labeling errors, our new algorithm reduces to the original NP algorithm.  Finally, we demonstrate on both simulated and real-world data, that under the NP paradigm the new algorithm dominates the original unadjusted one and competes favorably against existing methods which handle label noise in classification.

The rest of the paper is organized as follows. In Section \ref{sec:model}, we introduce some notation and a corruption model to study the label noise. In Section \ref{sec:method}, we demonstrate the ineffectiveness of the original NP umbrella algorithm under label noise and propose a new label-noise-adjusted version.  The validity and the high-probability type I error control property of the new algorithm are established in Section \ref{sec:theory}. Simulation and real data analysis are conducted in Section \ref{sec:sim_and_real_data}, followed by a Discussion section.  All proofs, additional numerical results, and technical results are relegated to the Appendix. %\textcolor{red}{After checking Section 3, come back and check this paragraph.}

%\textcolor{red}{Will come back to this paragraph and see whether the wording ``the new algorithm" is good enough. Or we need to make it more complicated and claim two versions. Probably argue that Algorithm 1 is a special case of Algorithm 2 is a way to go.} 

%\textcolor{purple}{change all literature review to past tense. }

\section{Notation and Corruption Model}\label{sec:model}

Let $(X,Y,\tilde{Y})$ be a random triplet, where $X \in \mathcal{X} \subset \R^d$ represents features, $Y \in \{0,1\}$ encodes the true class labels and $\tilde{Y} \in \{0,1\}$ the corrupted ones. Note that in our setting, we cannot observe $Y$; the observations come from $(X,\tilde{Y})$.    Denote $X^0 \triangleq X \mid (Y=0)$ and $X^1\triangleq X \mid (Y=1)$. Similarly, denote $\tilde{X}^0 \triangleq X \mid (\tilde{Y}=0)$ and $\tilde{X}^1\triangleq X \mid (\tilde{Y}=1)$. Denote by $\p$ and $\E$ generic probability measure and expectation whose meanings depend on the context.   For any Borel set $A\subset \mathcal{X}$, we denote
\[
\begin{aligned}
P_0(A)&= \p(X \in A|Y=0)\,,\text{ } P_1(A)= \p(X \in A|Y=1)\,, \\
\tilde{P}_0(A)&= \p(X \in A|\tilde{Y}=0)\,,\text{ } \tilde{P}_1(A)= \p(X \in A|\tilde{Y}=1)\,. \\
\end{aligned}
\]
Then, we denote by $F_0\,, F_1\,, \tilde{F}_0$ and $\tilde{F}_1$ their respective distribution functions and by $f_0\,,f_1\,,\tilde{f}_0$ and $\tilde{f}_1$ the density functions, assuming they exist. Moreover, for a measurable function $T : \mathcal{X} \rightarrow \R$, we denote, for any $z \in \R$,
\[
\begin{aligned}
F^{T}_0(z) &= P_0(T(X) \leq z)\,,\text{ }
F^{T}_1(z) = P_1(T(X) \leq z)\,, \\
\tilde{F}^{T}_0(z) &=\tilde{P}_0(T(X) \leq z)\,,\text{ }
\tilde{F}^{T}_1(z)= \tilde{P}_1(T(X) \leq z)\,. \\
\end{aligned}
\]

Since the effect of, and adjustment to, the label noise depend on the type and severity of corruption, we need to specify a corruption model to work with. Our choice for this work is the \textit{class-conditional noise (contamination) model}, which is specified in the next assumption.

\begin{customassumption}{1}\label{assumption:mixture}
There exist constants $m_0, m_1 \in [0,1]$ such that for any Borel set $A\subset \mathcal{X}$,
\begin{align}\label{eqn:model}
    \tilde{P}_0(A) = m_0P_0(A) + (1-m_0)P_1(A)\, \text{ }\text{ and }\text{ }
    \tilde{P}_1(A) = m_1P_0(A) + (1-m_1)P_1(A)\,.
\end{align}
Furthermore, assume $m_0 > m_1$ but both quantities can be unknown. Moreover, let $m_0^\#, m_1^\# \in [0,1]$ be known constants such that $m_0^\# \geq m_0$ and $m_1^\#  \leq m_1$.
\end{customassumption}

%The following is a canonical example for Assumption \ref{assumption:mixture}:
\begin{example}\label{ex: gmm}[An example of Assumption \ref{assumption:mixture}]
Let $X^0 \sim \mathcal{N}(\mu_0, \sigma^2)$ and $X^1 \sim \mathcal{N}(\mu_1, \sigma^2)$, where $\mu_0, \mu_1 \in \R$ and $\sigma > 0$. Then $\tilde{F}_0(z) = m_0\Phi(\frac{z-\mu_0}{\sigma}) + (1-m_0)\Phi(\frac{z-\mu_1}{\sigma})$ and $\tilde{F}_1(z) = m_1\Phi(\frac{z-\mu_0}{\sigma}) + (1-m_1)\Phi(\frac{z-\mu_1}{\sigma})$, where $\Phi(\cdot)$ is the distribution function of $\mathcal{N}(0,1)$. With the choice of $\mu_0 = 0$, $\mu_1 = 1$, $\sigma = 1$, $m_0 = 0.9$, and $m_1 = 0.05$, the density functions $f_0$, $\tilde f_0$, $f_1$ and $\tilde f_1$ are plotted in Figure \ref{fig:mixture_density}.
\end{example}

\begin{figure}
\begin{center}
    \includegraphics[width = \textwidth, height = 7cm]{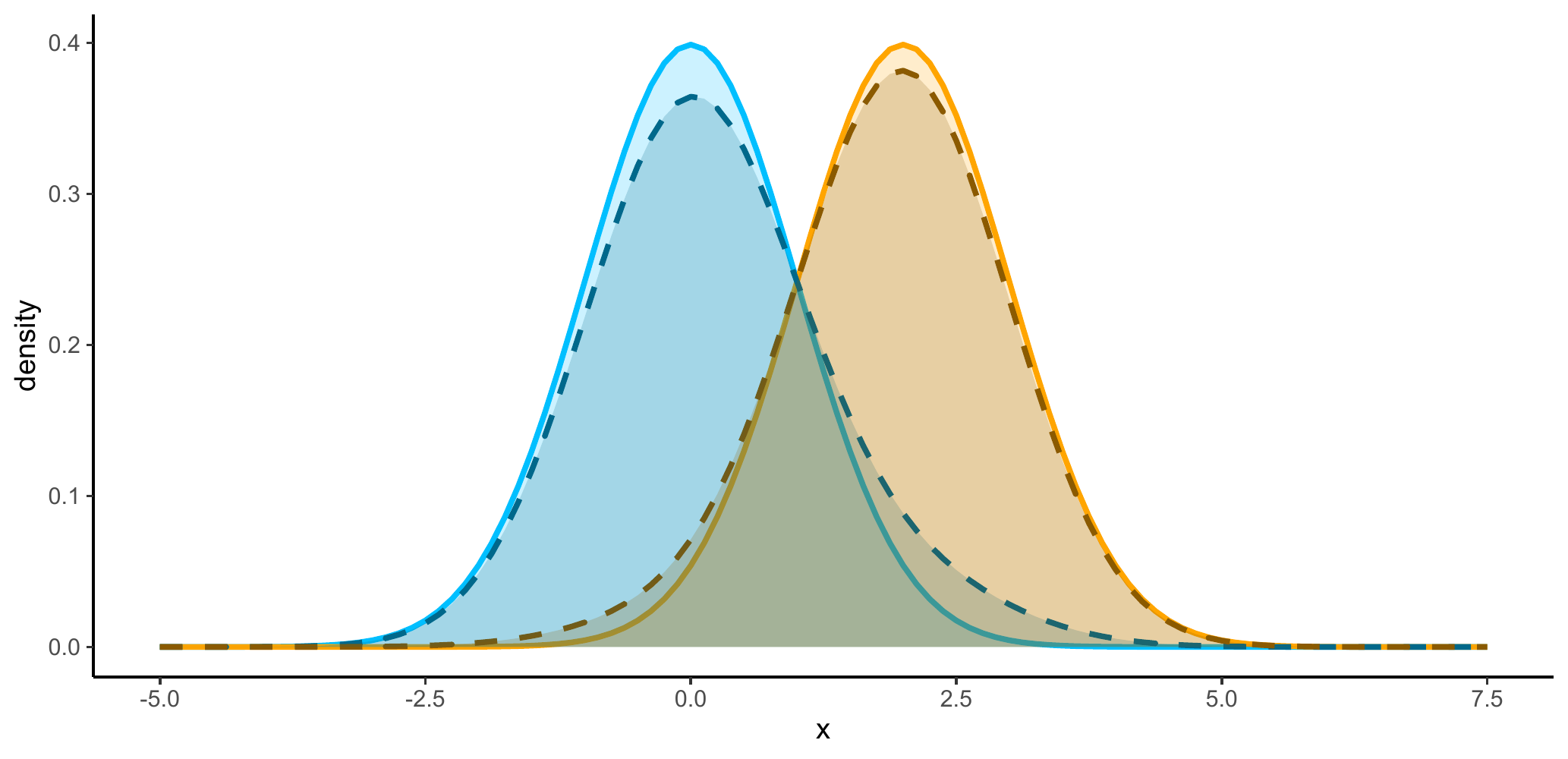}
    \caption{Density plots in Example \ref{ex: gmm}. True (lighter and solid) and corrupted (darker and dashed). }
    \label{fig:mixture_density}
\end{center}
\end{figure}

Note that equation \eqref{eqn:model} specifies perhaps the simplest model for label noise in supervised learning. Here, $m_0$ and $m_1$ represent the severity of corruption levels. Concretely, $m_0$ can be interpreted as the proportion of true $0$ observations among corrupted $0$ observations, and $m_1$  the proportion of true $0$ observations among corrupted $1$ observations. The assumption $m_0 > m_1$  means that corrupted class $0$ resembles true class $0$ more than corrupted class $1$ does, and that  corrupted class $1$ resembles true class $1$ more than corrupted class $0$ does.  However, this assumption does not mean that corrupted class $0$ resembles true class $0$ more than it resembles true class $1$ (i.e, $m_0>1/2$) or that corrupted class $1$ resembles true class $1$ more than it resembles true class $0$ (i.e., $m_1 < 1/2$). Note that by the way our model is written, $m_0=1$ and $m_1=0$ correspond to the no label noise situation; as such, the roles of $m_0$ and $m_1$ are not symmetric. Hence, the assumptions  $m_0^\# \geq m_0$ and $m_1^\#  \leq m_1$ mean that we know some \textit{lower bounds} of the corruption levels.

The class-conditional label noise model  has been widely adopted in the literature  \citep{ natarajan2013learning, liu2015classification,blanchard_flaska_handy_pozzi_scott_2016}.  We note here that the assumption $m_0 >m_1$ aligns with the \textit{total noise assumption} $\pi_0 + \pi_1 < 1$ in \cite{blanchard_flaska_handy_pozzi_scott_2016} as $\pi_0$ and $\pi_1$ in their work correspond to $1-m_0$ and $m_1$ in Assumption 1, respectively. In \cite{natarajan2013learning} and \cite{liu2015classification}, the label noise was modeled through the label flipping probabilities: $\mu_i = \p(\tilde Y = 1-i|Y=i)$, $i = 0, 1$. This alternative formulation is related to our formulation via Bayes' rule.
An in-depth study of the class-conditional label noise model, including mutual irreducibility and identifiability, was presented in \cite{blanchard_flaska_handy_pozzi_scott_2016}. Moreover, \cite{blanchard_flaska_handy_pozzi_scott_2016} developed a noisy label trained classifier based on weighted cost-sensitive surrogate loss and established its consistency. Similarly, \cite{natarajan2013learning} provided two methods to train classifiers, both relying on classification-calibrated surrogate loss; bounds for respective excess risks of these two methods were also given. % and it was shown that both methods achieve risk minimization. 
Moreover, \cite{liu2015classification} proposed an importance reweighting method and extended the result in \cite{natarajan2013learning} to all surrogate losses. Other than \cite{blanchard_flaska_handy_pozzi_scott_2016}, which briefly discussed the NP paradigm at the population level, in all aforementioned papers, though loss functions vary, the goal of classification is to minimize the overall risk. Our work focuses on the NP paradigm. Moreover, we focus on high probability control on the type I error based on finite samples,  in contrast to asymptotic results in the literature.

%\textcolor{purple}{point to discuss: from flip rates to corruption levels; the details of derivation. a system of linear equations with two variables?}

%\textcolor{red}{Shunan, this paragraph still needs to be updated. Also need to talk about Natarajan and Blanchard papers here. Also, I feel that the current sentence to summarize Liu and Tao is not the best way. Also, should add something to contrast their missions with ours}

In this work, we take the perspective that the domain experts can provide under-estimates of corruption levels.  In the literature, there are existing methods to estimate these levels.  For example, \cite{liu2015classification} and \cite{blanchard_flaska_handy_pozzi_scott_2016} developed methods to estimate $\pi_i$'s and $\mu_i$'s, and showed consistency of their estimators.  In numerical studies, we apply the method in \cite{liu2015classification} to estimate $m_0$ and $m_1$\footnote{Note that though their method targets at $\mu_i$'s, estimates of $m_i$'s in equation \eqref{eqn:model} can be constructed from those of $\mu_i$'s by the Bayes' theorem.}. % with the knowledge of $\p(\tilde Y = 0)$ and $\p(\tilde Y = 1)$ which can be approximated by empirical proportions of corrupted classes. 
Numerical evidence shows that using these estimators in our proposed algorithm fails to establish a high probability control of the true type I error. In fact, even using consistent and unbiased estimators of $m_0$ and $m_1$ as inputs of our proposed algorithm would not be able to control the true type I error with high probability. One such case is  demonstrated in Simulation \ref{sim:normal_estimator} of the Appendix,  where estimators for $m_0$ and $m_1$ are normally distributed and centered at the true values. To have high probability control on the true type I error, we do need the ``under-estimates'' of corruption levels as in Assumption \ref{assumption:mixture}.

\section{Methodology}\label{sec:method}

In this section, we first formally introduce the Neyman-Pearson (NP) classification paradigm and review the NP umbrella algorithm \citep{tong2018neyman} for the uncorrupted label scenario (Section \ref{sec:3.1}).  Then we provide an example demonstrating that in the presence of label noise, naively implementing the NP umbrella algorithm leads to excessively conservative type I error. i.e., type I error much smaller than the control target $\alpha$. We analyze and capitalize on this phenomenon, and present new  noise-adjusted versions of the NP umbrella algorithm, Algorithm \ref{alg:adj_umbrella} for known corruption levels  (Section \ref{sec:adj_alg}) and Algorithm $1^{\#}$ for unknown corruption levels  (Section \ref{sec:adj_alg_unknown}). Algorithm \ref{alg:adj_umbrella} can be considered as a special case of Algorithm $1^{\#}$: $m_0^{\#}=m_0$ and $m_1^{\#}=m_1$. %, but a separate introduction of the former is easier to present our algorithm development.    %\textcolor{red}{edit the previous sentence after finishing the whole section.}
%\textcolor{red}{change the first paragraph as we have Section 3.3 now.}

A few additional notations are introduced to facilitate our discussion.  A classifier $\phi: \mathcal{X}\rightarrow\{0,1\}$ maps from the feature space to the label space. The (population-level) type I and II errors of $\phi(\cdot)$ regarding the \textit{true} labels (a.k.a., true type I and II errors) are respectively  $R_0(\phi) = P_0(\phi(X) \neq Y)$ and $R_1(\phi)= P_1(\phi(X) \neq Y)$.  The (population-level) type I and II errors of $\phi(\cdot)$ regarding the \textit{corrupted} labels (a.k.a., corrupted type I and II errors) are respectively $\tilde{R}_0(\phi) = \tilde{P}_0(\phi(X) \neq \tilde{Y})$ and $\tilde{R}_1(\phi) = \tilde{P}_1(\phi(X) \neq \tilde{Y})$. In verbal discussion in this paper, \textit{type I error} without any suffix refers to type I error regarding the true labels.

\subsection{The NP umbrella algorithm without label noise}\label{sec:3.1}

The NP paradigm \citep{cannon2002learning, scott2005neyman} aims to mimic the  NP oracle 
\[
\phi^*_{{\alpha}} \in \argmin_{\phi:\  R_0(\phi) \leq {\alpha}} R_1(\phi)\,,
\]
where $\alpha\in(0,1)$ is a user-specified level that reflects the priority towards the type I error. In practice, with or without label noise, based on training data of finite sample size, it is usually impossible to ensure $R_0(\cdot)\leq \alpha $ almost surely. Instead, we aim to control the type I error with high probability. Recently, the NP umbrella algorithm \cite[]{tong2018neyman} has attracted significant attention\footnote{At the time of writing, the NP umbrella package has been downloaded over $35{,}000$ times.}. This algorithm works in conjunction with any score based classification method (e.g., logistic regression, support vector machines, or random forest) to compress a $d$-dimensional feature measurement to a $1$-dimensional score, and then threshold the score to classify. 
%Common examples include logistic regression, support vector machines, or random forests.
%\textcolor{red}{
%A score based classification method  first compresses a $d$-dimensional feature measurement to a $1$-dimensional score, and then thresholds the score to classify. The NP umbrella algorithm does not interfere with the score construction stage, but focus on choosing a proper threshold.}  
Specifically, \textit{given a (score based) classification method}, the NP umbrella algorithm uses a model-free order statistics approach to decide the threshold,  attaining a high probability control on type I error with minimum type II error \textit{for that method}.  Moreover, when coupling with a classification method that matches the underlying data distribution, the NP umbrella algorithm also achieves a diminishing excess type II error, i.e.,  $R_1(\hat\phi_{\alpha}) - R_1(\phi^*_{\alpha})\rightarrow 0$.  For example, \cite{Tong.Xia.Wang.Feng.2020} showed that under a linear discriminant analysis (LDA) model, an LDA classifier with the score threshold determined by the NP umbrella algorithm satisfies both the control on type I error and a diminishing excess type II error\footnote{These two properties together were coined as the \textit{NP oracle inequalities} by  \cite{rigollet2011neyman}. Classifiers with these properties were constructed with non-parametric assumptions in \cite{tong2013plug} and \cite{zhao2016neyman}. }.  Next we will review the implementation of the NP umbrella algorithm.  

%In view of this, \cite{rigollet2011neyman} proposed that a good classifier $\hat\phi_{\alpha}$ under the NP paradigm should satisfy \textit{NP oracle inequalities}:  it holds  simultaneously with high probability  that (i) $R_0(\hat \phi)\leq \alpha$ and (ii) the excess type II error $R_1(\hat\phi_{\alpha}) - R_1(\phi^*_{\alpha})$ diminishes with explicit rates. Classifiers that satisfy these inequalities were constructed under the nonparametric \citep{tong2013plug, zhao2016neyman} and parametric \citep{Tong.Xia.Wang.Feng.2020} assumptions.

%This departure forgoes an ambition to derive an excess type II error bound, because no distributional assumptions are assumed and there is no guarantee that a user-preferred classification method will actually match the underlying data distribution.

%The NP umbrella algorithm has attracted attention from practitioners who routinely just use a handful of classification methods. At the time when this article is written, the number of downloads of the companying package has exceeded $20,000$. In the same spirit, we will pursuit an ``umbrella" idea to adapt popular classification methods to the NP paradigm in the presence of label noise. A \textit{fundamental challenge} here is that we can only train our classifiers using \textit{corrupted} data, but would like to have type I error regarding the \textit{true} labels controlled.  

Let $\mathcal{S}^0 = \{X_j^0\}_{j=1}^{M_0}$ and $\mathcal{S}^1 = \{X_j^1\}_{j=1}^{M_1}$, respectively be the \textit{uncorrupted} observations in classes 0 and 1, where $M_0$ and $M_1$ are the number of observations from each class\footnote{Note that the uncorrupted data $\mathcal{S}^0$ and $\mathcal{S}^1$ are not available in our present label noise setting and we only use them here for review purposes.}. Then, given a classification method (i.e., base algorithm, e.g., logistic regression), the NP umbrella algorithm is implemented by randomly splitting the class 0 data $\mathcal{S}^0$ into two parts: $\mathcal{S}^0_{\text{b}}$ and $\mathcal{S}^0_{\text{t}}$. The first part,  $\mathcal{S}^0_{\text{b}}$, together with $\mathcal{S}^1$, is used to train the \textit{base} algorithm, while the second part  $\mathcal{S}^0_{\text{t}}$ determines the {\it threshold} candidates. Specifically, we train a base algorithm with scoring function $\hat T(\cdot)$ (e.g., the sigmoid function in logistic regression) using $\mathcal{S}^0_{\text{b}}\, \cup\,  \mathcal{S}^1$, apply $\hat T(\cdot)$ on  $\mathcal{S}^0_{\text{t}}$ ($|\mathcal{S}^0_{\text{t}}| = n$)  to get threshold candidates  $\{t_1, \ldots, t_n\}$, and sort them in an increasing order $\{t_{(1)}, \ldots, t_{(n)}\}$. Then the NP umbrella algorithm proposes classifier $\hat{\phi}_{k_*}(\cdot) = \1\{\hat{T}(\cdot) > t_{(k_*)}\}$, where 
\begin{equation}
\nllabel{eqn.kstar}
    k_*   = \text{min}\left\{k\in\{1, \ldots, n\}:\sum_{j=k}^{n} {n \choose j}(1-{\alpha})^j{\alpha}^{(n-j)}\leq \delta \right\}\,,
\end{equation} 
in which $\delta$ is a user-specified  tolerance probability of the type I error exceeding $\alpha$. The key to this approach is that \cite{tong2018neyman} established, for all $\hat \phi_k (\cdot) = \1\{\hat T (\cdot)> t_{(k)}\}$ where $k\in\{1, \ldots, n\}$,  it holds $\p(R_0(\hat \phi_{k})>\alpha)
\leq \sum_{j=k}^{n}  {n \choose j}(1-{\alpha})^j{\alpha}^{(n-j)}$, where $\p$ corresponds to random draws of $\mathcal{S}^0$ and $\mathcal{S}^1$, as well as potential randomness in the classification method (e.g., random forest), and the inequality becomes an equality when $\hat T$ is continuous almost surely.  In view of this inequality and the definition for $k_*$, we have $\p(R_0(\hat \phi_{k_*})>\alpha)\leq \delta$, and $\hat \phi_{k_*}$ achieves the smallest type II error among the $\hat \phi_{k}$'s that respect the $(1-\delta)$ probability control of the type I error. We call this algorithm  the \textit{original} NP umbrella algorithm to contrast with the newly developed versions.

%\textcolor{red}{update Figure 2, so that it has a panel b, which implements the new algorithm.}

\subsection{Algorithm \ref{alg:adj_umbrella}: label-noise-adjusted NP umbrella algorithm with known corruption levels}\label{sec:adj_alg}

Returning to our errors in labels problem leads one to ask what would happen if we were to directly apply  the original NP umbrella algorithm to the label noise setting? The results are mixed. While this algorithm successfully controls type I error, it tends to be massively conservative, leading to very low type I errors, but high type II errors. The next example illustrates this phenomenon.  
\begin{figure}[t]
\begin{center}
    \includegraphics[width = \textwidth, height = 7cm]{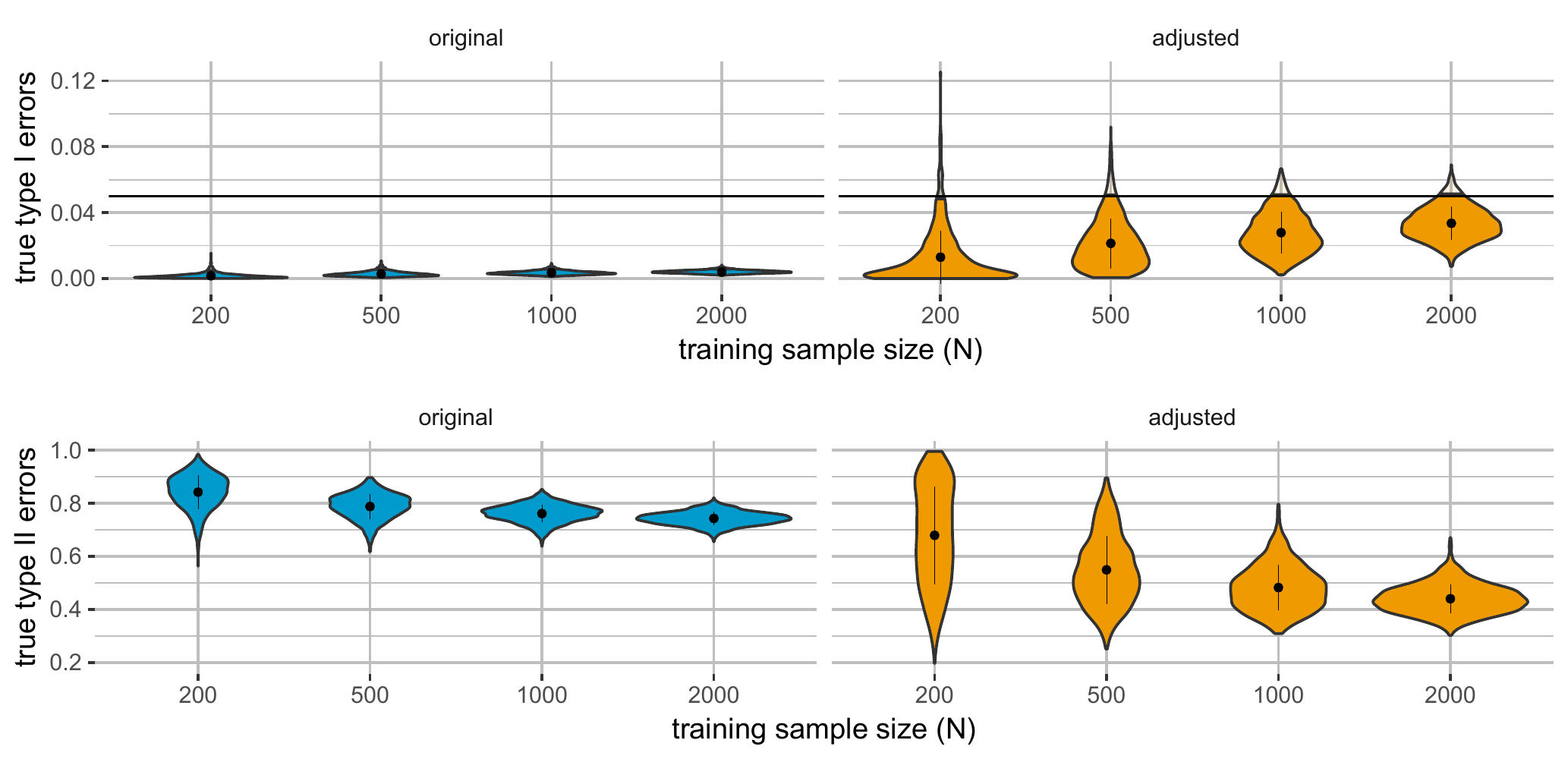}
    \caption{The original NP umbrella algorithm vs. a label-noise-adjusted version  for Example \ref{ex:trivial_classifier}.  The plots in the left panel (blue) are the true type I and II errors for the original NP umbrella algorithm. The plots in the right panel (orange) are the true type I and II errors for the label-noise-adjusted NP umbrella algorithm with known corruption levels. The black dot and vertical bar in every violin represent mean and standard deviation, respectively. In the top row, the horizontal black line is $\alpha = 0.05$ and the boundaries between lighter and darker color in each violin plot mark the $1-\delta =95\%$ quantiles. }\label{fig:trivial}
\end{center}
\end{figure}

\begin{example}\label{ex:trivial_classifier}
Let $X^0 \sim \mathcal{N}(0,1)$ and $X^1 \sim \mathcal{N}(2,1)$, $m_0 = 0.85$, $m_1 = 0.15$, $\alpha = 0.05$ and $\delta = 0.05$. For simplicity, we use the identity scoring function: $\hat{T}(X) = X$. We generate $N\in\{200, 500, 1000, 2000\}$ corrupted class $0$ observations and train a classifier $\hat \phi_{k_*}(\cdot)$ based on them.  Due to normality, we can analytically calculate the type I and II errors regarding the true labels. The above steps are repeated $1{,}000$ times for every value of $N$ to graph the violin plots of both errors as shown in the left panel of Figure \ref{fig:trivial}. Clearly, all the achieved true type I errors are much lower than the control target $\alpha$ and true type II errors are very high \footnote{To make a contrast, we also plot in the right panel of Figure \ref{fig:trivial}  the true type I and II errors of $\hat \phi_{k^*}(\cdot)$, the classifier constructed by the  label-noise-adjusted NP umbrella algorithm with known corruption levels to be introduced in the next section. The details to generate $\hat \phi_{k^*}(\cdot)$'s are skipped here, except we reveal that corrupted class 1 observations, in addition to the corrupted class 0 observations, are also needed to construct the thresholds.}.

% Then, we generate $N$ corrupted class $0$ observations and $N$ corrupted class $1$ observations and compute $\hat \phi_{k^*}(\cdot)$, the classifier generated by label-noise-adjusted NP umbrella algorithm which will be later in this section. The details are included in Algorithm \ref{alg:adj_umbrella} with modifications that we omit set $\tilde{\mathcal{S}}_{\text{b}}^0$ and $\tilde{\mathcal{S}}_{\text{b}}^1$. That is, we directly set $\hat{T}(X) = X$, divide $200$ corrupted $0$ observations evenly into $\tilde{\mathcal{S}}_{\text{t}}$ and $\tilde{\mathcal{S}}_{\text{e}}^0$, and use all $200$ corrupted class $1$ observations as $\tilde{\mathcal{S}}_{\text{e}}^1$.
%  As a remark, the oracle NP classifier for the uncorrupted model is $\phi(X) = \1\{X > 1.64\}$, with type I error equalling $0.05$ and type II error equalling $0.36$. As shown in Figure \ref{fig:trivial}, $\hat{\phi}_{k^*}$ controls true type I errors with high probability while greatly reduces true type II errors compared with $\hat{\phi}_{k_*}$.

\end{example}

The phenomenon illustrated in the left panel of Figure \ref{fig:trivial} is not a contrived one. Indeed, under the class-conditional noise model (i.e., Assumption \ref{assumption:mixture}), at the same threshold level,  the tail probability of corrupted class $0$ is greater than that of true class $0$ since the corrupted $0$ distribution is a mixture of true $0$ and $1$ distributions.  
% Hence, even if we have a corrupted sample of infinity size and therefore would be able to construct the corrupted level-$
% \alpha$ NP oracle, this corrupted oracle would have true type I error smaller than the target $\alpha$.  
Figure \ref{fig:intuition} provides further illustration. In this figure, the black vertical line ($x=2.52$) marks the threshold of the classifier $\1\{X>2.52\}$ whose corrupted type I error (i.e., the right tail probability under the orange dashed curve) is $0.05$.  In contrast, its true type I error (i.e., the right tail probability under the blue solid curve) is much smaller.

\begin{figure}
\begin{center}
    \includegraphics[width = \textwidth, height = 7cm,]{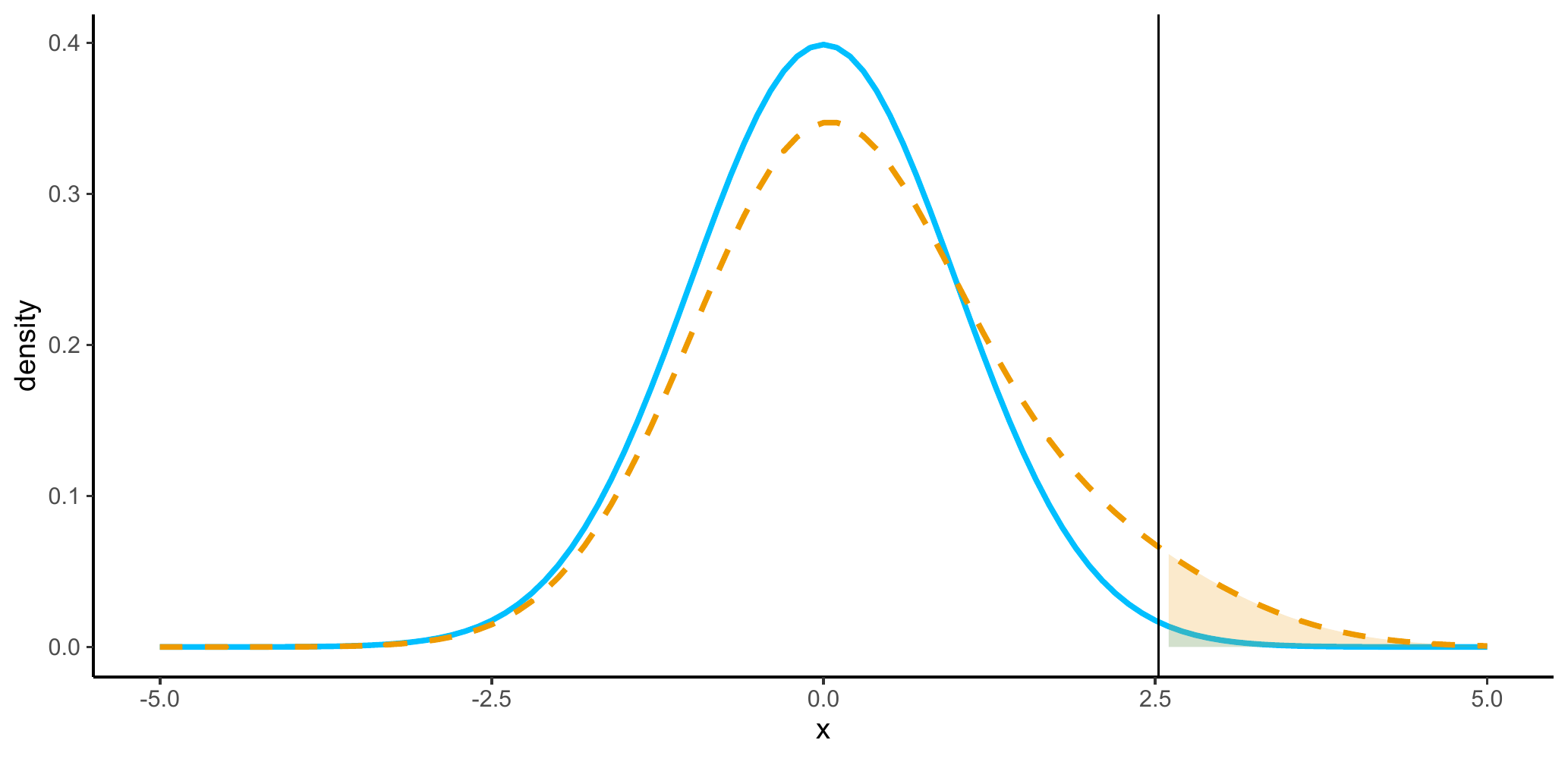}
    \caption{The blue solid curve is the density of true class $0$ (i.e., $\mathcal{N}(0,1)$) and the orange dashed curve is the density of corrupted class $0$ (i.e., a mixture of $\mathcal{N}(0,1)$ and $\mathcal{N}(2,1)$ with $m_0 = 0.85$). The black vertical line  marks the threshold of the classifier $\1\{X > 2.52 \}$ whose corrupted type I error is $0.05$.  }\label{fig:intuition}
\end{center}
\end{figure}

The above observation motivates us to create new label-noise-adjusted NP umbrella algorithms by carefully studying the discrepancy between true and corrupted type I errors, whose population-level relation is channeled by the class-conditional noise model and can be estimated based on data with corrupted labels alone. We will first develop a version for known corruption levels (i.e., Algorithm \ref{alg:adj_umbrella}) and then a variant for unknown corruption levels (i.e., Algorithm $1^{\#}$). Although the latter variant is suitable for most applications, we believe that presenting first the known corruption level version streamlines the reasoning and presentation.   
%We delay the detailed rationale and theory-backing to Section \ref{sec:theory} and present here the implementation of the new algorithm. 

For methodology and theory development, we assume the following sampling scheme. Let $\tilde{\mathcal{S}}^0 = \{\tilde{X}_j^0\}_{j=1}^{N_0}$ be \textit{corrupted} class $0$ observations and $\tilde{\mathcal{S}}^1 = \{\tilde{X}^1_j\}_{j=1}^{N_1}$ \textit{corrupted} class $1$ ones. The sample sizes $N_0$ and $N_1$ are considered to be non-random numbers, and we assume that all observations in $\tilde{\mathcal{S}}^0$ and $\tilde{\mathcal{S}}^1$ are independent. Then, we divide $\tilde{\mathcal{S}}^0$ into \textit{three} random disjoint non-empty subsets. The first two parts  $\tilde{\mathcal{S}}^0_{\text{b}}$ and  $\tilde{\mathcal{S}}^0_{\text{t}}$  are used to train the {\it base} algorithm and determine the {\it threshold} candidates, respectively. The third part  $\tilde{\mathcal{S}}^0_{\text{e}}$ is used to {\it estimate} a correction term to account for the label noise. Similarly, we randomly divide $\tilde{\mathcal{S}}^1$ into \textit{two} disjoint non-empty subsets: $\tilde{\mathcal{S}}^1_{\text{b}}$ and $\tilde{\mathcal{S}}^1_{\text{e}}$. 
 %Furthermore, we denote $\tilde{\mathcal{S}}_{\text{b}} = \tilde{\mathcal{S}}_{\text{b}}^0 \cup \tilde{\mathcal{S}}^1_{\text{b}}$, which contains both corrupted class $0$ and $1$ observations. 
 
 %The subindex `b' indicates that $\tilde{\mathcal{S}}_{\text{b}}$ is for training a base algorithm (i.e., a scoring function), `t' indicates that $\tilde{\mathcal{S}}_{\text{t}}$ is for thresholding, and `e' indicates that $\tilde{\mathcal{S}}^0_{\text{e}}$ and $\tilde{\mathcal{S}}^1_{\text{e}}$ are researved for some estimation. 
% For $\delta \in (0, 1)$ and $k\in\{1, \ldots, n\}$, let $\alpha_{k,\delta}$ be such that
%\begin{align}\label{eqn:alpha_k_delta}
%    \sum_{j=k}^n{n \choose j}\alpha_{k,\delta}^{n-j}(1-\alpha_{k,\delta})^j = \delta\,.
%\end{align}
 
Let $\hat{T}(\cdot)$ be a scoring function trained on $\tilde{\mathcal{S}}_{\text{b}}=\tilde{\mathcal{S}}_{\text{b}}^0 \cup \tilde{\mathcal{S}}^1_{\text{b}}$.  
%We denote $\tilde n  = |\tilde{\mathcal{S}}_{\text{b}}|$, $n = |\tilde{\mathcal{S}}_{\text{t}}|$, $n_0 = |\tilde{\mathcal{S}}^0_{\text{e}}|$, and $n_1 = |\tilde{\mathcal{S}}^1_{\text{e}}|$. \textcolor{red}{consider moving the previous sentence to section 4} 
We apply $\hat T(\cdot)$ to elements in $\tilde{\mathcal{S}}^0_{\text{t}}$ and sort them in an increasing order: $\{t_{(1)}, \ldots, t_{(n)}\}$, where $n = |\tilde{\mathcal{S}}^0_{\text{t}}|$\footnote{In Appendix \ref{sec:sampling_scheme_summary}, we summarize the notations related to the sampling scheme for the readers' convenience.}. These will serve as the threshold candidates, just as in the original NP umbrella algorithm. However, instead of $k_*$, the label-noise-adjusted NP umbrella algorithm with known corruption levels will take the order $k^*$ defined by
$$
k^* =  \min\{k\in\{1, \ldots, n\}:\alpha_{k,\delta} - \hat{D}^+(t_{(k)}) \leq \alpha \}\,,
$$   
where $\alpha_{k,\delta}$\footnote{The existence and uniqueness of $\alpha_{k,\delta}$ are ensured by Lemma \ref{lemma:existence_of_alpha_k_delta} in the Appendix.} satisfies
\begin{align}\label{eqn:alpha_k_delta}
    \sum_{j=k}^n{n \choose j}\alpha_{k,\delta}^{n-j}(1-\alpha_{k,\delta})^j = \delta\,,
\end{align}
$\hat{D}^+(\cdot) =\hat D(\cdot)\vee 0:= \max(\hat D(\cdot), 0)$ and $\hat D(\cdot ) = \frac{1-m_0}{m_0-m_1}\left(\hat{\tilde{F}}_0^{\hat{T}}(\cdot) - \hat{\tilde{F}}_1^{\hat{T}}(\cdot)\right)$, in which $\hat{\tilde{F}}_0^{\hat{T}}(\cdot)$ and $\hat{\tilde{F}}_1^{\hat{T}}(\cdot)$ are empirical estimates of $\tilde{F}_0^{\hat{T}}(\cdot)$ and $\tilde{F}_1^{\hat{T}}(\cdot)$ based on $\tilde{\mathcal{S}}^0_{\text{e}}$ and $\tilde{\mathcal{S}}^1_{\text{e}}$, respectively.

\begin{algorithm}[htb!]
\caption{Label-noise-adjusted NP Umbrella Algorithm with known corruption levels
\label{alg:adj_umbrella}}
\SetKw{KwBy}{by}
\SetKwInOut{Input}{Input}\SetKwInOut{Output}{Output}
\SetAlgoLined

\Input{$\tilde{\mathcal{S}}^0$: sample of corrupted $0$ observations \\
$\tilde{\mathcal{S}}^1$: sample of corrupted $1$ observations \\
$\alpha$: type I error upper bound, $0 < \alpha < 1$\\
$\delta$: type I error violation rate target, $0 < \delta < 1$ \\
$m_0$: probability of a corrupted class 0 sample being of true class 0 \\
$m_1$: probability of a corrupted class 1 sample being of true class 0  \\
}

$\tilde{\mathcal{S}}^0_{\text{b}}$, $\tilde{\mathcal{S}}^0_{\text{t}}$,$\tilde{\mathcal{S}}^0_{\text{e}} \leftarrow$ random split on $\tilde{\mathcal{S}}^0$ 

$\tilde{\mathcal{S}}^1_{\text{b}}$, $\tilde{\mathcal{S}}^1_{\text{e}} \leftarrow$ random split on $ \tilde{\mathcal{S}}^1$

$\tilde{\mathcal{S}}_{\text{b}} \leftarrow \tilde{\mathcal{S}}^1_{\text{b}} \cup \tilde{\mathcal{S}}^0_{\text{b}}$ \tcp*{combine $\tilde{\mathcal{S}}^0_{\text{b}}$ and $\tilde{\mathcal{S}}^1_{\text{b}}$ as $\tilde{\mathcal{S}}_{\text{b}}$}

$\hat{T}(
\cdot) \leftarrow$ \textsf{base classification algorithm}$(\tilde{\mathcal{S}}_{\text{b}})$ \tcp*{train a scoring function on $\tilde{\mathcal{S}}_{\text{b}}$}

%$n\leftarrow|\mathcal{S}_{\text{t}}|$ \tcp*{denote $n$ as the size of $\mathcal{S}_{\text{t}}$}

%$\{\tilde{X}^0_1,\tilde{X}^0_2,\ldots,\tilde{X}^0_{n}\} \leftarrow \mathcal{S}_{\text{t}}$ \tcp*{write $\mathcal{S}_{\text{t}}$ in terms of its entries}

$\mathcal{T}_{\text{t}} = \{t_1, t_2, \ldots, t_{n}\} \leftarrow \hat T(\tilde{\mathcal{S}}^0_{\text{t}})$ \tcp*{apply $\hat{T}$ to every entry in $\tilde{\mathcal{S}}_{\text{t}}$}

$\{t_{(1)}, t_{(2)}, \ldots, t_{(n)}\} \leftarrow \textsf{sort}(\mathcal{T}_{\text{t}})$
	   
$\mathcal{T}^0_{\text{e}} \leftarrow \hat{T}(\tilde{\mathcal{S}}^0_{\text{e}})$

$\mathcal{T}^1_{\text{e}} \leftarrow \hat{T}(\tilde{\mathcal{S}}^1_{\text{e}})$ \tcp*{apply $\hat{T}$ to all elements in $\tilde{\mathcal{S}}^0_{\text{e}}$ and $\tilde{\mathcal{S}}^1_{\text{e}}$}

\For{$k$ in $\{1,\dots,n\}$}{

$\alpha_{k,\delta} \leftarrow \textsf{BinarySearch}(\delta, k, n)$ \tcp*{compute $\alpha_{k, \delta}$ through binary search}

$\hat{\tilde{F}}_0^{\hat{T}}(t_{(k)}) \leftarrow \left|\mathcal{T}^0_{\text{e}}\right|^{-1}\cdot\sum_{t \in \mathcal{T}^0_{\text{e}}}\1\{t \leq t_{(k)}\}$ 

$\hat{\tilde{F}}_1^{\hat{T}}(t_{(k)}) \leftarrow \left|\mathcal{T}^1_{\text{e}}\right|^{-1}\cdot\sum_{t \in \mathcal{T}^1_{\text{e}}}\1\{t \leq t_{(k)}\}$ \tcp*{compute the empirical distributions}

$\hat{D}(t_{(k)}) \leftarrow \frac{1-m_0}{m_0-m_1}\left(\hat{\tilde{F}}_0^{\hat{T}}(t_{(k)}) - \hat{\tilde{F}}_1^{\hat{T}}(t_{(k)})\right)$ \tcp*{compute an estimate of $\tilde{R}_0 - R_0$}

$\hat{D}^+(t_{(k)}) \leftarrow \hat{D}(t_{(k)}) \vee 0$ \tcp*{if $\hat{D}(t_{(k)})$ is negative, then set it to $0$}
}

$k^* \leftarrow \min\{k\in\{1, \ldots, n\}:\alpha_{k,\delta} - \hat{D}^+(t_{(k)}) \leq \alpha \}$ \tcp*{select the order}

$\hat{\phi}_{k^*}(\cdot) \leftarrow \1\{\hat{T}(\cdot) > t_{(k^*)}\}$ \tcp*{construct an NP classifier}
 
\Output{$\hat{\phi}_{k^*}(\cdot)$}
\end{algorithm}

The entire construction process of $\hat \phi_{k^*}(\cdot) = \1\{\hat T(\cdot) > t_{(k^*)}\}$ is summarized and detailed in Algorithm \ref{alg:adj_umbrella}. In this algorithm, to solve $\alpha_{k,\delta}$, we use a binary search subroutine (Algorithm \ref{alg:bi_search} in Appendix \ref{sec:binary_search}) on the function $x \mapsto \sum_{j=k}^n{n\choose k}x^{n-j}(1-x)^j$, leveraging its strict monotone decreasing property in $x$. Interested readers are referred to the proof of Lemma \ref{lemma:existence_of_alpha_k_delta} in the Appendix for further reasoning. Currently we randomly split $\tilde{\mathcal{S}}^0$ and $\tilde{\mathcal{S}}^1$ respectively  into three and two equal sized subgroups. %in Algorithm \ref{alg:adj_umbrella}, we split them into even pieces. 
An optimal splitting strategy could be a subject for future research.

The key to the new algorithm is $\hat D^+(\cdot)$, which adjusts for the label corruption. Indeed, the original NP umbrella algorithm can be seen as a special case of our approach where $\hat D^+(\cdot)=0$. The numerical advantage of the new algorithm is demonstrated in the right panel of Figure \ref{fig:trivial} and in Section \ref{sec:sim_and_real_data}. 
We will prove in the next section that the \textit{label-noise-adjusted NP classifier}  $\hat{\phi}_{k^*}(\cdot) = \1\{\hat{T}(\cdot) > t_{(k^{*})}\}$ controls true type I error with high probability while avoiding the  excessive conservativeness of the original NP umbrella algorithm. Note that in contrast to the deterministic order $k_*$ in the original NP umbrella algorithm, the new order $k^*$ is random, calling for much more involved technicalities to establish the theoretical properties of $\hat{\phi}_{k^*}(\cdot)$.

\subsection{Algorithm $1^{\#}$: label-noise-adjusted NP umbrella algorithm with unknown corruption levels}\label{sec:adj_alg_unknown}

% For most applications in practice, accurate corruption levels $m_0$ and $m_1$ are inaccessible. To address this, we propose Algorithm $1^{\#}$, a variant of Algorithm \ref{alg:adj_umbrella} that uses estimated $m_0$ and $m_1$ (i.e., $m_0^\#$ and $m_1^\#$) as inputs. Specifically, when estimating $\tilde R_0 - R_0$, Algorithm $1^{\#}$ uses  $\hat{D}_\#(t_{(k)}) = \frac{1-m_0^\#}{m_0^\#-m_1^\#}\left(\hat{\tilde{F}}_0^{\hat{T}}(t_{(k)}) - \hat{\tilde{F}}_1^{\hat{T}}(t_{(k)})\right)$ and $\hat{D}^+_\#(t_{(k)}) = \hat{D}_\#(t_{(k)}) \vee 0$. Then, Algorithm $1^{\#}$ delivers the NP classifier $\hat{\phi}_{k^*_\#}(\cdot) = \1\{\hat{T}(\cdot) > t_{(k^*_\#)}\}$, where $k^*_\# = \min\{k\in\{1, \ldots, n\}:\alpha_{k,\delta} - \hat{D}^+_\#(t_{(k)}) \leq \alpha \}$. Due to the similarity with Algorithm \ref{alg:adj_umbrella}, we do not re-produce the other steps of Algorithm $1^{\#}$ to write it out in a full algorithm format.  

For most applications in practice, accurate corruption levels $m_0$ and $m_1$ are inaccessible.
To address this, we propose Algorithm $1^{\#}$, a simple variant of Algorithm \ref{alg:adj_umbrella} that replaces $m_0$ and $m_1$ with estimates $m_0^{\#}$ and $m_1^{\#}$. In all other respects the two algorithms are identical. Specifically, when estimating $\tilde R_0 - R_0$, Algorithm $1^{\#}$ uses  $\hat{D}_\#(t_{(k)}) = \frac{1-m_0^\#}{m_0^\#-m_1^\#}\left(\hat{\tilde{F}}_0^{\hat{T}}(t_{(k)}) - \hat{\tilde{F}}_1^{\hat{T}}(t_{(k)})\right)$ and $\hat{D}^+_\#(t_{(k)}) = \hat{D}_\#(t_{(k)}) \vee 0$. Then, Algorithm $1^{\#}$ delivers the NP classifier $\hat{\phi}_{k^*_\#}(\cdot) = \1\{\hat{T}(\cdot) > t_{(k^*_\#)}\}$, where $k^*_\# = \min\{k\in\{1, \ldots, n\}:\alpha_{k,\delta} - \hat{D}^+_\#(t_{(k)}) \leq \alpha \}$. Due to the similarity with Algorithm \ref{alg:adj_umbrella}, we do not re-produce the other steps of Algorithm $1^{\#}$ to write it out in a full algorithm format.  

Rather than supplying unbiased estimates for $m_0$ and $m_1$, we will demonstrate that it is important that $m_0^{\#}$ and $m_1^{\#}$ are under-estimates of the corruption levels (i.e., $m_0^{\#}\geq  m_0$ and $m_1^{\#}\leq m_1$ as in Assumption \ref{assumption:mixture}). In this work, we assume that domain experts supply these under-estimates.  While it would be unrealistic to assume that these experts know $m_0$ and $m_1$ exactly, in many scenarios one can provide accurate bounds on these quantities. It would be interesting to investigate data-driven estimators that have such a property for future work.

\section{Theory}\label{sec:theory}

In this section, we first elaborate the rationale behind Algorithm \ref{alg:adj_umbrella}  (Section \ref{sec:3.2}), and then show that under a few technical conditions, this new algorithm induces  well-defined classifiers whose type I errors are bounded from above by the desired level with high probability (Section  \ref{sec:theoretical properties}). Then we establish a similar result for its unknown-corruption-level variant, Algorithm $1^{\#}$ (Section \ref{sec: algorithm 2 property}).

\subsection{Rationale behind Algorithm \ref{alg:adj_umbrella}  }\label{sec:3.2}

\begin{proposition}\label{prop:umbrella}
Let $\hat{T}(\cdot)$ be a scoring function (e.g., sigmoid function in logistic regression) trained on $\tilde{\mathcal{S}}_{\emph{\text{b}}}$. Applying $\hat{T}(\cdot)$ to every element in $\tilde{\mathcal{S}}^0_{\emph{\text{t}}}$, we get a set of scores. Order these scores and denote them by $\{t_{(1)}, t_{(2)}, \ldots, t_{(n)}\}$,  in which $t_{(1)} \leq t_{(2)} \leq \ldots \leq t_{(n)}$. Then, for any $\alpha \in (0, 1)$ and $k \in \{1,2,\ldots,n\}$, the classifier $\hat{\phi}_k(\cdot) = \1\{\hat{T}(\cdot) > t_{(k)}\}$ satisfies
\[
\p\left(\tilde{R}_0(\hat{\phi}_k) > {\alpha}\right) \leq \sum_{j=k}^{n} {n \choose j}(1-{\alpha})^j{\alpha}^{(n-j)}\,,
\]
%for any $\alpha \in (0,1)$. When $\hat{T}(X)$ is a continuous random variable, the bound is tight.
in which $\p$ is regarding the randomness in all training observations, as well as additional randomness if we adopt certain random classification methods (e.g., random forest). Moreover, when $\hat T(\cdot)$ is continuous almost surely, the above inequality obtains the equal sign. 
\end{proposition}

Recall that $\tilde R_0(\cdot)$ denotes  type I error regarding the \textit{corrupted} labels. We omit a proof for Proposition \ref{prop:umbrella} as it follows the same proof as its counterpart in \cite{tong2018neyman}. For $\alpha, \delta \in (0, 1)$, recall that the original  NP umbrella algorithm selects $k_*   = \text{min}\{k\in\{1, \ldots, n\}:\sum_{j=k}^{n} {n \choose j}(1-{\alpha})^j{\alpha}^{(n-j)}\leq \delta \}$.  The smallest $k$ among all that satisfy $\sum_{j=k}^{n} {n \choose j}(1-{\alpha})^j{\alpha}^{(n-j)}\leq \delta$  is desirable because we also wish to minimize the type II error. There is a sample size requirement for this order statistics approach to work because a finite order $k_*$ should exist.  Precisely, an order statistics approach works if the last order does; that is $(1-\alpha)^n \leq \delta$. This translates to Assumption \ref{assumption:sample_size} on $n$, the sample size of $\tilde{\mathcal{S}}^0_{\text{t}}$. This is a mild requirement. For instance, when $\alpha = \delta = 0.05$, $n$ should be at least $59$.

\begin{customassumption}{2}\label{assumption:sample_size}
$n \geq \lceil\log\delta/\log(1 - \alpha)\rceil$, in which $\lceil \cdot \rceil$ denotes the ceiling function.  
\end{customassumption}

In view of Proposition \ref{prop:umbrella}, the choice of $k_*$ guarantees $\p\left(\tilde{R}_0(\hat{\phi}_{k_*}) \leq  {\alpha}\right) \geq 1-\delta$. In other words, if we were to ignore the label noise presence and apply the original NP umbrella algorithm, the type I error regarding the \textit{corrupted} labels, $\tilde R_0$, is controlled under level $\alpha$ with probability at least $1- \delta$. Moreover, the achieved $\tilde R_0$ is usually not far from $\alpha$ when the sample size $n$ is much larger than the lower bound requirement.   However, this is not our main target; what we really want is to control $R_0$. Example~\ref{ex:trivial_classifier} in Section~\ref{sec:3.1} convincingly demonstrates that in the presence of label noise, the achieved $R_0$ after naive implementation of the original NP umbrella algorithm  can be much lower than the control target $\alpha$. This is no exception.  To aid in analyzing the gap between $R_0$ and $\tilde R_0$, we make the following assumption.

\begin{customassumption}{3}\label{assumption:separability}
The scoring function $\hat{T}$ is trained such that $\tilde{F}_0^{\hat{T}}(z) > \tilde{F}_1^{\hat{T}}(z)$ for all $z \in \R$ with probability at least $1 - \delta_1(n_{\emph{b}})$, where $n_{\emph{b}} = |\tilde{\mathcal{S}}_{\emph{b}}|$ and $\delta_1(n_{\emph{b}})$ converges to $0$ as $n_{\emph{b}}$ goes to infinity.
\end{customassumption}

Loosely, Assumption \ref{assumption:separability} means that the scoring function trained on corrupted data still has the ``correct direction." For any classifier of the form $\hat{\phi}_c(\cdot) = \1\{\hat{T}(\cdot) > c\}$, Assumption \ref{assumption:separability} implies that with probability at least $1- \delta_1(n_{\text{b}})$,  $\tilde{P}_0(\hat{\phi}_c(X) = 0) > \tilde{P}_1(\hat{\phi}_c(X) = 0)$, which means that a corrupted class $0$ observation is more likely to be classified as $0$ than a corrupted class $1$ observation is.  Interested readers can find a concrete example that illustrates this mild assumption  in the Appendix \ref{sec:example_3} (Example \ref{ex:separability}). Now we are ready to describe the discrepancy between $R_0$ and $\tilde R_0$.

\begin{lemma}\label{lemma:general_gap}
Let $\hat T$ be a scoring function trained on $\tilde{\mathcal{S}}_{\emph{\text{b}}}$ and  $\hat{\phi}_c(\cdot) = \1\{\hat{T}(\cdot) > c\}$ be a classifier that thresholds the scoring function at $c\in \R$. Denote $D(c) = \tilde{R}_0(\hat{\phi}_c) - R_0(\hat{\phi}_c)$. Then, under Assumptions \ref{assumption:mixture}, \ref{assumption:sample_size} and \ref{assumption:separability}, for given $\alpha$ 
%(i.e., type I error upper bound) 
and $\delta$,
%(i.e., type I error violation rate target), 
it holds that %$\inf_{c \in \R}D(c) \geq 0$ with probability $1 - \delta_1(\tilde{n})$ and
\begin{align*}
  \p\left(\inf_{c \in \R}D(c) \geq 0\right)\geq 1- \delta_1(n_{\emph{b}})\,\text{ }\text{and }\text{ }  \p\left(R_0(\hat{\phi}_{k_*}) > \alpha - D(t_{(k_*)})\right) \leq \delta +  \delta_1(n_{\emph{b}})\,,
\end{align*}
where $k_*$ and $\delta$ are related via equation \eqref{eqn.kstar}.
Moreover, we have
\begin{equation}
\label{eqn.d}
    D(c) = M\left(\tilde{F}_0^{\hat{T}}(c) - \tilde{F}_1^{\hat{T}}(c)\right)\,, 
\end{equation}
where $M = (1-m_0)(m_0-m_1)^{-1}$. 

%Furthermore, let $D_\#(c) = M_\# \left(\tilde{F}^{\hat{T}}_0(c) - \tilde{F}_1^{\hat{T}}(c)\right)$ where $M_\# = (1-m^{\#}_0)(m_0^{\#} - m_1^{\#})^{-1}$. Then,
%\begin{align*}
%    \p\left(R_0(\hat{\phi}_{k_*}) > \alpha - D_\#(t_{(k_*)})\right) \leq \delta +  \delta_1(n_{\emph{b}})\,.
%\end{align*}
\end{lemma}

%\textcolor{red}{Shunan, In the last line of the proof of Lemma 1, quantities related to Algorithm 2 appears.  It should not be that because Lemma 1 is only about Algorithm 1.  If you need the last line of the proof somewhere later, please put it in the exact place where you need it and quote Lemma 1.} 

%\textcolor{purple}{double check that the deleted part in the proof is NOT actually needed.}

Note that $D(c)$ measures the discrepancy between the \textit{corrupted} type I error  and the \textit{true} type I error  of the classifier $\hat \phi_c(\cdot)$.  
Lemma \ref{lemma:general_gap} implies that with high probability, $\hat{\phi}_{k_*}(\cdot)$ has $R_0$, the type I error regarding \textit{true} labels, under a level that is smaller than the target value $\alpha$, and that the gap is measured by $D(t_{(k_*)})$. It is important to note that $D(c)$ is solely a function of the distributions of the corrupted data, and does not require any knowledge of the uncorrupted scores, so we are able to estimate this quantity from our observed data.

As argued previously, excessive  conservativeness in type I error is not desirable because it is usually associated with a high type II error.  Therefore, a new NP umbrella algorithm should adjust to the label noise, so that the resulting classifier respects the true type I error control target, but is not excessively  conservative.   Motivated by Lemma \ref{lemma:general_gap}, our central plan is to choose some less conservative (i.e., smaller) order than that in the original NP umbrella algorithm, in view of the difference between $R_0$ and $\tilde R_0$. Recall that $\delta\in(0,1)$ is the target type I error violation rate. In the presence of label noise, we do not expect to control at this precise violation rate, but just some number around it.

For any $\hat \phi_k(\cdot)$, under Assumptions \ref{assumption:mixture}, \ref{assumption:sample_size} and \ref{assumption:separability}, Lemma \ref{lemma:general_gap} implies $\tilde{R}_0(\hat{\phi}_k) \geq R_0(\hat{\phi}_k)$ with probability at least $1 - \delta_1(n_{\text{b}})$. Note that the $\delta_1(n_{\text{b}})$ term is small and asymptotically $0$; we will ignore it in this section when motivating our new strategy. With this simplification,  $\tilde{R}_0(\hat{\phi}_k)$ is always greater than $R_0(\hat{\phi}_k)$, as illustrated in Figure~\ref{fig:typeIerrorbound}. The definition of $\alpha_{k, \delta}$ in equation \eqref{eqn:alpha_k_delta} and Proposition \ref{prop:umbrella} imply with probability at least $1 - \delta$, $\alpha_{k,\delta} \geq \tilde{R}_0(\hat{\phi}_k)$, which corresponds to the green region (the region on the right) in Figure \ref{fig:typeIerrorbound}. Since we only need $1 - \delta$ probability control on $R_0$, it suffices to control $R_0$ corresponding to this region. Combining the results $\alpha_{k,\delta} \geq \tilde{R}_0(\hat{\phi}_k)$ and $\tilde{R}_0(\hat{\phi}_k) \geq R_0(\hat{\phi}_k)$, we have the inequalities $\alpha_{k,\delta} \geq \alpha_{k,\delta} - D(t_{(k)}) \geq R_0(\hat{\phi}_k)$ on our interested region (Recall $D(t_{(k)}) = \tilde{R}_0(\hat{\phi}_k) - R_0(\hat{\phi}_k)$). By the previous argument, $\alpha_{k, \delta}$ can be used as an upper bound for $R_0$, but to have a good type II error, a better choice is clearly the smaller $\alpha_{k,\delta} - D(t_{(k)})$.  So if $D(t_{(k)})$ were a known quantity, we can set the order to be $\tilde k^* = \min\{k\in\{1\ldots, n\}: \alpha_{k,\delta} - D(t_{(k)}) \leq \alpha\}$ and propose a classifier $\hat{\phi}_{\tilde k^*}(\cdot) = \1\{\hat{T}(\cdot) > t_{(\tilde k^*)}\}$.  This is to be compared with the order $k_*$ chosen by the original NP umbrella algorithm, which can be equivalently expressed as $k_* = \min\{k\in\{1\, \ldots, n\}:\alpha_{k, \delta}\leq \alpha\}$ (Lemma \ref{lemma:existence_of_alpha_k_delta} in the Appendix).  Then we have $\tilde k^* \leq k_*$, and so $\hat \phi_{\tilde k^*}(\cdot)$ is less conservative than $\hat \phi_{k_*}(\cdot)$ in terms of type I error.  

\begin{figure}
\begin{center}
    \includegraphics[width = \textwidth, height = 8cm]{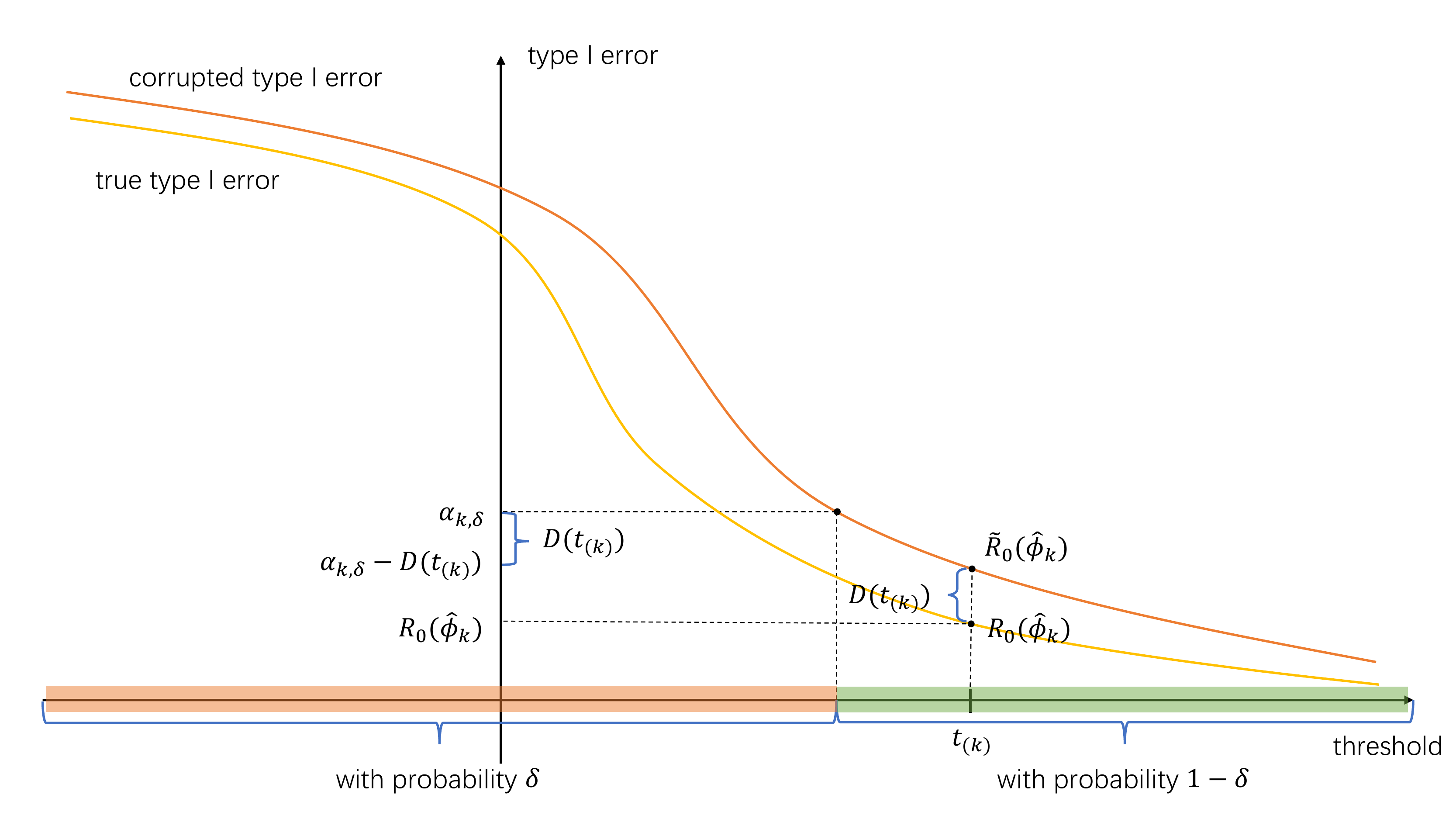}
    \caption{A cartoon illustration of $1-\delta$ probability upper bound of type I error. }\label{fig:typeIerrorbound}
\end{center}
\end{figure}

However, $\hat \phi_{\tilde k^*}(\cdot)$ is not accessible because $D$ is unknown.  Instead we estimate $D$ by replacing $\tilde{F}^{\hat{T}}_0$ and $\tilde{F}_1^{\hat{T}}$ in \eqref{eqn.d} with their empirical distributions $\hat{\tilde{F}}^{\hat{T}}_0$ and $\hat{\tilde{F}}^{\hat{T}}_1$, which are calculated using $\tilde{\mathcal{S}}^0_{\text{e}}$ and $\tilde{\mathcal{S}}^1_{\text{e}}$, i.i.d. samples from the corrupted $0$ and $1$ observations. Note that these estimates are independent of $\tilde{\mathcal{S}}_{{\text{b}}}$ and $\tilde{\mathcal{S}}^0_{{\text{t}}}$.
%in Lemma~\ref{lemma:general_gap} has $D(c) = (1-m_0)(m_0-m_1)^{-1}(\tilde{F}^{\hat{T}}_0(c) - \tilde{F}_1^{\hat{T}}(c))$ for any $c \in \R$, so we can estimate $D$ by replacing $\tilde{F}^{\hat{T}}_0$ and $\tilde{F}_1^{\hat{T}}$ with empirical distributions $\hat{\tilde{F}}^{\hat{T}}_0$ and $\hat{\tilde{F}}^{\hat{T}}_1$ 
%where $\hat{\tilde{F}}^{\hat{T}}_0$ and $\hat{\tilde{F}}^{\hat{T}}_1$ are computed on data independent of $\tilde{\mathcal{S}}_{{\text{b}}}$ and $\tilde{\mathcal{S}}_{{\text{t}}}$. For this purpose, we have reserved $\tilde{\mathcal{S}}^0_{\text{e}}$, a collection of i.i.d. corrupted $0$ observations and $\tilde{\mathcal{S}}^1_{\text{e}}$, a collection of i.i.d. corrupted $1$ observations.   
For a given $\hat{T}$, we  define for every $c \in \R$, 
\begin{align*}
    \hat{D}(c) = \frac{1-m_0}{m_0-m_1}\left(\hat{\tilde{F}}^{\hat{T}}_0(c) - \hat{\tilde{F}}^{\hat{T}}_1(c)\right)\, \text{ and }\, k^{**} = \min\{k\in\{1, \ldots, n\}:\alpha_{k,\delta} - \hat{D}(t_{(k)}) \leq \alpha - \varepsilon\}\,,
\end{align*}
in which a small $\varepsilon>0$ is introduced to compensate for the randomness of $\hat{D}$ in the theory proofs. For simulation and real data, we actually just use $\varepsilon = 0$.  Finally, \textit{the proposed  new label-noise-adjusted NP classifier with known corruption levels is  $\hat{\phi}_{k^*}(\cdot) = \1\{\hat{T}(\cdot) > t_{(k^{*})}\}$}, in which  $k^*$ is a small twist from $k^{**}$ by replacing $\hat D$ with its positive part.  The construction of $\hat \phi_{k^*}(\cdot)$ was detailed in Algorithm \ref{alg:adj_umbrella}.

We have two comments on the implementation of Algorithm \ref{alg:adj_umbrella}.   First, though the $\varepsilon$ compensation for the randomness is necessary for the theory proof, our empirical results suggest almost identical performance between $\varepsilon=0$ relative to any small $\varepsilon$, so we recommend setting $\varepsilon$ to $0$
%, as shown by the comparison between the additional numerical results in Appendix \ref{appendix:numerical}, where we set a small positive $\varepsilon$, and the numerical results in Section \ref{sec:sim_and_real_data}, where we set $\varepsilon=0$. 
for simplicity, and we do not use the $\varepsilon$ compensation in Algorithm \ref{alg:adj_umbrella}. Second, in the order selection criterion of $k^*$ in Algorithm \ref{alg:adj_umbrella}, we use $\hat{D}^+ =\hat{D}\vee 0 :=  \max(\hat{D}, 0)$ instead of $\hat{D}$, because empirically, although highly unlikely, $\hat{D}$ can be negative, which results in $\min\{k\in\{1, \ldots, n\}:\alpha_{k,\delta} - \hat{D}(t_{(k)}) \leq \alpha\} \geq \min\{k\in\{1, \ldots, n\}:\alpha_{k,\delta} \leq \alpha\}$. In this case, the new order could be greater than $k_*$. Since we aim to reduce the conservativeness of the original NP umbrella algorithm, the possibility of $k^* \geq k_*$ will reverse this effort and worsen the conservativeness. To solve this issue, we force the empirical version of $D$ to be non-negative by replacing $\hat{D}$ with $\hat{D}^+$ in Algorithm \ref{alg:adj_umbrella}.

\subsection{Theoretical properties of Algorithm \ref{alg:adj_umbrella} }\label{sec:theoretical properties}

In this subsection, we first formally establish that Algorithm \ref{alg:adj_umbrella} gives rise to valid classifiers (Lemma \ref{lemma:alpha_k_delta}) and then show that these classifiers have the true type I errors controlled under $\alpha$ with high probability (Theorem \ref{thm:adjustment}).

\begin{lemma}\label{lemma:alpha_k_delta}
Under Assumption \ref{assumption:sample_size}, $k^* = \min\{k\in\{1, \ldots, n\}:\alpha_{k,\delta} - \hat{D}^+(t_{(k)}) \leq \alpha\}$ in Algorithm \ref{alg:adj_umbrella} exists. Moreover, this label-noise-adjusted order is no larger than that chosen by the original NP umbrella algorithm; that is  $k^* \leq k_*$. 
\end{lemma}

Lemma \ref{lemma:alpha_k_delta} implies that Algorithm \ref{alg:adj_umbrella} reduces the excessive  conservativeness of the original NP umbrella algorithm on the type I error by choosing a smaller order statistic as the threshold. %Concretely, because $k^*\leq k_*$,  we have $t_{(k^*)} \leq t_{(k_*)}$. Therefore, $R_0(\hat{\phi}_{{k^*}}) \geq R_0(\hat{\phi}_{{k_*}})$ and $R_1(\hat{\phi}_{{k^*}}) \leq R_1(\hat{\phi}_{{k_*}})$, i.e., reduction of conservativeness and increase in power. 
Moreover, if there is no label noise, i.e., when $m_0 = 1$ and $m_1 = 0$, we have  $k^*=\min\{k\in\{1,\ldots, n\}: \alpha_{k,\delta} \leq \alpha\} = k_*$. That is, Algorithm \ref{alg:adj_umbrella} reduces to the original NP umbrella algorithm.

Another important question is whether Algorithm \ref{alg:adj_umbrella} can control the true type I error with high probability.  The following condition is assumed for the rest of this section.
\begin{customassumption}{4}\label{assumption:regularity}
The scoring function $\hat{T}$ is trained from a class of functions $\mathcal{T}$ such that the density functions for both $\hat{T}(\tilde{X}^0)$ and $\hat{T}(\tilde{X}^1)$ exist for every $\hat{T} \in \mathcal{T}$. Then, we denote these two densities by  $\tilde f_0^{\hat{T}}$ and  $\tilde f_1^{\hat{T}}$, respectively. Furthermore, $\sup_{\hat{T} \in \mathcal{T}}\|\tilde{f}_0^{\hat{T}}\vee\tilde{f}_1^{\hat{T}}\|_\infty \leq C$ and $\inf_{\hat{T} \in \mathcal{T}}\inf_{z \in \mathcal{D}_{\hat{T}}}\tilde{f}^{\hat{T}}_0(z) > c$ for some positive $c$ and $C$ with probability $1-\delta_2(n_{\emph{b}})$, where $\mathcal{D}_{\hat T}$ is the support of $\tilde{f}^{\hat{T}}_0$ and is a closed interval, and $\delta_2(n_{\emph{b}})$ converges to $0$ as $n_{\emph{b}}$ goes to infinity.
\end{customassumption}

%Note that  Assumption \ref{assumption:regularity} is a suite of a technical assumptions that we take for technical convenience in establishing the next theorem. In particular, we assume the existence of densities $\tilde f_0^{\hat{T}}$ and  $\tilde f_1^{\hat{T}}$, which holds if $\tilde X^0$ and $\tilde X^1$ have densities and $\hat T(\cdot)$ is smooth. Moreover, we assume that with high probability, both the densities are uniformly bounded from above and $\tilde{f}^{\hat{T}}_0(\cdot)$ is bounded uniformly from below.  

Note that  Assumption \ref{assumption:regularity} summarizes assumptions that we make for technical convenience in establishing the next theorem. In particular, we assume the existence of densities $\tilde f_0^{\hat{T}}$ and  $\tilde f_1^{\hat{T}}$, which holds if $\tilde X^0$ and $\tilde X^1$ have densities and $\hat T(\cdot)$ is smooth. Moreover, we assume that with high probability, both the densities are uniformly bounded from above and $\tilde{f}^{\hat{T}}_0(\cdot)$ is bounded uniformly from below.

Recall that in Algorithm \ref{alg:adj_umbrella}, we set $k^* = \min\{k\in\{1, \ldots, n\}:\alpha_{k,\delta} - \hat{D}^+(t_{(k^*)}) \leq \alpha\}$ without an $\varepsilon$ term. Setting $\varepsilon = 0$ intuitively seems reasonable since, when the sample size is small, the sets $\{ k\in\{1, \ldots, n\}:\alpha_{k,\delta} - \hat{D}^+(t_{(k^*)}) \leq \alpha - \varepsilon\}$ and $\{k\in\{1, \ldots, n\}:\alpha_{k,\delta} - \hat{D}^+(t_{(k^*)}) \leq \alpha\}$ agree with high probability, and, when the sample size is large, concentration of random variables takes effect so there is little need for compensation for randomness. Our simulation results further reinforce this intuition. However, we include an $\varepsilon$ term in the next theorem as this is required in our proof for the theory to hold. 

\begin{theorem}\label{thm:adjustment}
Under Assumptions \ref{assumption:mixture}, \ref{assumption:sample_size}, \ref{assumption:separability} and \ref{assumption:regularity}, the classifier $\hat{\phi}_{k^*}(\cdot)$, given by Algorithm \ref{alg:adj_umbrella} with $k^* = \min\{k\in\{1, \ldots, n\}:\alpha_{k,\delta} - \hat{D}^+(t_{(k)}) \leq \alpha - \varepsilon\}$, satisfies
\begin{align*}
    \p\left(R_0(\hat{\phi}_{k^*}) > \alpha\right) \leq \delta + \delta_1(n_{\emph{b}}) + \delta_2(n_{\emph{b}})+ 2e^{-8^{-1}nM^{-2}C^{-2}c^2\varepsilon^2} + 2e^{-8^{-1}n^0_{\emph{e}}M^{-2}\varepsilon^2} + 2e^{-8^{-1}n^1_{\emph{e}}M^{-2}\varepsilon^2}\,,
\end{align*}
in which $n_{\emph{\text{b}}}=|\tilde{\mathcal{S}}_{\emph{\text{b}}}|$, $n = |\tilde{\mathcal{S}}^0_{\emph{\text{t}}}|$, $n^0_{\emph{e}} = |\tilde{\mathcal{S}}^0_{\emph{\text{e}}}|$, and $n^1_{\emph{e}} = |\tilde{\mathcal{S}}^1_{\emph{\text{e}}}|$.
\end{theorem}

Note that the upper bound of $\p\left(R_0(\hat{\phi}_{k^*})>\alpha\right)$ is $\delta$, our violation rate control target, plus a few terms which converge to zero as the sample sizes increase. To establish this inequality, we first exclude the complement of the events described in Assumption \ref{assumption:separability} and \ref{assumption:regularity}. Then, we further restrict ourselves on the event constructed by a Glivenko-Cantelli type inequality where $\hat{D}$ and $D$ only differ by $2^{-1}\varepsilon$. There, the order selection criterion can be written as $k^* = \min\{k\in\{1, \ldots, n\}:\alpha_{k,\delta} - D(t_{(k)}) \leq \alpha - 2^{-1}\varepsilon\}$. The main difficulty of the proof is to handle the randomness of the threshold $t_{(k^*)}$. Unlike the deterministic order $k_*$ in the original NP umbrella algorithm, the new order $k^*$ is stochastic. As such, even when conditioning on $\hat T$, $t_{(k^*)}$ is sill random and cannot be handled as a normal order statistic.  Our solution is to find a high probability deterministic lower bound for $t_{(k^*)}$. To do this, we introduce $c_k$, the $k/n$ quantile of $\tilde{F}^{\hat{T}}_0$, which is a deterministic value if we consider $\hat{T}$ to be fixed. Then, we show that $D(t_{(k)})$ only differs from $D(c_k)$ by $4^{-1}\varepsilon$ for all $k$ and that $\alpha_{k^*,\delta} - D(c_{k^*}) \leq \alpha - 4^{-1}\varepsilon$. Then, we define $k_0 = \min\{k\in\{1, \ldots, n\}: \alpha_{k,\delta} - D(c_k) \leq \alpha - 4^{-1}\varepsilon\}$, which is another deterministic value, given that $\hat{T}$ is considered to be fixed. Then, we find that $k_0 \leq k^*$ and $\alpha_{k_0,\delta} - D(t_{(k_0)}) \leq \alpha$ with high probability. Therefore, $t_{(k_0)}$ is a high probability lower bound for $t_{(k^*)}$. Moreover, $t_{(k_0)}$ is an order statistic with deterministic order (for fixed $\hat{T}$) and thus its distribution can be written as a binomial probability. The fact $\alpha_{k_0,\delta} - D(t_{(k_0)}) \leq \alpha$ combined with Proposition \ref{prop:umbrella} yields that the violation rate of $\hat \phi_{k_0}(\cdot)$ is smaller than $\delta$. The readers are referred to Appendix \ref{proof:adjustment} for a complete proof.

\subsection{Theoretical properties of Algorithm $1^{\#}$  }\label{sec: algorithm 2 property}

In this subsection, we discuss the properties of Algorithm $1^{\#}$. Recall that $m_0^\#\geq m_0$ and $m_1^\# \leq m_1$ in Assumption \ref{assumption:mixture} mean that the corruption levels are ``underestimated.'' As such, Algorithm $1^{\#}$ produces a more conservative result than Algorithm \ref{alg:adj_umbrella}. To see this, note that the only difference between two algorithms is that $(1-m_0)(m_0-m_1)^{-1}$ in Algorithm \ref{alg:adj_umbrella} is replaced with $(1-m_0^\#)(m_0^\# - m_1^\#)^{-1}$ in Algorithm $1^{\#}$. The latter is no larger  than the former, so  we have a threshold in Algorithm $1^{\#}$ larger than or equal to that in Algorithm \ref{alg:adj_umbrella}. 

On the other hand, under Assumption \ref{assumption:mixture}, Algorithm $1^{\#}$ is still less conservative than the original NP umbrella algorithm. To digest this, we first consider the case where the label noise is totally ``ignored'', i.e., $m_0^\# = 1$ and $m_1^\# = 0$. In this case, Algorithm $1^{\#}$ is equivalent to the original NP umbrella algorithm. Then, since usually $m_0^\# < 1$ and $m_1^\# > 0$, Algorithm $1^{\#}$ produces a smaller threshold than the NP original umbrella algorithm. Therefore, Algorithm $1^{\#}$ overcomes, at least partially, the conservativeness issue of the original NP umbrella algorithm.

These insights are formalized in the following lemma.
\begin{lemma}\label{lemma:threshold_of_alg_2}
Under Assumptions \ref{assumption:mixture} - \ref{assumption:sample_size}, $k^*_\# = \min\{k \in \{1, \ldots, n\}: \alpha_{k,\delta} - \hat{D}^+_\#(t_{(k)}) \leq \alpha\}$ in Algorithm $1^{\#}$ exists. Moreover, the order $k^*_\#$ is between $k^*$ and $k_*$, i.e.,  $k^* \leq k^*_\# \leq k_*$. 
\end{lemma}

Next we establish a high probability control on type I error for Algorithm $1^{\#}$. Recall that a high probability control on type I error for Algorithm \ref{alg:adj_umbrella} was established in Theorem \ref{thm:adjustment}. In view of Lemma \ref{lemma:threshold_of_alg_2}, $\hat{\phi}_{k^*_\#}(\cdot)$ produced in Algorithm $1^{\#}$ has a larger threshold, and thus smaller true type I error, than that of $\hat{\phi}_{k^*}(\cdot)$ produced by Algorithm \ref{alg:adj_umbrella}. Then, a high probability control on true type I error of $\hat{\phi}_{k^*_\#}(\cdot)$ naturally follows. This result is summarized in the following corollary. 

\begin{Corollary}\label{cor:adjusted}
Under Assumptions \ref{assumption:mixture} - \ref{assumption:regularity}, the classifier $\hat{\phi}_{k^*_\#}(\cdot)$ given by Algorithm $1^{\#}$ with $k^*_\# = \min\{k \in \{1,\ldots,n\}:\alpha_{k,\delta} - \hat{D}^+_\#(t_{(k)}) \leq \alpha - \varepsilon\}$, satisfies 
\begin{align*}
    \p\left(R_0(\hat{\phi}_{k^*_\#}) > \alpha\right) \leq \delta + \delta_1(n_{\emph{b}}) + \delta_2(n_{\emph{b}})+ 2e^{-8^{-1}nM^{-2}C^{-2}c^2\varepsilon^2} + 2e^{-8^{-1}n^0_{\emph{e}}M^{-2}\varepsilon^2} + 2e^{-8^{-1}n^1_{\emph{e}}M^{-2}\varepsilon^2}\,.
\end{align*}
in which $n_{\emph{\text{b}}}=|\tilde{\mathcal{S}}_{\emph{\text{b}}}|$, $n = |\tilde{\mathcal{S}}^0_{\emph{\text{t}}}|$, $n^0_{\emph{e}} = |\tilde{\mathcal{S}}^0_{\emph{\text{e}}}|$, and $n^1_{\emph{e}} = |\tilde{\mathcal{S}}^1_{\emph{\text{e}}}|$.
\end{Corollary}

\section{Numerical Analysis}\label{sec:sim_and_real_data}

%\textcolor{red}{Shunan, you need to read through this numerical section and edit heavily.  before section 5.3 (maybe in a new paragraph), something about algorithm 2 and comparison with others' methods should be mentioned. }

In this section, we apply Algorithms \ref{alg:adj_umbrella} (known corruption levels) and $1^{\#}$ (unknown corruption levels) on simulated and real datasets, and compare with other methods in the literature.  %In the subroutine that searches for $\alpha_{k, \delta}$ (Algorithm \ref{alg:bi_search} in the Appendix), we set $r = 10^{-5}$. 
 We present the (approximate) type I error violation rates\footnote{Strictly speaking, the observed type I error violation rate is only an approximation to the real violation rate. The approximation is two-fold: i). in each repetition of an experiment, the population type I error is approximated by empirical type I error on a large test set; ii). the violation rate should be calculated based on infinite repetitions of the experiment, but we only calculate it based on a finite number of repetitions. However, such approximation is unavoidable in numerical studies.}
 and the averages of (approximate) true type II errors. Besides the simulations in this section, we have additional simulations in Appendix \ref{appendix:sim}. Furthermore, the violin plots associated with selected simulation are presented in Appendix \ref{appendix:violin_for_section_5}.

As a justification of the minor discrepancy between our theory and implementation, readers can find in Appendix \ref{appendix:numerical} the results for a slightly different implementation of Algorithm \ref{alg:adj_umbrella},  in which $k^{*} = \min\{k\in\{1, \ldots, n\} : \alpha_{k,\delta} - \hat{D}^+(t_{(k)}) \leq \alpha - \varepsilon\}$ and $\varepsilon = 0.0001$. In principle, it is possible that setting $\varepsilon > 0$ will make $k^*$ larger than when $\varepsilon = 0$ as $\{k \in \{1,2,\ldots, n\}, \alpha_{k,\delta} - \hat{D}^+(t_{(k)}) \leq \alpha_{k,\delta} - \varepsilon\}$ is a subset of $\{k \in \{1,2,\ldots, n\}, \alpha_{k,\delta} - \hat{D}^+(t_{(k)}) \leq \alpha_{k,\delta}\}$. This will make the threshold larger and the type I error and the violation rate smaller. However, since $\varepsilon = 0.0001$ is a very small value, its effect on $k^*$ is very minor. In numerical studies, two implementations ($\varepsilon = 0.0001$ in the Appendix vs. $\varepsilon = 0$ in this section) give nearly identical results for all examples. Both implementations generate the same type I errors and type II errors for most (at least $95\%$) cases. Moreover, the difference in violation rates of the two implementations is no larger than a very small number $0.1\delta$.

 %\textcolor{red}{Shunan, the Appendix D.2 just contains violin plots for Examples 1-4 and email spam data.  Either you put all updpated results in there, or just say that the violin plots are only for a part of the examples.}

\subsection{Simulation}\label{sec:simulation}

\subsubsection{Algorithm \ref{alg:adj_umbrella}.}

We present three distributional settings for Algorithm \ref{alg:adj_umbrella} (known $m_0$ and $m_1$). In each setting, $2N$ observations are generated as a training sample, of which half are from the \textit{corrupted} class $0$ and half from the \textit{corrupted} class $1$. The number $N$ varies from $200$ to $2{,}000$. To approximate the true type I and II errors, we generate $20{,}000$ \textit{true} class $0$ observations and $20{,}000$ \textit{true} class $1$ observations as the evaluation set. For each distribution and sample size combination, we repeat the procedure $1{,}000$ times. Algorithm \ref{alg:adj_umbrella} (``adjusted") and the original NP umbrella algorithm (``original") are both applied, paired with different base algorithms.  

\begin{simulation}[Gaussian Distribution]\label{sim:gmm}
Let $X^0 \sim \mathcal{N}(\mu_0, \Sigma)$ and $X^1 \sim \mathcal{N}(\mu_1, \Sigma)$, where $\mu_0 = (0,0,0)^\top, \mu_1 = (1,1,1)^\top$ and
\begin{align*}
    \Sigma = 
    \begin{pmatrix}
    2 & -1 & 0 \\
    -1 & 2 & -1 \\
    0 & -1 & 2
    \end{pmatrix}\,,
\end{align*}
and the base algorithm is linear discriminant analysis (LDA). For different $(m_0, m_1, \alpha, \delta)$ combinations, the (approximate) type I error violation rates  and the averages of (approximate) true type II errors generated by Algorithm \ref{alg:adj_umbrella} are reported in Tables \ref{table:violation_rate_sim_1} and \ref{table:violation_rate_sim_2}, respectively. %The (approximate) true type I and II errors are presented in Figures \ref{fig:gmm0} and \ref{fig:gmm1}, respectively. The two rows in each figure respectively correspond to the $m_1=0.95, m_1=0.05$ and $m_0=0.85, m_1=0.15$ settings.

\end{simulation}

%\end{simulation}

\begin{table}[h]
\centering
\caption{(Approximate) type I error violation rates over $1{,}000$ repetitions for Simulation \ref{sim:gmm}. Standard errors ($\times10^{-3}$) in parentheses.}
\vspace{5 pt}
\footnotesize 
\begin{tabular}{|l|l|l|l|l|l|l|l|l|}
\hline
\multirow{2}{*}{$N$} & \multicolumn{2}{l|}{\begin{tabular}[c]{@{}l@{}}$m_0 = .95, m_1 = .05$\\ $\alpha = .05, \delta = .05$\end{tabular}} & \multicolumn{2}{l|}{\begin{tabular}[c]{@{}l@{}}$m_0 = .9, m_1 = .1$\\ $\alpha = .05, \delta = .05$\end{tabular}} & \multicolumn{2}{l|}{\begin{tabular}[c]{@{}l@{}}$m_0 = .95, m_1 = .05$\\ $\alpha = .1, \delta = .1$\end{tabular}} & \multicolumn{2}{l|}{\begin{tabular}[c]{@{}l@{}}$m_0 = .9, m_1 = .1$\\ $\alpha = .1, \delta = .1$\end{tabular}} \\ \cline{2-9} 
 & adjusted & original & adjusted & original & adjusted & original & adjusted & original \\ \hline
$200$ & $.026$ $(5.03)$ & $.001$ $(1.00)$ & $.033$ $(5.65)$ & $0$ $(0)$ & $.078$ $(8.84)$ & $.003$ $(1.73)$ & $.073$ $(8.23)$ & $0$ $(0)$ \\ \hline
$500$ & $.031$ $(5.40)$ & $0$ $(0)$ & $.046$ $(6.63)$ & $0$ $(0)$ & $.090$ $(9.05)$ & $.001$ $(1.00)$ & $.085$ $(8.82)$ & $0$ $(0)$ \\ \hline
$1{,}000$ & $.038$ $(5.97)$ & $0$ $(0)$ & $.049$ $(6.83)$ & $0$ $(0)$ & $.105$ $(9.70)$ & $0$ $(0)$ & $.081$ $(8.63)$ & $0$ $(0)$ \\ \hline
$2{,}000$ & $.053$ $(6.96)$ & $0$ $(0)$ & $.046$ $(6.63)$ & $0$ $(0)$ & $.087$ $(8.92)$ & $0$ $(0)$ & $.099$ $(9.45)$ & $0$ $(0)$ \\ \hline
\end{tabular}
\label{table:violation_rate_sim_1}
\end{table}

\begin{table}[h]
\caption{Averages of (approximate) true type II errors over $1{,}000$ repetitions for Simulation \ref{sim:gmm}. Standard errors ($\times10^{-3}$) in parentheses.}
\vspace{5 pt}
\footnotesize 
\centering
\begin{tabular}{|l|l|l|l|l|l|l|l|l|}
\hline
\multirow{2}{*}{$N$} & \multicolumn{2}{l|}{\begin{tabular}[c]{@{}l@{}}$m_0 = .95, m_1 = .05$\\ $\alpha = .05, \delta = .05$\end{tabular}} & \multicolumn{2}{l|}{\begin{tabular}[c]{@{}l@{}}$m_0 = .9, m_1 = .1$\\ $\alpha = .05, \delta = .05$\end{tabular}} & \multicolumn{2}{l|}{\begin{tabular}[c]{@{}l@{}}$m_0 = .95, m_1 = .05$\\ $\alpha = .1, \delta = .1$\end{tabular}} & \multicolumn{2}{l|}{\begin{tabular}[c]{@{}l@{}}$m_0 = .9, m_1 = .1$\\ $\alpha = .1, \delta = .1$\end{tabular}} \\ \cline{2-9} 
 & adjusted & original & adjusted & original & adjusted & original & adjusted & original \\ \hline
$200$ & $.685$ $(7.16)$ & $.706$ $(4.65)$ & $.697$ $(7.06)$ & $.826$ $(3.54)$ & $.333$ $(3.93)$ & $.403$ $(3.56)$ & $.369$ $(4.93)$ & $.537$ $(4.03)$ \\ \hline
$500$ & $.481$ $(4.08)$ & $.590$ $(2.99)$ & $.512$ $(4.92)$ & $.743$ $(2.79)$ & $.249$ $(1.94)$ & $.307$ $(1.83)$ & $.257$ $(2.21)$ & $.436$ $(2.48)$ \\ \hline
$1{,}000$ & $.396$ $(2.53)$ & $.534$ $(2.19)$ & $.387$ $(2.37)$ & $.663$ $(1.68)$ & $.218$ $(1.18)$ & $.287$ $(1.22)$ & $.213$ $(1.01)$ & $.381$ $(1.28)$ \\ \hline
$2{,}000$ & $.350$ $(1.51)$ & $.491$ $(1.45)$ & $.371$ $(1.99)$ & $.651$ $(1.45)$ & $.201$ $(.76)$ & $.268$ $(.77)$ & $.205$ $(.87)$ & $.375$ $(1.01)$ \\ \hline
\end{tabular}
\label{table:violation_rate_sim_2}
\end{table}

\begin{simulation}[Uniform Distribution within Circles]\label{sim:two_circle}
Let $X^0$ and $X^1$ be uniformly distributed within unit circles respectively centered at $(0,0)^\top$ and $(1,1)^\top$. The base algorithm is logistic regression.  We only report (approximate) type I error violation rates and the averages of (approximate) true type II errors generated by Algorithm \ref{alg:adj_umbrella} for one combination ($m_0 = .95$, $m_1 = .05$, $\alpha = .1$ and $\delta = .1$) in Table \ref{table:sim_u}.
%The (approximate) true type I and II errors are presented in Figures \ref{fig:two_circle0} and \ref{fig:two_circle1}, respectively.
\end{simulation}

\begin{table}[h]
\caption{(Approximate) type I error violation rates, and averages of (approximate) true type II errors over $1{,}000$ repetitions for  Simulation \ref{sim:two_circle} ($m_0 = .95$, $m_1 = .05$, $\alpha = .1$ and $\delta = .1$). Standard errors ($\times10^{-3}$) in parentheses.}
\vspace{5 pt}
\centering
\begin{tabular}{|l|l|l|l|l|}
\hline
\multirow{2}{*}{$N$} & \multicolumn{2}{l|}{\begin{tabular}[c]{@{}l@{}}(approximate) \\ violation rate\end{tabular}} & \multicolumn{2}{l|}{\begin{tabular}[c]{@{}l@{}}averages of \\ (approximate) true \\ type II errors\end{tabular}} \\ \cline{2-5}  
 & adjusted & original & adjusted & original \\ \hline
$200$ & $.079$ $(8.53)$ & $.006$ $(2.44)$ & $.164$ $(2.77)$ & $.226$ $(3.35)$  \\ \hline
$500$ & $.086$ $(8.87)$ & $.001$ $(1.00)$ & $.123$ $(.92)$ & $.161$ $(.80)$  \\ \hline
$1{,}000$ & $.085$ $(8.82)$ & $0$ $(0)$ & $.109$ $(.61)$ & $.151$ $(.58)$  \\ \hline
$2{,}000$ & $.085$ $(8.82)$ & $0$ $(0)$ & $.101$ $(.44)$ & $.142$ $(.39)$  \\ \hline
\end{tabular}
\label{table:sim_u}
\end{table}

%\end{simulation}

\begin{simulation}[T Distribution]\label{sim:two_t}
Let $X^0$ and $X^1$ be t-distributed with shape matrix $\Sigma$, which was specified in Simulation \ref{sim:gmm}, $4$ degrees of freedom, and centered at $(0,0,0)^\top$ and $(1,1,1)^\top$ respectively. 
%, shape matrix $\Sigma$ and degree of freedom $4$, where $\Sigma$ is the same as in \ref{sim:gmm}.
The base algorithm is LDA. Similar to the previous simulation, we only report (approximate) type I error violation rates and the averages of (approximate) true type II errors generated by Algorithm \ref{alg:adj_umbrella} for one combination ($m_0 = .95$, $m_1 = .05$, $\alpha = .1$ and $\delta = .1$) in Table \ref{table:sim_t}.
%The (approximate) true type I and II errors are presented in Figures \ref{fig:two_t0} and \ref{fig:two_t1}, respectively.
\end{simulation}

\begin{table}[h]
\caption{(Approximate) type I error violation rates, and averages  of (approximate) true type II errors over $1{,}000$ repetitions for  Simulation \ref{sim:two_t}  ($m_0 = .95$, $m_1 = .05$, $\alpha = .1$ and $\delta = .1$). Standard errors ($\times10^{-3}$) in parentheses.}
\vspace{5 pt}
\centering
\begin{tabular}{|l|l|l|l|l|}
\hline
\multirow{2}{*}{$N$} & \multicolumn{2}{l|}{\begin{tabular}[c]{@{}l@{}}(approximate) \\ violation rate\end{tabular}} & \multicolumn{2}{l|}{\begin{tabular}[c]{@{}l@{}}average of \\ (approximate) true \\ type II errors\end{tabular}} \\ \cline{2-5}  
 & adjusted & original & adjusted & original \\ \hline
$200$ & $.068$ $(7.96)$ & $.008$ $(2.82)$ & $.526$ $(5.67)$ & $.575$ $(4.32)$ \\ \hline
$500$ & $.085$ $(8.82)$ & $.002$ $(1.41)$ & $.398$ $(3.32)$ & $.472$ $(2.59)$ \\ \hline
$1{,}000$ & $.090$ $(9.05)$ & $0$ $(0)$ & $.345$ $(2.07)$ & $.432$ $(1.78)$ \\ \hline
$2{,}000$ & $.093$ $(9.19)$ & $0$ $(0)$ & $.314$ $(1.24)$ & $.401$ $(1.18)$ \\ \hline
\end{tabular}
\label{table:sim_t}
\end{table}

%\end{simulation}

%All the three simulation studies presented above show that the empirical type I error of the original umbrella algorithm is under control with high probability. However, the simulation confirms the conservativeness of the original umbrella algorithm. 

%A complete report of the (approximate) true type I and II errors for these three simulation models are presented in figures included in Appendix \ref{appendix:violin_for_section_5}. 

The results from Simulations \ref{sim:gmm}-\ref{sim:two_t} confirm that the original NP umbrella algorithm is overly conservative on type I error when there is label noise in the training data, resulting in type I error violation rates (close to) 0 in all settings.    
In contrast, the label-noise-adjusted Algorithm \ref{alg:adj_umbrella} has type I errors controlled at the specified level with high probability and achieves much better type II errors.

\subsubsection{Algorithm $1^{\#}$.} In this section, we show numerically that under the NP paradigm, the ``under-estimates" of corruption levels serve Algorithm $1^{\#}$ well, while "over-estimates" do not.

\begin{simulation}\label{sim:underestimate}
The distributional setting is the same as in Simulation \ref{sim:gmm}. Different combinations of $m_0^\#$ and $m_1^\#$ are used. the (approximate) type I error violation rates  and the averages of (approximate) true type II errors generated by Algorithm $1^{\#}$ for one combination ($m_0 = .95$, $m_1 = .05$, $\alpha = .1$ and $\delta = .1$) are reported in Tables \ref{tbl:underestimate_I} and \ref{tbl:underestimate_II}.

\begin{table}[h]
\centering
\caption{(Approximate) type I error violation rates over $1{,}000$ repetitions for Simulation \ref{sim:underestimate}. Standard errors ($\times 10^{-3}$) in parentheses.}\label{tbl:underestimate_I}
\begin{tabular}{|l|l|l|l|l|}
\hline
$N$ & \begin{tabular}[c]{@{}l@{}}$m_0^\# = .93,$\\ $m_1^\# = .07$\end{tabular} &  \begin{tabular}[c]{@{}l@{}}$m_0^\# = .95,$\\ $m_1^\# = .05$\end{tabular} &  \begin{tabular}[c]{@{}l@{}}$m_0^\# = .97,$\\ $m_1^\# = .03$\end{tabular} & 
\begin{tabular}[c]{@{}l@{}}original\end{tabular} \\ \hline
$200$ & $.136(10.85)$  & $.078(8.48)$  & $.055(7.21)$ &  $.003(1.73)$ \\ \hline
$500$ & $.218(13.06)$  & $.090(9.05)$ & $.038(6.05)$ & $.001(1.00)$ \\ \hline
$1,000$ & $.324(14.81)$  & $.105(9.70)$ & $.012(3.44)$ & $0(0)$ \\ \hline
$2,000$ & $.462(15.77)$  & $.087(8.92)$ & $.005(2.23)$& $0(0)$ \\ \hline
\end{tabular}
\end{table}

\begin{table}[h]
\caption{(Approximate) type II error violation rates over $1{,}000$ repetitions for  Simulation \ref{sim:underestimate}. Standard errors ($\times 10^{-3}$) in parentheses.}\label{tbl:underestimate_II}
\centering
\begin{tabular}{|l|l|l|l|l|}
\hline
$N$ & \begin{tabular}[c]{@{}l@{}}$m_0^\# = .93,$\\ $m_1^\# = .07$\end{tabular} &  \begin{tabular}[c]{@{}l@{}}$m_0^\# = .95,$\\ $m_1^\# = .05$\end{tabular} &  \begin{tabular}[c]{@{}l@{}}$m_0^\# = .97,$\\ $m_1^\# = .03$\end{tabular} & \begin{tabular}[c]{@{}l@{}}original\end{tabular} \\ \hline
$200$ & $.287(3.43)$  & $.333(3.92)$  & $.373(4.62)$ & $.403(3.56)$ \\ \hline
$500$ & $.215(1.61)$  & $.249(1.94)$ &  $.285(2.22)$ & $.307(1.83)$ \\ \hline
$1,000$ & $.189(1.02)$  & $.218(1.18)$ &  $.250(1.37)$ & $.287(1.22)$ \\ \hline
$2,000$ & $.174(.65)$  & $.201(.76)$  & $.230(.86)$ & $.268(.77)$ \\ \hline
\end{tabular}
\end{table}

\end{simulation}

%\textcolor{red}{The algorithm should be applied to the current simulation 8. The other simulations should be moved to the appendix.  }

The second to the last column in Table \ref{tbl:underestimate_I} confirms that, using strict under-estimates of corruption levels (i.e., $m_0^{\#} > m_0$ and $m_1^{\#} < m_1$), the type I error control objective is satisfied.  Note that we also include the strict over-estimate scenarios in the second  column (i.e., $m_0^{\#} < m_0$ and $m_1^{\#} > m_1$), where we see that the type I violation rates exceed the target $\delta$. Hence the under-estimate requirement in the theory part is not merely for technical convenience. Table \ref{tbl:underestimate_II} confirms that the using strict under-estimates would lead to higher type II errors than using the true corruption levels. This is a necessary price to pay for not knowing the exact levels, but still it is better than totally ignoring the label corruption and applying the original NP umbrella algorithm.  

We state again that in this work, we rely on domain experts to supply under-estimates of corruption levels.  In the literature, there are existing estimators. For example, we implement estimators proposed by \cite{liu2015classification} in Simulations \ref{sim:uniform_unknown_flip_rate} and \ref{sim:normal_unknown_flip_rate} in Appendix \ref{appendix:sim}. There, we would see that those estimators do not help Algorithm $1^{\#}$ achieve the type I error control objective. But this is not a problem with these estimators themselves. Even ``oracle" consistent and unbiased estimators that center at $m_0$ and $m_1$ do not serve the purpose either, as revealed in  Simulation \ref{sim:normal_estimator} in Appendix \ref{appendix:sim}.  %As Assumption \ref{assumption:mixture} suggests, under the NP paradigm, ideal estimators of corruption levels are ``under-estimators,'', i.e., $m_0^\# \geq m_0$ and $m_1^\# \leq m_1$.
As expected, given our discussion about the need for under-estimates of the corruption levels (i.e., $m_0^\# \geq m_0$ and $m_1^\# \leq m_1$), Algorithm $1^{\#}$ performs poorly using these unbiased estimates. It could be an interesting topic for future research to identify an efficient method for producing biased estimates which will satisfy (with high probability) the bounds necessary to ensure correct type 1 error control.

\subsubsection{Benchmark Algorithms.}

In the next simulation, we apply existing state-of-the-art algorithms that perform classification on data with label noise. In particular, we apply the backward loss correction algorithm in \cite{patrini2017making} and the T-revision method in \cite{xia2019anchor}. Since we focus on the NP paradigm, we will report the same (approximate) type I error violation rates and averages of (approximate) true type II errors as for our own methods.

%In the next simulation, we apply existing state-of-the-art algorithms that perform classification on data with label noise. In particular, we apply the backward loss correction algorithm in \cite{patrini2017making} and the T-revision method in \cite{xia2019anchor}. Since we focus on the NP paradigm, we will report the same (approximate) type I error violation rates and averages of (approximate) true type II errors as for our own methods. %Furthermore, we provide comparisons of benchmark algorithms with Algorithms \ref{alg:adj_umbrella} and $1^{\#}$.

%Though the two benchmark algorithms are not built for NP paradigm, we report the violation rates of the type I errors at the control level $\alpha = 0.1$ and $\delta = 0.1$. 

%Moreover, we provide the performance of our algorithm at the same control level for comparison. 

%\textcolor{red}{Shunan, you did not put the new bib file in it. I used an old version. So the newly added reference does not show up properly.  Please use the new version next time.}

%\textcolor{purple}{change the simulation description to the standard language, and get rid of the 0. in $m_0$ etc.}

\begin{simulation}\label{sim:compare_gmm}
The distributional setting is the same as in Simulation \ref{sim:gmm}. The (approximate) type I error violation rates and averages of (approximate) true type II errors generated by benchmark algorithms for one combination ($m_0 = .95$, $m_1 = .05$, $\alpha = .1$ and $\delta = .1$) are reported in Table \ref{tbl:compare_gmm_1} in the main and Table \ref{tbl:compare_gmm_2} in Appendix \ref{appendix:table}, respectively. %Furthermore, results generated by Algorithms \ref{alg:adj_umbrella} and $1^{\#}$ (with LDA as the base algorithm) are provided for comparison. %\textcolor{red}{Is simulation 5 about algorithm 2 or algorithm 1?  Currently, I write it for Algorithm 2, but if it is for algorithm 1, please change it accordingly.  Now, we have two algorithms, it is critical that we specify what if the algorithm used in each example. Also, if Simulation 5 is for algorithm1, probably should put it (together with the discussion about Simulations 9 and 10) in section 5.1.1 instead of here.} 
%\textcolor{purple}{Shunan, Algorithm 2 is not in ttables 7 and 8.}

%\textcolor{purple}{In tables 7 and 8, write down the parameter combination, so that the readers do not need to search for a comparison target all over the place. Also, please add the output from the algorithm 1 table for easy comparison.}

\begin{table}[h]
\centering
\caption{(Approximate) type I error violation rates over $1{,}000$ repetitions for Simulation \ref{sim:compare_gmm} ($m_0 = .95$, $m_1 = .05$, $\alpha = .1$ and $\delta = .1$). Standard errors ($\times 10^{-3}$) in parentheses.}\label{tbl:compare_gmm_1}
\begin{tabular}{|l|l|l|l|l|}
\hline
\multirow{2}{*}{algorithms} & \multicolumn{4}{l|}{$N$} \\ \cline{2-5} 
 & $200$ & $500$ & $1{,}000$ & $2{,}000$ \\ \hline
T-revision & $.713(14.31)$  & $.675(14.82)$ & $.651(15.08)$  & $.621(15.35)$ \\ \hline
\begin{tabular}[c]{@{}l@{}}backward loss correction\\ (known corruption levels)\end{tabular} & $.994(2.44)$ & $.977(4.74)$ & $.770(13.31)$ & $.127(10.53)$ \\ \hline
\begin{tabular}[c]{@{}l@{}}backward loss correction\\ (unknown corruption levels)\end{tabular} & $.984(3.97)$ & $.793(5.20)$ & $.320(6.89)$ & $.131(3.60)$ \\ \hline

% \begin{tabular}[c]{@{}l@{}}Algorithm \ref{alg:adj_umbrella}\\ (known corruption levels)\end{tabular} & $.078(8.84)$ & $.090(9.05)$ & $.105(9.70)$ & $.087(8.92)$ \\ \hline

% \begin{tabular}[c]{@{}l@{}}Algorithms $1^{\#}$\\ (unknown corruption levels)\end{tabular} & $.530(15.79)$ & $.758(13.51)$ & $.953(6.70)$ & $.957(6.42)$ \\ \hline
\end{tabular}
\end{table}
\end{simulation}

%\textcolor{purple}{please edit the next two paragraphs, given our conversation today.}

In Simulation \ref{sim:compare_gmm}, the benchmark algorithms fail to control the true type I error with the pre-specified high probability. 
%The only algorithm listed in Table \ref{tbl:compare_gmm_1} that has the desired type I error control is Algorithm \ref{alg:adj_umbrella}, i.e., label-noise-adjusted algorithm with known corruption levels. Algorithm $1^{\#}$, without under-estimates of the corruption levels, also fails the purpose for reasons discussed in the previous section.  
This is understandable, as none of the benchmark algorithms have $\alpha$ or $\delta$ as inputs. As such, these algorithms, unlike Algorithms \ref{alg:adj_umbrella} or $1^{\#}$, are not designed for the NP paradigm.

%In Simulation \ref{sim:compare_gmm}, the benchmark algorithms fail to control the true type I error with the pre-specified high probability. 
%This is understandable, as none of the benchmark algorithms have $\alpha$ or $\delta$ as inputs. As such, these algorithms are not designed to adapt to the NP paradigm.

%Further evidence is offered in the real data analysis.  

%\textcolor{red}{Shunan, there is another  example missing from the simulation. We haven't shown with underestimates $m_0^{\#}$ and $m_1^{\#}$, does Algorithm 2 still control the type I error?  This example fits the Section 5.1.2    }

\subsection{Real Data Analysis}\label{sec:realdata}

%\subsubsection{Email Spam Dataset.}\label{sec:email}

We analyze a canonical email spam dataset \citep{hopkins1999spambase}, which consists of $4{,}601$ observations including $57$ attributes describing characteristics of emails and a $0-1$ class label. Here, $1$ represents \textit{spam} email while $0$ represents \textit{non-spam}, and the type I/II error is defined accordingly. The labels in the dataset are all assumed to be correct.  

We create corrupted labels according to the class-conditional noise model.  Concretely, we flip the labels of true class $0$ observations with probability $r_0$ and flip the labels of true class $1$ observations with probability $r_1$. %Here, the subtlety is 
Note that $m_0$ and $m_1$ are $\p(Y = 0\mid\tilde{Y}=0)$ and $\p(Y=0\mid\tilde{Y}=1)$, respectively, while $r_0 = \p(\tilde{Y} = 1\mid Y=0)$ and $r_1 = \p(\tilde{Y} = 0\mid Y=1)$. 
%
\begin{comment}
Let $p_0 = \p(Y=0)$, then Bayes theorem implies that 

%That is, estimate $\p(Y=0)$, or $p_0$ in abbreviation, by the proportion of true class $0$ observations in the data and solve the equations

\begin{equation*}
m_0 = \frac{p_0(1-r_0)}{p_0(1-r_0)+(1-p_0)r_1}\,\text{ and }\, m_1 = \frac{p_0r_1}{p_0r_0+(1-p_0)(1 - r_1)}\,.
\end{equation*}
Solving the above equations for $r_0$ and $r_1$ yields
\begin{equation*}
r_0 = \frac{(m_0-p_0)m_1}{(m_0-m_1)p_0}\,\text{ and }\, r_1 = \frac{(1-m_0)(p_0-m_1)}{(m_0-m_1)(1-p_0)}\,.
\end{equation*}
The proportions $r_0$ and $r_1$ should be in $(0,1)$.  Because  $0 < m_1 < m_0 < 1$ by Assumption \ref{assumption:mixture}, it suffices to have  $m_0 > p_0$ and $p_0 > m_1$. But note that,
\begin{align*}
p_0 = m_0\p(\tilde{Y} = 0) + m_1\p(\tilde{Y} = 1) < m_0\left(\p(\tilde{Y} = 0) + \p(\tilde{Y} = 1)\right) = m_0 \,,
\end{align*}
where the inequality follows from Assumption \ref{assumption:mixture}. Similarly,
\begin{align*}
p_0 = m_0\p(\tilde{Y} = 0) + m_1\p(\tilde{Y} = 1) \geq m_1\left(\p(\tilde{Y} = 0) + \p(\tilde{Y} = 1)\right) = m_1\,.
\end{align*}
\end{comment}
In our analysis, we choose $m_0 = 0.95$ and $m_1 = 0.05$, which implies setting $r_0 = 0.032$ and $r_1 = 0.078$ \footnote{This is an application of the Bayes theorem with $\p(Y=0)$ estimated to be $0.610$, which is the proportion of class $0$ observations in the whole dataset.}. 
%To compute $r_0$ and $r_1$, we take $p_0=0.610$, which is the proportion of class $0$ observations in the whole dataset. Then it follows that  $r_0 = 0.060$ and $r_1 = 0.161$.  
For each training and evaluation procedure, we split the data by stratified sampling into training and evaluation sets. Specifically, $20\%$ of the true class $0$ observations and $20\%$ of the true class $1$ observations are randomly selected to form the training dataset, and the rest of the observations form the evaluation dataset. In total, the training set contains $921$ observations and the evaluation set contains $3{,}680$ observations. The larger evaluation set is reserved to better approximate (population-level) true type I/II error. We leave the evaluation data untouched, but randomly flip the training data label according to the calculated $r_0$ and $r_1$. Four base algorithms are coupled with the original and new NP umbrella algorithms, with $\alpha = \delta = 0.1$.  We repeat the procedure $1{,}000$ times. %and obtain an array of (approximate) true type I and type II errors.  

The (approximate) type I error violation rates and averages of (approximate) true type II errors generated by Algorithm \ref{alg:adj_umbrella} and the original NP umbrella algorithm are summarized in Table \ref{table:real_data}. %Furthermore, a complete summary of (approximate) true type I and II errors are presented in Appendix \ref{appendix:violin_for_section_5}. 
Similar to the simulation studies, we observe that Algorithm \ref{alg:adj_umbrella}  correctly controls type I error at the right level, while the original NP umbrella algorithm is significantly overly conservative on type I error, and consequently has much higher type II error.  We also summarize the results generated by Algorithm $1^{\#}$ in Tables \ref{tbl:email_under_1} and \ref{tbl:email_under_2}. Clearly, while strict under-estimates lead to higher type II errors than using exact corruption levels, the type I error control objective is achieved, and the type II error is better than just ignoring label corruption and applying the original NP umbrella algorithm.   

% We also implement Algorithm $1^{\#}$ with estimators \cite{liu2015classification} of corruption levels in Table \ref{table:real_data_est} in Appendix \ref{appendix:table}. There we observe that the type I error violation rates far exceed the targeted $\delta = 0.1$. These results are similar to what we observe in simulation studies.  

%\textcolor{red}{Now need to be careful about the wording, which label-noise-adjusted algorithm are we talking about here?}
%a proper high probability control on (approximate) true type I errors regardless of base algorithms used. Moreover, as expected,  the (approximate) true type II errors of the label-noise-adjusted umbrella algorithm are lower than those of the original NP umbrella algorithm.

\begin{table}[h]
\caption{(Approximate) type I error violation rates, and averages  of (approximate) true type II errors by Algorithm \ref{alg:adj_umbrella} and original NP umbrella algorithm over $1{,}000$ repetitions for the email spam data. Standard errors ($\times10^{-3}$) in parentheses.}
\vspace{5 pt}
\centering
\begin{tabular}{|l|l|l|l|l|}
\hline
\multirow{2}{*}{} & \multicolumn{2}{l|}{\begin{tabular}[c]{@{}l@{}}(approximate) \\ violation rate\end{tabular}} & \multicolumn{2}{l|}{\begin{tabular}[c]{@{}l@{}}average of \\ (approximate) true \\ type II errors\end{tabular}} \\ \cline{2-5} 
 & adjusted & original & adjusted & original \\ \hline
penalized logistic regression & $.082(8.68)$ & $0(0)$ & $.205(2.65)$ & $.272(2.71)$ \\ \hline
linear discriminant analysis & $.096(9.32)$ & $0(0)$ & $.226(3.05)$ & $.314(2.77)$ \\ \hline
support vector machine & $.093(9.19)$ & $.004(2.00)$ & $.183(3.15)$ & $.218(1.93)$ \\ \hline
random forests & $.080(8.58)$ & $0(0)$ & $.120(1.13)$ & $.152(1.54)$ \\ \hline
\end{tabular}
\label{table:real_data}
\end{table}

\begin{table}[h]
\caption{(Approximate) type I error violation rates by Algorithm $1^{\#}$ over $1{,}000$ repetitions for the email spam data. Standard errors ($\times10^{-3}$) in parentheses.}
\vspace{5 pt}
\centering
\begin{tabular}{|l|l|l|l|l|}
\hline
 & \begin{tabular}[c]{@{}l@{}}$m^\#_0 = 0.93,$\\ $m^\#_1 = 0.07$\end{tabular} &
 \begin{tabular}[c]{@{}l@{}}$m^\#_0 = 0.95,$\\ $m^\#_1 = 0.05$\end{tabular} &
 \begin{tabular}[c]{@{}l@{}}$m^\#_0 = 0.97,$\\ $m^\#_1 = 0.03$\end{tabular} &  \begin{tabular}[c]{@{}l@{}}original\end{tabular} \\ \hline
penalized logistic regression & $.231(13.33)$ & $.082(8.68)$ & $.028(5.22)$ & $0(0)$ \\ \hline
linear discriminant analysis & $.223(13.17)$ & $.096(9.32)$ & $.023(4.74)$ & $0(0)$ \\ \hline
support vector machine & $.220(13.11)$ & $.093(9.19)$ & $.026(5.03)$ & $.004(2.00)$ \\ \hline
random forest & $.238(13.47)$ & $.080(8.58)$ & $.019(4.32)$ & $0(0)$ \\ \hline
\end{tabular}
\label{tbl:email_under_1}
\end{table}

\begin{table}[h]
\caption{Averages of (approximate) true type II errors by Algorithm $1^{\#}$ over $1{,}000$ repetitions for the email spam data. Standard errors ($\times10^{-3}$) in parentheses.}
\vspace{5 pt}
\centering
\begin{tabular}{|l|l|l|l|l|}
\hline
 & \begin{tabular}[c]{@{}l@{}}$m^\#_0 = 0.93,$\\ $m^\#_1 = 0.07$\end{tabular} & \begin{tabular}[c]{@{}l@{}}$m^\#_0 = 0.95,$\\ $m^\#_1 = 0.05$\end{tabular} & \begin{tabular}[c]{@{}l@{}}$m^\#_0 = 0.97,$\\ $m^\#_1 = 0.03$\end{tabular} & \begin{tabular}[c]{@{}l@{}}original\end{tabular} \\ \hline
penalized logistic regression & $.165(2.04)$ & $.205(2.65)$ & $.254(3.10)$ & $.272(2.71)$ \\ \hline
linear discriminant analysis & $.213(2.54)$ & $.226(3.05)$ & $.314(3.37)$ & $.314(2.77)$ \\ \hline
support vector machine & $.138(1.20)$ & $.183(3.15)$ & $.199(2.11)$ & $.218(1.93)$ \\ \hline
random forest & $.102(.78)$ & $.120(1.13)$ & $.143(1.41)$ & $.152(1.54)$ \\ \hline
\end{tabular}
\label{tbl:email_under_2}
\end{table}

To make a comparison, we also apply the loss correction algorithm in \cite{patrini2017making} and the T-revision method in \cite{xia2019anchor} to the email spam data, with results summarized in Table \ref{tbl:real_data_benchmark} in Appendix \ref{appendix:table}. Since these benchmark algorithms are not designed for the NP paradigm, as discussed in Section \ref{sec:simulation}, none of the (approximate) true type I error violation rates are controlled as we desire. In addition to the email spam data, we also apply Algorithm \ref{alg:adj_umbrella} to the CIFAR10 dataset \citep{krizhevsky2009learning}  and successfully have the type I error controlled (Appendix \ref{appendix:cifar}).
%\textcolor{purple}{Shunan, please double check the table 5.2.1. Also, make Tables 9, 10... consistent with earlier tables.}

%\textcolor{purple}{Try put Table 10 and Table 11 in the Appendix in the proper section. Probably a subsection in Appendix D.}

%\textcolor{purple}{Try put Cifar dataset into the appendix. Change the wordings as necessary after you move this subsection.   Right after the end of the email spam analysis, mention that we have another dataset in the Appendix. }

%\textcolor{red}{Shunan, still need to do the CIFAR 10 data. Also, need to add the competitor's methods for the email spam data. For these datasets, applying algorithms 1 from our side probably suffices. BUT of course, you still need to compare with others' methods.    }

%\textcolor{purple}{Shunan, add how data were divided...... a similar description as you wrote for the email spam data. And explanation of why we only use neural network algorithm 1}

\section{Discussion}

Under the NP paradigm, we developed the first label-noise-adjusted umbrella algorithms.  There are several interesting directions for future research. First, we can consider a more complex noise model in which the corruption levels depend on both the class and features.  Another direction is to consider data-driven ``under-estimates" of the corruption levels in the class-conditional noise model and develop (distributional) model-specific adjustment algorithms. For instance, we can adopt the linear discriminant analysis model, i.e., $X^0\sim \mathcal{N}(\mu_0, \Sigma)$ and $X^1\sim \mathcal{N}(\mu_1, \Sigma)$. %which are mutually irreducible to make the class-conditional noise model identifiable.   

%Besides possible extension to the algorithm, there might be other promising future works. Throughout our research on this topic, we tried a `subsampling' version of the algorithm. The idea is to randomly take a subsample of $\mathcal{S}_{\text{t}}$ multiple times so there would be a copy subsample free of corruption, i.e., comprised only of true class $0$ observations with high probability. Then, if one can identify this copy and use it as the candidate sets for threshold, it is equivalent to dealing with a noise-free setting. However, the difficulty remains in how such a subsample would be identified. Given the mixture model as in Assumption \ref{assumption:mixture}, if the two true underlying distributions differ a lot, then a corruption-free distribution would be easy to identify since outliers will greatly disturb some statistics such as sample mean or variance. However, this is not the case when the two underlying distributions do not differ much. If one could find a good indicator, this subsampling approach could still be promising.

%\textcolor{red}{We can probably talk about parametric assumptions... and other generalization..   }

\section{Acknowledgements}
We would like to Dr. Xiao Han from the University of Science and Technology of China and Dr. Min Zhou from United International College for inspirational discussions. This work was partially supported by U.S. NSF grant DMS 2113500. 

\bibliographystyle{ECA_jasa}
\bibliography{npcorruption}

\newpage
\setcounter{page}{1}
\appendix
\appendixpage

\section{Summary of sampling scheme}\label{sec:sampling_scheme_summary}

This section summarizes our sampling scheme and related notations for the readers' convenience. First, to review the NP paradigm and to make a contrast with the corrupted setting, we introduced the notation for uncorrupted samples: let $\mathcal{S}^0 = \{X_j^0\}_{j=1}^{M_0}$ and $\mathcal{S}^1 = \{X_j^1\}_{j=1}^{M_1}$, respectively be the \textit{uncorrupted} observations in classes 0 and 1, where $M_0$ and $M_1$ are the number of observations from each class. To construct the original NP umbrella algorithm (for uncorrupted data), $\mathcal{S}^0$ is randomly split into $\mathcal{S}^0 = \mathcal{S}^0_{\text{b}} \cup \mathcal{S}^0_{\text{t}}$, where the subscript b reinforces that this part is to train a \textit{base algorithm} (e.g., logistic regression, random forest), and the subscript t reinforces that this part of the data is to find the \textit{threshold}. For the uncorrupted scenario, we do not split $\mathcal{S}^1$.  All $\mathcal{S}^1$ are used together with $\mathcal{S}^0_{\text{b}}$ to train a base algorithm.  

For the corrupted scenario, which is the focus of our paper, we assume the following sampling scheme for methodology and theory development. Let $\tilde{\mathcal{S}}^0 = \{\tilde{X}_j^0\}_{j=1}^{N_0}$ be \textit{corrupted} class $0$ observations and $\tilde{\mathcal{S}}^1 = \{\tilde{X}^1_j\}_{j=1}^{N_1}$ be \textit{corrupted} class $1$ observations. The sample sizes $N_0$ and $N_1$ are considered to be non-random numbers.  The split for the corrupted scenario is more complicated than the uncorrupted counterpart. Concretely, we split $\tilde{\mathcal{S}}^0$ into three parts: $\tilde{\mathcal{S}}^0=\tilde{\mathcal{S}}^0_{\text{b}}\cup \tilde{\mathcal{S}}^0_{\text{t}}\cup \tilde{\mathcal{S}}^0_{\text{e}}$, and split $\tilde{\mathcal{S}}^1$ into two parts $\tilde{\mathcal{S}}^1 = \tilde{\mathcal{S}}^1_{\text{b}}\cup \tilde{\mathcal{S}}^1_{\text{e}}$. The subscripts b and t have the same meaning as the uncorrupted case while the subscript e stands for \textit{estimation}, and $\tilde{\mathcal{S}}^0_{\text{e}}$ and $\tilde{\mathcal{S}}^1_{\text{e}}$ are used to estimate a correction term to account for the label noise. 

Given the above decomposition of $\tilde{\mathcal{S}}^0$ and $\tilde{\mathcal{S}}^1$, we also used $\tilde{S}_{\text{b}} = \tilde{\mathcal{S}}^0_{\text{b}} \cup \tilde{\mathcal{S}}^1_{\text{b}}$ to denote all corrupted class 0 and class 1 observations that are used to train the base algorithm in the label-noise-adjusted NP umbrella algorithm.  The sample size $n$ is reserved for $|\mathcal{S}^0_{\text{t}}|$ in the uncorrupted scenario, or for $|\tilde{\mathcal{S}}^0_{\text{t}}|$ in the corrupted scenario. The other sub-sample size notations are all for the corrupted scenario.  In particular, $n_{\text{b}} = |\tilde{S}_{\text{b}}| = |\tilde{\mathcal{S}}^0_{\text{b}} \cup \tilde{\mathcal{S}}^1_{\text{b}}|$, $n^0_{\text{e}} = |\tilde{S}^0_{\text{e}}|$, and $n^1_{\text{e}} = |\tilde{S}^1_{\text{e}}|$.

\section{BINARY SEARCH Algorithm}\label{sec:binary_search}

\begin{algorithm}[htb!]
\caption{\label{alg:bi_search}
Binary Search For $\alpha_{k,\delta}$}
\SetKw{KwBy}{by}
\SetKwInOut{Input}{Input}\SetKwInOut{Output}{Output}
\SetAlgoLined

\Input{
$\delta$: a small tolerance level, $0 < \delta < 1$ \\
$k, n$: two integers such that $k \leq n$ \\
$r$: a small number for error (we implement $r= 10^{-5}$ in our numerical analysis) \\
}

$\alpha_{\text{min}} \leftarrow 0$

$\alpha_{\text{max}} \leftarrow 1$

$\delta_{\text{max}} \leftarrow \sum_{j = k}^{n}{n \choose j}(1-\alpha_{\text{min}})^j\alpha_{\text{min}}^{n-j}$

$\delta_{\text{min}} \leftarrow \sum_{j = k}^{n}{n \choose j}(1-\alpha_{\text{max}})^j\alpha_{\text{max}}^{n-j}$

$E \leftarrow 2$

\While{$E > r$}{
$\alpha_{\text{middle}} \leftarrow (\alpha_{\text{min}} + \alpha_{\text{max}})/2$

$\delta_{\text{middle}} \leftarrow \sum_{j = k}^{n}{n \choose j}(1-\alpha_{\text{middle}})^j\alpha_{\text{middle}}^{n-j}$

\uIf{$\delta_{\emph{\text{middle}}} = \delta$}{
\Output{$\alpha_{\text{middle}}$}
}
\uElseIf{$\delta_{\emph{\text{middle}}} > \delta$}{
$\alpha_{\text{middle}} \leftarrow \alpha_{\text{min}}$
}
\Else{
$\alpha_{\text{middle}} \leftarrow \alpha_{\text{max}}$
}
$E \leftarrow |\delta_{\text{middle}} - \delta|$
}
\Output{$\alpha_{\text{middle}}$}

\end{algorithm}

Here $r$ is an error for stopping criterion of this binary search. That is, the algorithm stops when $\left| \sum_{j = k}^{n}{n \choose j}(1-\alpha_{\text{middle}})^j\alpha_{\text{middle}}^{n-j} - \delta \right| \leq r$.

\section{An example for assumption 3}\label{sec:example_3}

\begin{example}\label{ex:separability}
Under the same distributional setting as in Example \ref{ex: gmm}, let $\hat{T}$ be trained by linear discriminant analysis (LDA) on $\tilde{\mathcal{S}}_{\emph{\text{b}}}$; that is
$
    \hat{T}(X) = \hat{\tilde{\sigma}}^{-2}(\hat{\tilde{\mu}}_1 - \hat{\tilde{\mu}}_0)X\,,
$
in which  $\hat{\tilde{\mu}}_0$ and $\hat{\tilde{\mu}}_1$ are the sample means of corrupted class $0$ and $1$ observations, respectively, and $\hat{\tilde{\sigma}}^2$ is the pooled sample variance. For any $z\in \R $, by Lemma \ref{lemma:corollary_for_assumption_1} in the Appendix, we have 
\begin{align*}
    \tilde{F}^{\hat{T}}_0(z) - \tilde{F}^{\hat{T}}_1(z) = (m_0-m_1)\left(F^{\hat{T}}_0(z)-F^{\hat{T}}_1(z)\right)\,.
\end{align*}
Therefore, when $m_0 > m_1$ (as assumed in Assumption \ref{assumption:mixture}),  $\tilde{F}^{\hat{T}}_0(z) > \tilde{F}^{\hat{T}}_1(z)$ is equivalent to $F^{\hat{T}}_0(z)>F^{\hat{T}}_1(z)$. We first fix $\tilde{\mathcal{S}}_{\emph{b}}$, then $\hat{T}(X^0) \sim \mathcal{N}(\hat{\tilde{\sigma}}^{-2}(\hat{\tilde{\mu}}_1 - \hat{\tilde{\mu}}_0)\mu_0,\hat{\tilde{\sigma}}^{-4}(\hat{\tilde{\mu}}_1 - \hat{\tilde{\mu}}_0)^2\sigma^2)$ and $\hat{T}(X^1) \sim \mathcal{N}(\hat{\tilde{\sigma}}^{-2}(\hat{\tilde{\mu}}_1 - \hat{\tilde{\mu}}_1)\mu_0,\hat{\tilde{\sigma}}^{-4}(\hat{\tilde{\mu}}_1 - \hat{\tilde{\mu}}_0)^2\sigma^2)$. Since these two distributions are two normal with the same variance and different means, $F^{\hat{T}}_0(z)>F^{\hat{T}}_1(z)$ as long as $\hat{\tilde{\sigma}}^{-2}(\hat{\tilde{\mu}}_1 - \hat{\tilde{\mu}}_0)\mu_0 < \hat{\tilde{\sigma}}^{-2}(\hat{\tilde{\mu}}_1 - \hat{\tilde{\mu}}_0)\mu_1$, or equivalently, $(\hat{\tilde{\mu}}_1 - \hat{\tilde{\mu}}_0)(\mu_1 - \mu_0) > 0$. By Lemma \ref{lemma:corollary_for_assumption_1} in the Appendix, this condition can be written as $(\hat{\tilde{\mu}}_1 - \hat{\tilde{\mu}}_0)(\tilde{\mu}_1 - \tilde{\mu}_0)/(m_0-m_1) > 0$, where $\tilde{\mu}_0$ and $\tilde{\mu}_1$ are the means of $\tilde{X}^0$ and $\tilde{X}^1$ respectively. When $m_0 > m_1$, this is further equivalent to $(\hat{\tilde{\mu}}_1 - \hat{\tilde{\mu}}_0)(\tilde{\mu}_1 - \tilde{\mu}_0) > 0$. Then Assumption \ref{assumption:separability} follows from the law of large numbers.
\end{example}

\section{Additional Numerical Results}\label{appendix:num}

\subsection{Additional Simulations}\label{appendix:sim}

We apply Algorithm $1^{\#}$ in Simulation \ref{sim:uniform_unknown_flip_rate}. For $m_0^{\#}$ and $m_1^{\#}$ needed in Algorithm $1^{\#}$, we use the estimators proposed by \cite{liu2015classification}. Technically, \cite{liu2015classification} estimates the ``flip rates'' $\p\left(\tilde{Y} = 1\mid Y = 0\right)$ and $\p\left(\tilde{Y} = 0 \mid Y = 1\right)$. %Our $m_0$ and $m_1$, on the other hand, are $\p\left(Y = 0 \mid \tilde{Y} = 0\right)$ and $\p\left(Y = 0 \mid \tilde{Y} = 1\right)$. 
Our corruption levels can be derived from flip rates by the Bayes theorem.% with $\p\left(\tilde{Y} = 0\right)$ and $\p\left(\tilde{Y} = 1\right)$ being estimated by the proportion of both corrupted classes.

%\textcolor{purple}{Shunan, the simulation 5 should be updated as we previously agreed?}

\begin{simulation}\label{sim:uniform_unknown_flip_rate}

The distributional setting is the same as in Simulation \ref{sim:two_circle}. For different $(m_0, m_1, \alpha, \delta)$ combinations, the (approximate) type I error violation rates and averages of (approximate) true type II errors generated by Algorithm $1^{\#}$ are reported in Tables \ref{tbl:two_uni_violation_est} and \ref{tbl:two_uni_II_est}, respectively. %\textcolor{red}{Shunan, the tables 5 and 6 should be similar to Tables 3 and 4 in presentation.  In particular, the ``Standard error ($\times 10^3$) in parentheses" legend will help you get rid of the hard-to-read E-2, E-3, E-4....  Please also change the later tables accordingly.}

\begin{table}[h]
\centering
\caption{(Approximate) type I error violation rates over $1{,}000$ repetitions for Simulation \ref{sim:uniform_unknown_flip_rate}. Standard errors ($\times 10^{-3}$) in parentheses.}\label{tbl:two_uni_violation_est}
\begin{tabular}{|l|l|l|l|l|}
\hline
$N$ & \begin{tabular}[c]{@{}l@{}}$m_0 = .95, m_1 = .05$\\ $\alpha = .05, \delta = .05$\end{tabular} & \begin{tabular}[c]{@{}l@{}}$m_0 = .9, m_1 = .1$\\ $\alpha = .05, \delta = .05$\end{tabular} & \begin{tabular}[c]{@{}l@{}}$m_0 = .95, m_1 = .05$\\ $\alpha = .1, \delta = .1$\end{tabular} & \begin{tabular}[c]{@{}l@{}}$m_0 = .9, m_1 = .1$\\ $\alpha = .1, \delta = .1$\end{tabular} \\ \hline
$200$ & $.067(7.91)$ & $.068(7.96)$  & $.131(10.67)$  &  $.101(9.53)$\\ \hline
$500$ & $.084(8.78)$ &$.083(8.73)$  & $.134(10.78)$ &  $.115(10.09)$\\ \hline
$1{,}000$ & $.463(15.78)$ &$.182(12.21)$  & $.497(15.82)$ & $.197(12.58)$ \\ \hline
$2{,}000$ & $.665(14.93)$ &$.190(12.41)$  & $.695(14.57)$ & $.209(12.86)$ \\ \hline
\end{tabular}
\end{table}

\begin{table}[h]
\centering
\caption{Averages of (approximate) true type II errors over $1{,}000$ repetitions for Simulation \ref{sim:uniform_unknown_flip_rate}. Standard errors ($\times 10^{-3}$) in parentheses.}\label{tbl:two_uni_II_est}
\begin{tabular}{|l|l|l|l|l|}
\hline
$N$ & \begin{tabular}[c]{@{}l@{}}$m_0 = .95, m_1 = .05$\\ $\alpha = .05, \delta = .05$\end{tabular} & \begin{tabular}[c]{@{}l@{}}$m_0 = .9, m_1 = .1$\\ $\alpha = .05, \delta = .05$\end{tabular} & \begin{tabular}[c]{@{}l@{}}$m_0 = .95, m_1 = .05$\\ $\alpha = .1, \delta = .1$\end{tabular} & \begin{tabular}[c]{@{}l@{}}$m_0 = .9, m_1 = .1$\\ $\alpha = .1, \delta = .1$\end{tabular} \\ \hline
$200$ & $.431(9.36)$ &$.589(9.79)$  & $.150(2.44)$ & $.221(4.86)$ \\ \hline
$500$ & $.219(3.60)$ &$.391(7.25)$  & $.115(.95)$ & $.145(1.49)$  \\ \hline
$1{,}000$ & $.140(.99)$ &$.190(2.53)$  & $.082(.72)$  & $.107(1.01)$ \\ \hline
$2{,}000$ & $.128(.82)$ &$.175(1.75)$  & $.073(.65)$ & $.101(.88)$ \\ \hline
\end{tabular}
\end{table}
\end{simulation}

%\textcolor{purple}{Either change one of the parameter combination in Simu 4 to the setting of Simu 2, or change simulation 2 to one of the settings in simulation 4. Actually 1-3 all need some changes. When you change, the parameter choice for simulation 2 and that of simulation 3 should be the same: use that of Simulations 6 and 7. Use a new tex file after the change. }

%\textcolor{red}{Shunan, immediately after Simulation 4, you should explain the results, like what we wrote after Simulations 1-3.  Also, the next paragraph starts with Simulations 6 and 7 the Simulation 5 is not revealed}

In this simulation, Algorithm $1^{\#}$  fails to control the type I error with pre-specified high probability. Similar results on additional distributional settings can be found in Simulation  \ref{sim:normal_unknown_flip_rate} of Appendix \ref{appendix:sim}.  One might wonder: if we were to use other estimators of $m_0$ and $m_1$, will the result be different? The answer is that the usually ``good" estimators do not serve for the purpose of high probability control on type I error. For example,  Simulation \ref{sim:normal_estimator} in Appendix \ref{appendix:sim} uses consistent and unbiased estimators of $m_0$ and $m_1$, but Algorithm $1^{\#}$ still fails to control the type I error.

%\textcolor{red}{Shunan, please specify which is the algorithm we use in each example.  As I commented earlier, there are Algorithm 1 and Algorithm 2.}

\begin{simulation}\label{sim:normal_unknown_flip_rate}
The distributional setting is the same as in Simulation \ref{sim:gmm}. For different $(m_0, m_1, \alpha, \delta)$ combinations, the (approximate) true type I errors generated by Algorithm $1^{\#}$ are reported in Table \ref{tbl:two_normal_violation_est}.

\begin{table}[h]
\centering
\caption{(Approximate) type I error violation rates over $1{,}000$ repetitions for Simulation \ref{sim:uniform_unknown_flip_rate}. Standard errors ($\times 10^{-3}$) in parentheses.}\label{tbl:two_normal_violation_est}
\begin{tabular}{|l|l|l|l|l|}
\hline
$N$ & \begin{tabular}[c]{@{}l@{}}$m_0 = .95, m_1 = .05$\\ $\alpha = .05, \delta = .05$\end{tabular} & \begin{tabular}[c]{@{}l@{}}$m_0 = .9, m_1 = .1$\\ $\alpha = .05, \delta = .05$\end{tabular} & \begin{tabular}[c]{@{}l@{}}$m_0 = .95, m_1 = .05$\\ $\alpha = .1, \delta = .1$\end{tabular} & \begin{tabular}[c]{@{}l@{}}$m_0 = .9, m_1 = .1$\\ $\alpha = .1, \delta = .1$\end{tabular} \\ \hline
$200$ & $.430(15.66)$ & $.512(15.81)$  & $.530(15.79)$  &  $.504(15.82)$\\ \hline
$500$ & $.694(14.58)$ &$.488(15.81)$  & $.758(13.55)$ &  $.570(15.66)$\\ \hline
$1{,}000$ & $.940(7.51)$ &$.788(13.47)$  & $.953(6.70)$ & $.805(12.54)$ \\ \hline
$2{,}000$ & $.950(6.90)$ &$.792(12.80)$  & $.957(6.42)$ & $.818(12.21)$ \\ \hline
\end{tabular}
\end{table}

\end{simulation}

%\begin{simulation}\label{sim:t_unknown_flip_rate}
%The distributional setting is the same as in Simulation \ref{sim:two_t}. For different $(m_0, m_1, \alpha, \delta)$ combinations, the (approximate) true type I errors generated by Algorithm $1^{\#}$ are reported in Table \ref{tbl:two_t_violation_est}.
%
%
%\begin{table}[h]
%\centering
%\caption{(Approximate) type I error violation rates over $1{,}000$ repetitions for Simulation \ref{sim:uniform_unknown_flip_rate}. Standard errors ($\times 10^{-3}$) in parentheses.}\label{tbl:two_t_violation_est}
%\begin{tabular}{|l|l|l|l|l|}
%\hline
%$N$ & \begin{tabular}[c]{@{}l@{}}$m_0 = .95, m_1 = .05$\\ $\alpha = .05, \delta = .05$\end{tabular} & \begin{tabular}[c]{@{}l@{}}$m_0 = .9, m_1 = .1$\\ $\alpha = .05, \delta = .05$\end{tabular} & \begin{tabular}[c]{@{}l@{}}$m_0 = .95, m_1 = .05$\\ $\alpha = .1, \delta = .1$\end{tabular} & \begin{tabular}[c]{@{}l@{}}$m_0 = .9, m_1 = .1$\\ $\alpha = .1, \delta = .1$\end{tabular} \\ \hline
%$200$ & $.998(1.41)$ & $1(0)$  & $1(0)$  &  $1(0)$\\ \hline
%$500$ & $1(0)$ &$1(0)$  & $1(0)$ &  $1(0)$\\ \hline
%$1{,}000$ & $1(0)$ &$1(0)$  & $1(0)$ & $1(0)$ \\ \hline
%$2{,}000$ & $1(0)$ &$1(0)$  & $1(0)$ & $1(0)$ \\ \hline
%\end{tabular}
%\end{table}
%\end{simulation}

%\textcolor{purple}{In simulation 8 and 9, change $m_0$ and $m_1$ parameters to agree with the other settings.  Also, for 8-11, change the narrative to the standard language, like in the previous simulations.}

\begin{simulation}\label{sim:normal_estimator}

The distributional setting is the same as in Simulation \ref{sim:gmm}. The $m_0^\#$ and $m_1^\#$ are generated from $\mathcal{N}(m_0, 1/N)$ and $\mathcal{N}(m_1, 1/N)$, respectively. The (approximate) type I error violation rates generated by Algorithm $1^{\#}$ for one combination ($m_0 = .95$, $m_1 = .05$, $\alpha = .1$ and $\delta = .1$) are reported in Table \ref{tbl:normal_estimator}.

\begin{table}[t]
\centering
\caption{(Approximate) type I error violation rates over $1{,}000$ repetitions for  Simulation \ref{sim:normal_estimator}. Standard errors ($\times 10^{-2}$) in parentheses.}\label{tbl:normal_estimator}
\begin{tabular}{|l|l|}
\hline
$N$ & (approximate) violation rate \\ \hline
$200$ & $.193(1.25)$ \\ \hline
$500$ & $.208(1.28)$ \\ \hline
$1{,}000$ & $.186(1.23)$ \\ \hline
$2{,}000$ & $.203(1.27)$ \\ \hline
\end{tabular}
\end{table}
\end{simulation}

\subsection{CIFAR10 data analysis}\label{appendix:cifar}

%In this section, we apply Algorithms \ref{alg:adj_umbrella} to the CIFAR10 dataset \citep{krizhevsky2009learning}. As we focus on binary classification problems, we merge the ten categories of the CIFAR10 dataset into two: ``vehicles'' and ``non-vehicles.'' The class ``vehicles,'' encoded as $0$, contains the original ``automobile'' and ``truck'' classes, and the class ``non-vehicles,'' encoded as $1$, contains the other eight original classes. Then type I/II errors are defined accordingly. We employ the NP paradigm in this modified dataset to prioritize the control over the chances of the failure to detect vehicles. 

In this section we apply Algorithm \ref{alg:adj_umbrella} to the CIFAR10 dataset \citep{krizhevsky2009learning}. As we focus on binary classification problems, we merge the ten categories of the CIFAR10 dataset into two: ``vehicles'' and ``non-vehicles.'' The class ``vehicles,'' encoded as $0$, contains the original ``automobile'' and ``truck'' classes, and the class ``non-vehicles,'' encoded as $1$, contains the other eight original classes. Then type I/II errors are defined accordingly. We employ the NP paradigm to this modified dataset to prioritize control over the chance of failing to detect vehicles.

%The original CIFAR10 dataset has pre-specified training and test sets, but the number of class $0$ observations in the test set is too small ($2{,}000$ in total) to produce a reliable approximation to population-level type I error. Furthermore, given that the train-test procedure has to be repeated multiple times to approximate type I error violation rate, a fixed train-test split throughout all repetitions does not serve this purpose. As such, we perform stratified splits to the whole modified CIFAR10 dataset (with the newly assigned labels). In particular, $20\%$ true class $0$ observations and $20\%$ true class $1$ observations are randomly selected to form the new training set and the rest observations form the evaluation set. Then, the training and evaluation sets contain $12{,}000$ and $48{,}000$ observations, respectively. Moreover, the labels of all training observations are artificially corrupted by the same method in Section \ref{sec:realdata} with $m_0 = 0.95$ and $m_1 = 0.05$. By the Bayes theorem, the flip rates $r_0 = \p\left(\tilde{Y} = 1 \mid Y = 0\right)$ and $r_1 = \p\left(\tilde{Y} = 0 \mid Y = 1\right)$ are $0.2083$ and $0.0104$, respectively. We apply Algorithm \ref{alg:adj_umbrella} (with the parameter choice $\alpha = \delta = 0.1$ and CNN as the base algorithm) to the training set with corrupted labels and obtain a classifier. Then, the classifier is applied to the untouched evaluation set to calculate the (approximate) true type I and II errors. This procedure is repeated $1{,}000$ times.

The original CIFAR10 dataset has pre-specified training and test sets, but the number of class $0$ observations in the test set is too small ($2{,}000$ in total) to produce a reliable approximation to population-level type I error. Furthermore, given that the train-test procedure has to be repeated multiple times to approximate the type I error violation rate, a fixed train-test split throughout all repetitions does not serve our purpose. As such, we perform stratified splits to the whole modified CIFAR10 dataset (with the newly assigned labels). In particular, $20\%$ true class $0$ observations and $20\%$ true class $1$ observations are randomly selected to form the new training set and the remaining observations form the evaluation set. The training and evaluation sets contain $12{,}000$ and $48{,}000$ observations, respectively. Moreover, the labels of all training observations are artificially corrupted by the same method as in Section \ref{sec:realdata} with $m_0 = 0.95$ and $m_1 = 0.05$. By Bayes theorem, the flip rates $r_0 = \p\left(\tilde{Y} = 1 \mid Y = 0\right)$ and $r_1 = \p\left(\tilde{Y} = 0 \mid Y = 1\right)$ are $0.2083$ and $0.0104$, respectively. We apply Algorithm \ref{alg:adj_umbrella} (with the parameter choice $\alpha = \delta = 0.1$ and CNN as the base algorithm) to the training set with corrupted labels and obtain a classifier. Then, the classifier is applied to the untouched evaluation set to calculate the (approximate) true type I and II errors. This procedure is repeated $1{,}000$ times.

%In the main text, we have argued that unless Algorithm $1^{\#}$  uses under-estimators for corruption levels,  Algorithm \ref{alg:adj_umbrella} is the only algorithm that fulfills the goal of high-probability control over type I error control among all experimented methods in Section \ref{sec:sim_and_real_data}. Hence, we only apply Algorithm \ref{alg:adj_umbrella} to the modified CIFAR10 dataset since our primary interest of this work is the type I error control. For the base algorithm in Algorithm \ref{alg:adj_umbrella}, we use the CNN model for CIFAR10 data in \cite{chollet2017kerasR}, with modifications for the binary labels we create and the data splitting scheme we perform. The (approximate) type I error violation rate and average of (approximate) true type II errors are presented in Table \ref{tbl:cifar_adj}. Clearly, Algorithm \ref{alg:adj_umbrella} is able to achieve the high probability control of the true type I error under the specified level. 
%

In the main text, we have shown that Algorithm $1^{\#}$ with under-estimates of corruption levels fulfills the goal of high-probability control over the type I error, while other benchmark algorithms do not. To avoid delivering redundant messages, we only apply Algorithm \ref{alg:adj_umbrella} to the modified CIFAR10 dataset since our primary interest is the type I error control.  The (approximate) type I error violation rate and average of (approximate) true type II errors are presented in Table \ref{tbl:cifar_adj}. Clearly, Algorithm \ref{alg:adj_umbrella} is able to achieve high probability control of the true type I error under the specified level.

\begin{table}[h]
\caption{(Approximate) type I error violation rate, and average  of (approximate) true type II errors by Algorithm \ref{alg:adj_umbrella} over $1{,}000$ repetitions for the modified CIFAR10 dataset. Standard errors ($\times10^{-3}$) in parentheses.}
\centering
\begin{tabular}{|l|l|l|}
\hline
 & \begin{tabular}[c]{@{}l@{}}(approximate)\\ violation rate\end{tabular} & \begin{tabular}[c]{@{}l@{}}average of\\ (approximate) true\\ type II errors\end{tabular} \\ \hline
\begin{tabular}[c]{@{}l@{}}Algorithm \ref{alg:adj_umbrella} with \\ CNN as base algorithm\end{tabular} & $.099(9.45)$ & $.150(.87)$ \\ \hline
\end{tabular}\label{tbl:cifar_adj}
\end{table}
%\textcolor{purple}{Shunan, 1, continue to run a 1000 repetition version; 2, add CNN in the first column of table 12; 3, add that this is Algorithm 1, like in some previous tables.}

\subsection{Violin plots for Section \ref{sec:sim_and_real_data}}\label{appendix:violin_for_section_5}

%\textcolor{purple}{Shunan, here we should specify that the violin plots are only for Simulations 1-4. }

In this section, we present the violin plots (Figures \ref{fig:gmm0} - \ref{fig:two_t1}) for Simulations \ref{sim:gmm}-\ref{sim:two_t} in Section \ref{sec:sim_and_real_data}. The violin plots for the (approximate) true type I and type II errors over these $1{,}000$ repetitions are plotted for each $(m_0, m_1, \alpha, \delta)$ combination. Take Figures \ref{fig:gmm0} and \ref{fig:gmm1} as an example, the two rows in each figure respectively correspond to the $m_1=0.95, m_1=0.05$ and $m_0=0.85, m_1=0.15$ settings, while the two columns respectively correspond to $\alpha=0.05, \delta=0.05$ and $\alpha=0.10, \delta=0.10$. 
  The area of every plot with lighter color represents true type I errors above the $1-\delta$ quantile while the area with darker color represents true type I errors below the $1-\delta$ quantile. The black dots represent the average of true type I/II errors and the bars above and below the dots represent standard deviations. 
  
  %\textcolor{red}{Note that the simulation 5 is changed, so this paragraph needs to be edited.}

%\textcolor{red}{Shunan, here it says Figures 8 and 9, but it looks more like Figures 9 and 10.}

\begin{figure}
\begin{center}
    \includegraphics[width = \textwidth, height = 8cm]{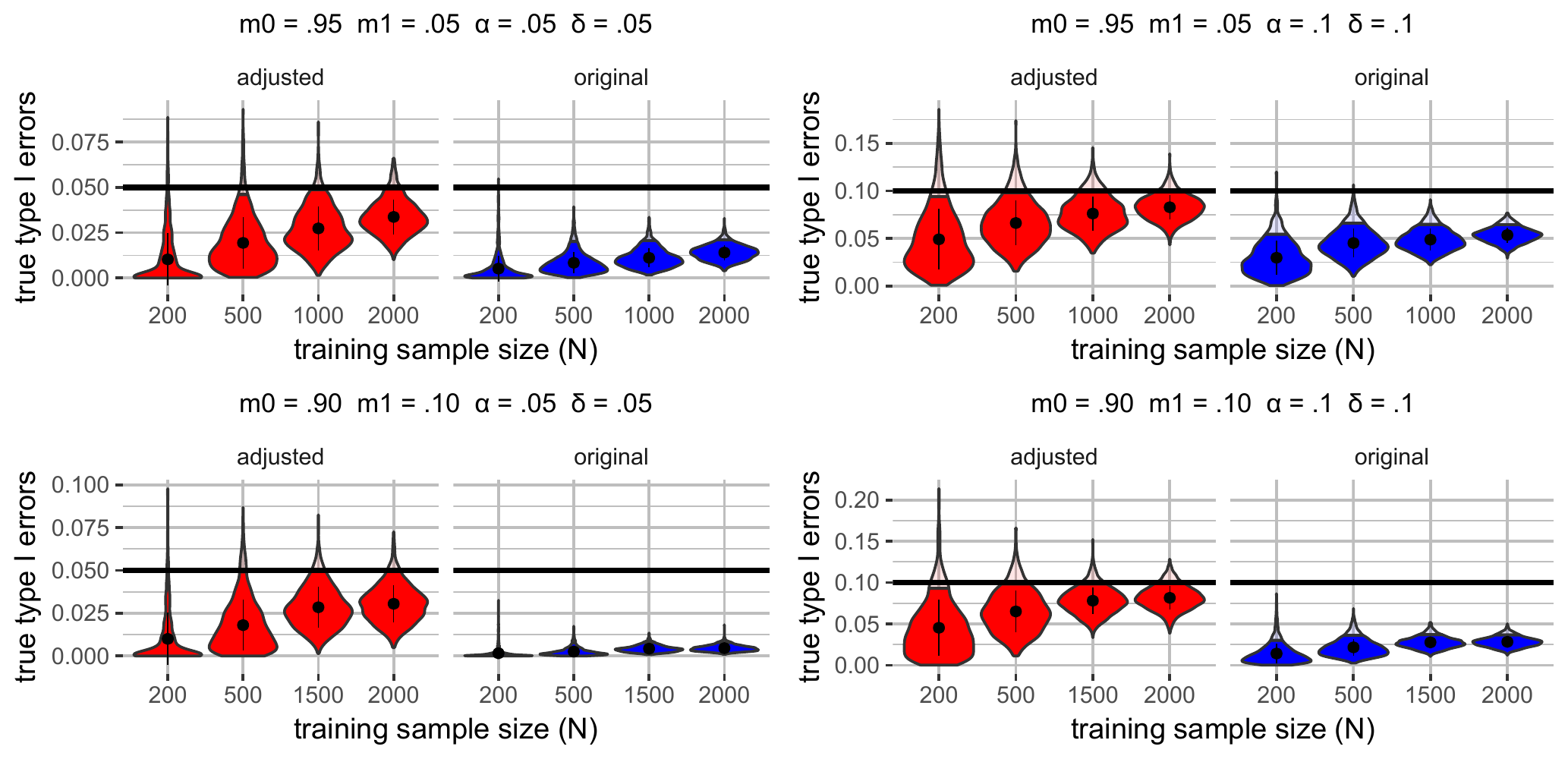}
    \caption{Violin plots for (approximate) true type I errors of Simulation \ref{sim:gmm}. }\label{fig:gmm0} 
\end{center}
\end{figure}

\begin{figure}
\begin{center}
    \includegraphics[width = \textwidth, height = 8cm]{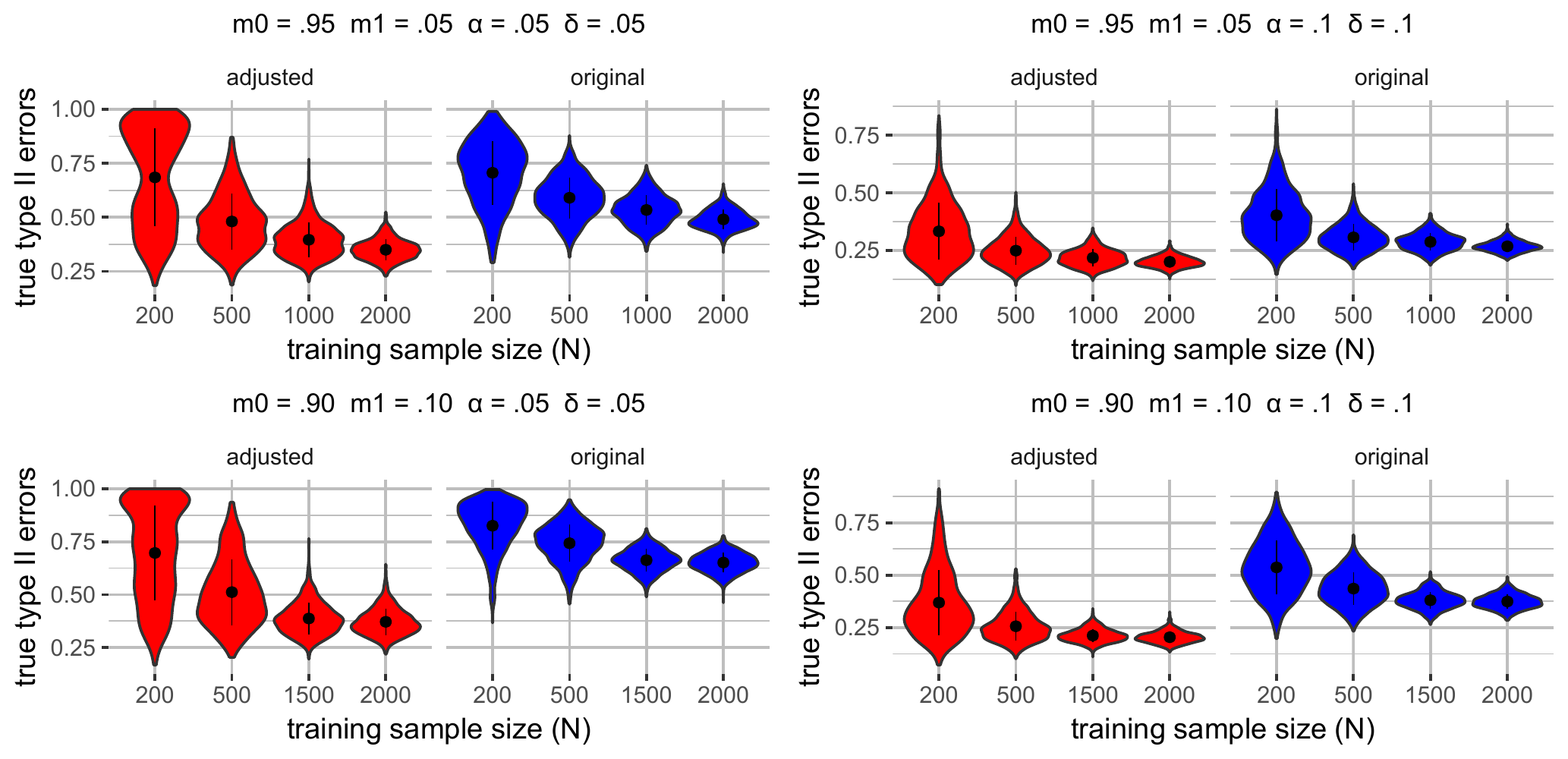}
    \caption{Violin plots for (approximate) true type II errors of Simulation \ref{sim:gmm}. }\label{fig:gmm1}
\end{center}
\end{figure}

\begin{figure}
\begin{center}
    \includegraphics[width = \textwidth, height = 8cm]{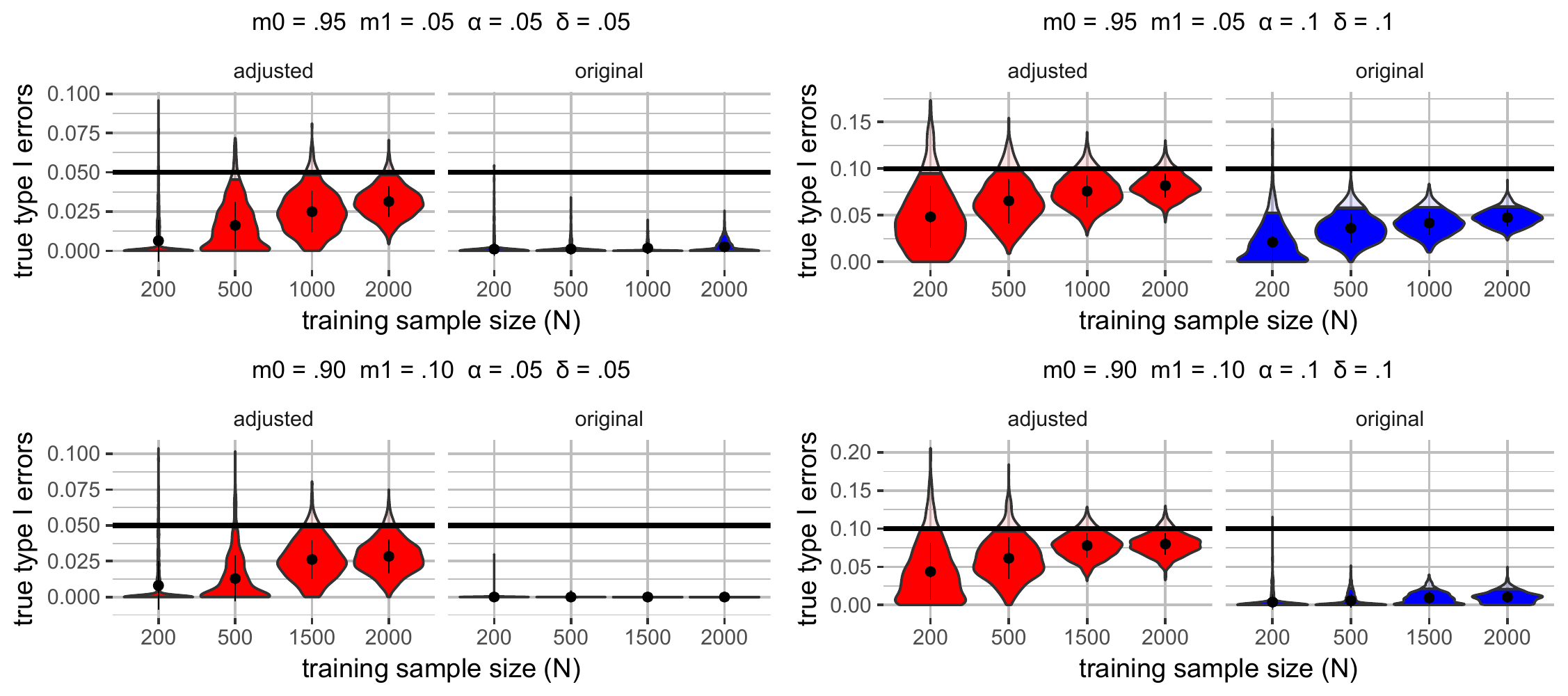}
    \caption{Violin plots for (approximate) true type I errors of Simulation \ref{sim:two_circle}. }\label{fig:two_circle0}
\end{center}
\end{figure}

\begin{figure}
\begin{center}
    \includegraphics[width = \textwidth, height = 8cm]{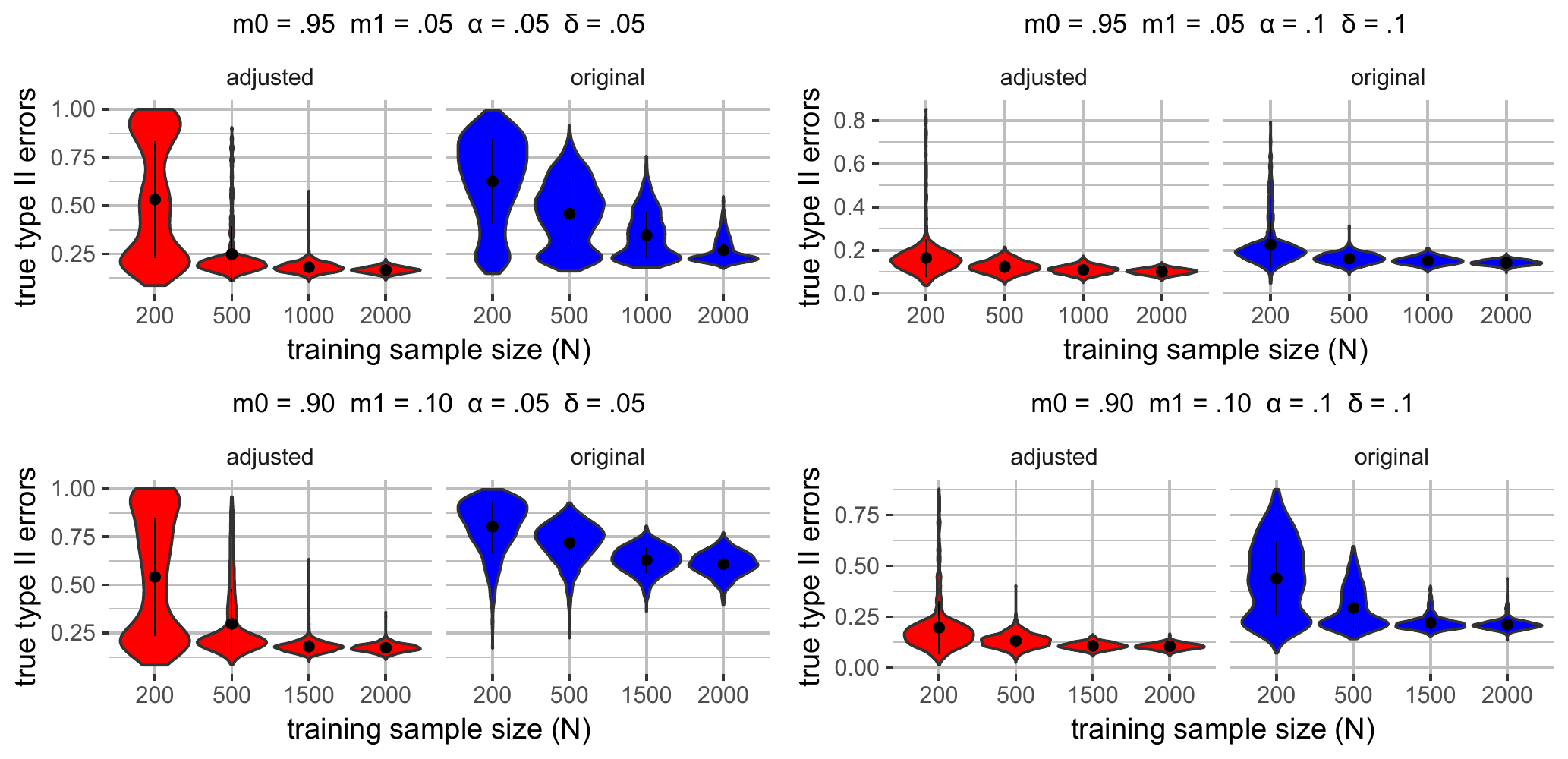}
    \caption{Violin plots for (approximate) true type II errors of Simulation \ref{sim:two_circle}. }\label{fig:two_circle1}
\end{center}
\end{figure}

\begin{figure}
\begin{center}
    \includegraphics[width = \textwidth, height = 8cm]{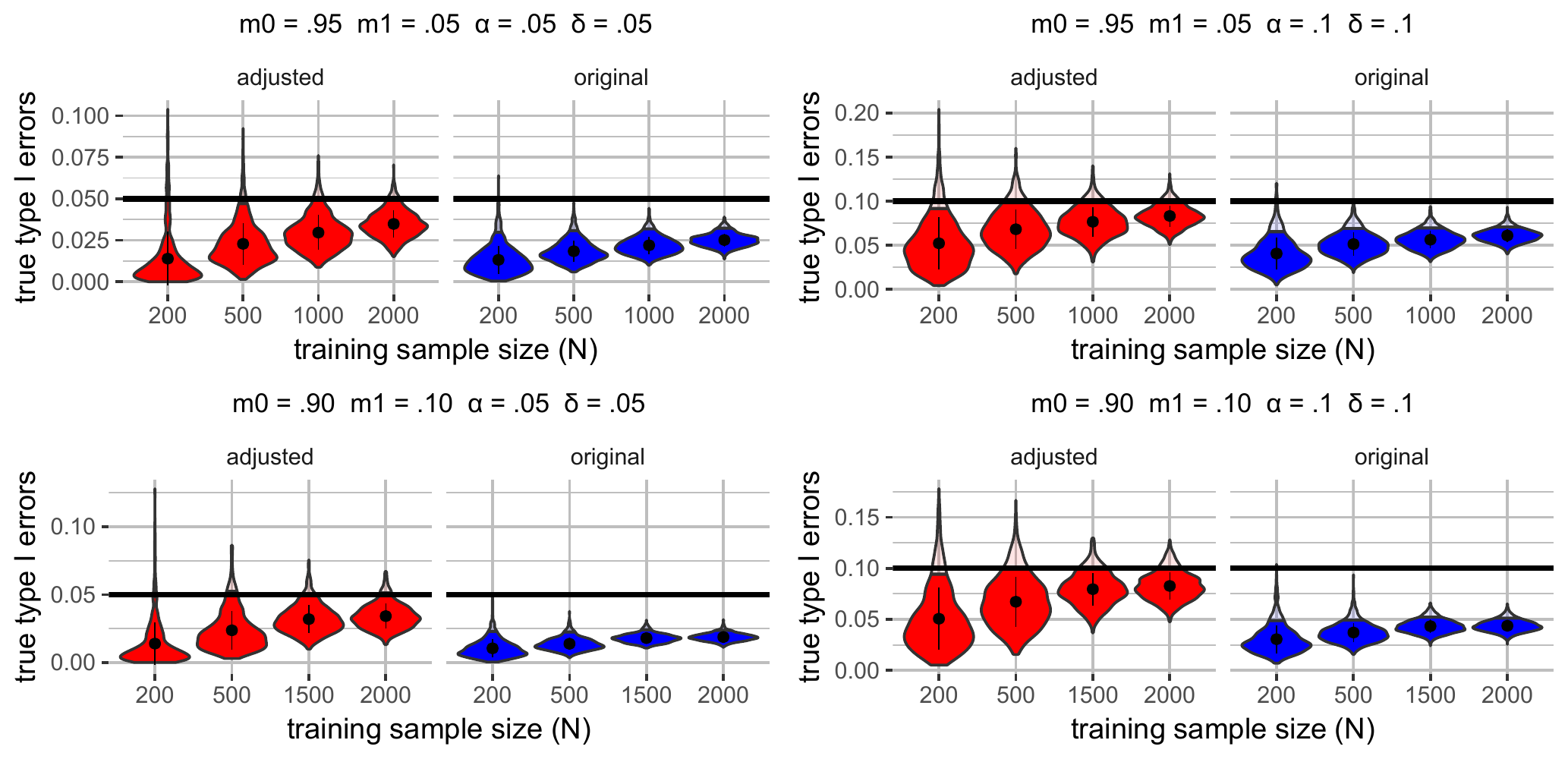}
    \caption{Violin plots for (approximate) true type I errors of Simulation \ref{sim:two_t}. }\label{fig:two_t0}
\end{center}
\end{figure}

\begin{figure}
\begin{center}
    \includegraphics[width = \textwidth, height = 8cm]{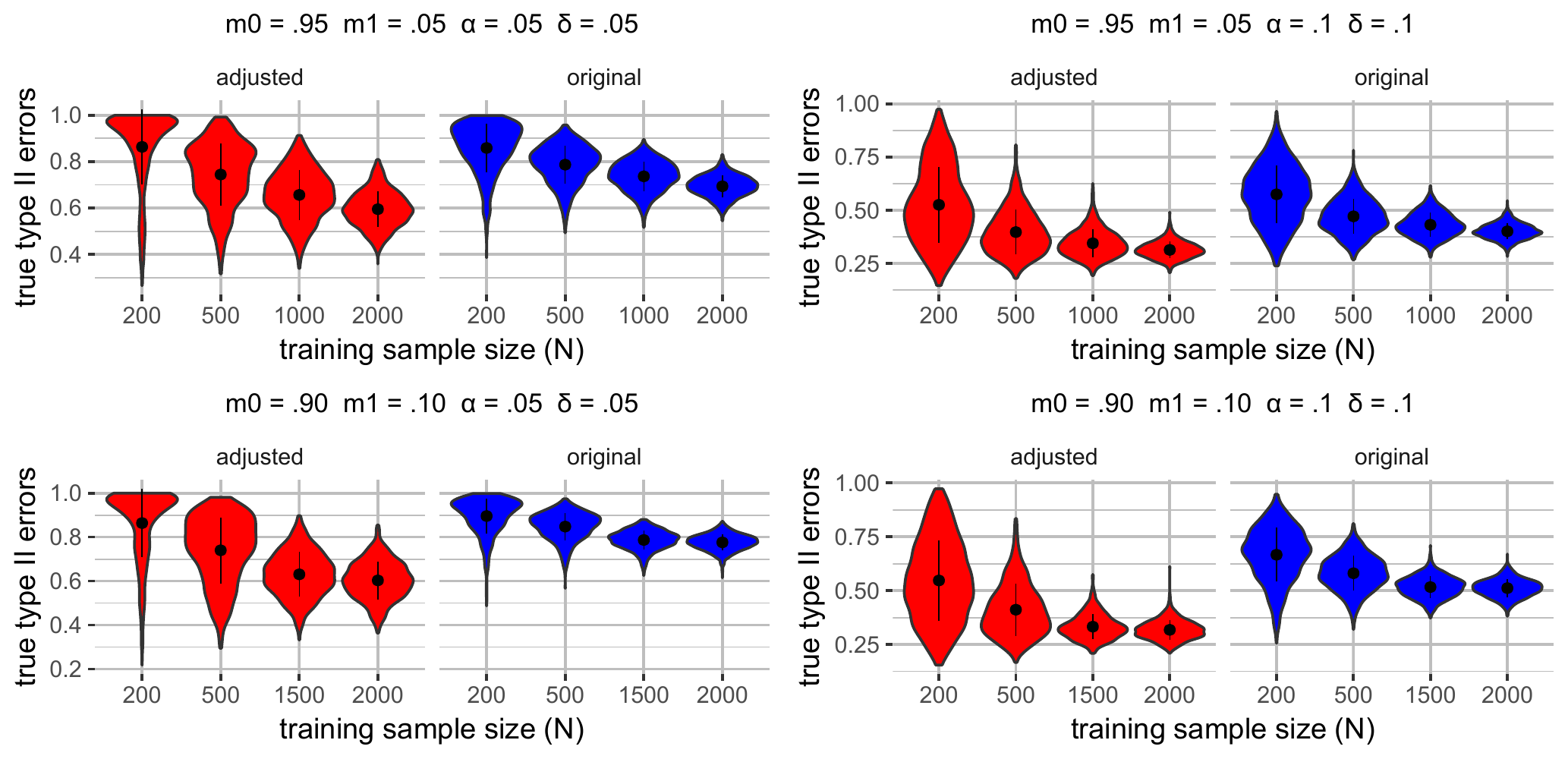}
    \caption{Violin plots for (approximate) true type II errors of Simulation \ref{sim:two_t}. } \label{fig:two_t1}
\end{center}
\end{figure}

% \begin{figure}
% \begin{center}
%     \includegraphics[width = \textwidth, height = 8cm]{Figures/uni_1_est.pdf}
%     \caption{Violin plots for (approximate) true type I errors of Simulation \ref{sim:uniform_unknown_flip_rate}.}\label{fig:uni_I_est} 
% \end{center}
% \end{figure}

% \begin{figure}
% \begin{center}
%     \includegraphics[width = \textwidth, height = 8cm]{Figures/uni_2_est.pdf}
%     \caption{Violin plots for (approximate) true type II errors of Simulation \ref{sim:uniform_unknown_flip_rate}.}\label{fig:uni_II_est} 
% \end{center}
% \end{figure}

\subsection{Tables for Section \ref{sec:sim_and_real_data}}\label{appendix:table}

In this section, we present Table \ref{tbl:compare_gmm_2} for in Simulation \ref{sim:compare_gmm} in Section \ref{sec:simulation} and Table \ref{tbl:real_data_benchmark} for the email spam data analysis in Section \ref{sec:realdata}.

\begin{table}[h]
\centering
\caption{Averages of (approximate) true type II errors over $1{,}000$ repetitions for Simulation \ref{sim:compare_gmm} ($m_0 = .95$, $m_1 = .05$, $\alpha = .1$ and $\delta = .1$). Standard errors ($\times 10^{-3}$) in parentheses.}\label{tbl:compare_gmm_2}
\begin{tabular}{|l|l|l|l|l|}
\hline
\multirow{2}{*}{algorithms} & \multicolumn{4}{l|}{$N$} \\ \cline{2-5} 
 & $200$ & $500$ & $1{,}000$ & $2{,}000$ \\ \hline
T-revision & $.165(4.32)$  & $.153(4.08)$ & $.146(3.52)$  & $.147(4.27)$ \\ \hline
\begin{tabular}[c]{@{}l@{}}backward loss correction\\ (known corruption level)\end{tabular} & $.151(.77)$ & $.139(.70)$ & $.161(.71)$ & $.199(.69)$ \\ \hline
\begin{tabular}[c]{@{}l@{}}backward loss correction\\ (unknown corruption level)\end{tabular} & $.158(.02)$ & $.163(.02)$ & $.186(.01)$ & $.192(.01)$ \\ \hline
% \begin{tabular}[c]{@{}l@{}}label-noise-adjusted NP umbrella\\ (known corruption levels)\end{tabular} & $.333(3.93)$ & $.249(1.94)$ & $.218(1.18)$ & $.201(.76)$ \\ \hline
% \begin{tabular}[c]{@{}l@{}}label-noise-adjusted NP umbrella\\ (unknown corruption levels)\end{tabular} & $.185(2.82)$ & $.145(1.46)$ & $.114(1.31)$ & $.113(1.34)$ \\ \hline
\end{tabular}
\end{table}

% \begin{table}[h]
% \caption{(Approximate) type I error violation rates, and averages  of (approximate) true type II errors by Algorithm $1^{\#}$ over $1{,}000$ repetitions for the email spam data. Standard errors ($\times10^{-3}$) in parentheses.}
% \vspace{5 pt}
% \centering
% \begin{tabular}{|l|l|l|}
% \hline
%  & \begin{tabular}[c]{@{}l@{}}(approximate) \\ violation rate\end{tabular} & \begin{tabular}[c]{@{}l@{}}average of \\ (approximate) true \\ type II errors\end{tabular} \\ \hline
% penalized logistic regression & $.869(10.67)$ & $.098(1.07)$ \\ \hline
% \begin{tabular}[c]{@{}l@{}}linear discriminant analysis\end{tabular} & $.850(11.30)$ & $.114(1.55)$ \\ \hline
% \begin{tabular}[c]{@{}l@{}}support vector machine\end{tabular} & $.727(14.10)$ & $.116(2.33)$ \\ \hline
% \begin{tabular}[c]{@{}l@{}}random forests\end{tabular} & $.412(15.57)$ & $.092(.68)$ \\ \hline
% \end{tabular}
% \label{table:real_data_est}
% \end{table}

\begin{table}[h]
\caption{(Approximate) type I error violation rates, and averages  of (approximate) true type II error by benchmark algorithms over $1{,}000$ repetitions for the email spam data. Standard errors ($\times10^{-3}$) in parentheses.}
\vspace{5 pt}
\centering
\begin{tabular}{|l|l|l|}
\hline
 & \begin{tabular}[c]{@{}l@{}}(approximate) \\ violation rate\end{tabular} & \begin{tabular}[c]{@{}l@{}}average of \\ (approximate) true \\ type II errors\end{tabular} \\ \hline
T-revision & $.829(11.91)$ & $.414(7.01)$ \\ \hline
\begin{tabular}[c]{@{}l@{}}backward loss correction\\ (known corruption level)\end{tabular} & $.831(11.86)$ & $.573(5.49)$ \\ \hline
\begin{tabular}[c]{@{}l@{}}backward loss correction\\ (unknown corruption level)\end{tabular} & $.750(13.70)$ & $.631(5.51)$ \\ \hline
\end{tabular}\label{tbl:real_data_benchmark}
\end{table}

%\begin{figure}
%\begin{center}
%    \includegraphics[width = \textwidth, height = 7cm]{Figures/realdata_01.pdf}
%    \caption{Violin plots for estimated true type I errors of the email spam data. }\label{fig:email_spam0}
%\end{center}
%\end{figure}
%
%\begin{figure}
%\begin{center}
%    \includegraphics[width = \textwidth, height = 8cm]{Figures/realdata_02.pdf}
%    \caption{Violin plots for estimated true type II errors of the email spam data.}\label{fig:email_spam1} 
%\end{center}
%\end{figure}
%
%

\subsection{Alternative implementation with a positive  $\varepsilon$}\label{appendix:numerical}

%\textcolor{purple}{For this section: run different $m_0$ and $m_1$ combination to be consistent with main text. }

In this section, we repeat the numerical studies for Simulations \ref{sim:gmm}-\ref{sim:two_t} in Section \ref{sec:sim_and_real_data} but replace $k^*$ in Algorithm \ref{alg:adj_umbrella} by $\min\{k\in\{1, \ldots, n\}:\alpha_{k,\delta} - \hat{D}^+(t_{(k)}) \leq \alpha - \varepsilon\}$ where $\varepsilon = 0.0001$. The results are presented in Figures  \ref{fig:sim_gmm_epssilon_1} - \ref{fig:sim:two_t_type_II}.     Numerical evidence shows that whether to have a small positive $\varepsilon$ in selection of $k^*$ does not actually affect much the performance of label-noise-adjusted umbrella algorithm. Thus, as a simpler algorithm is always preferred, we recommend taking $\varepsilon = 0$.

\begin{figure}
\begin{center}
    \includegraphics[width = \textwidth, height = 8cm]{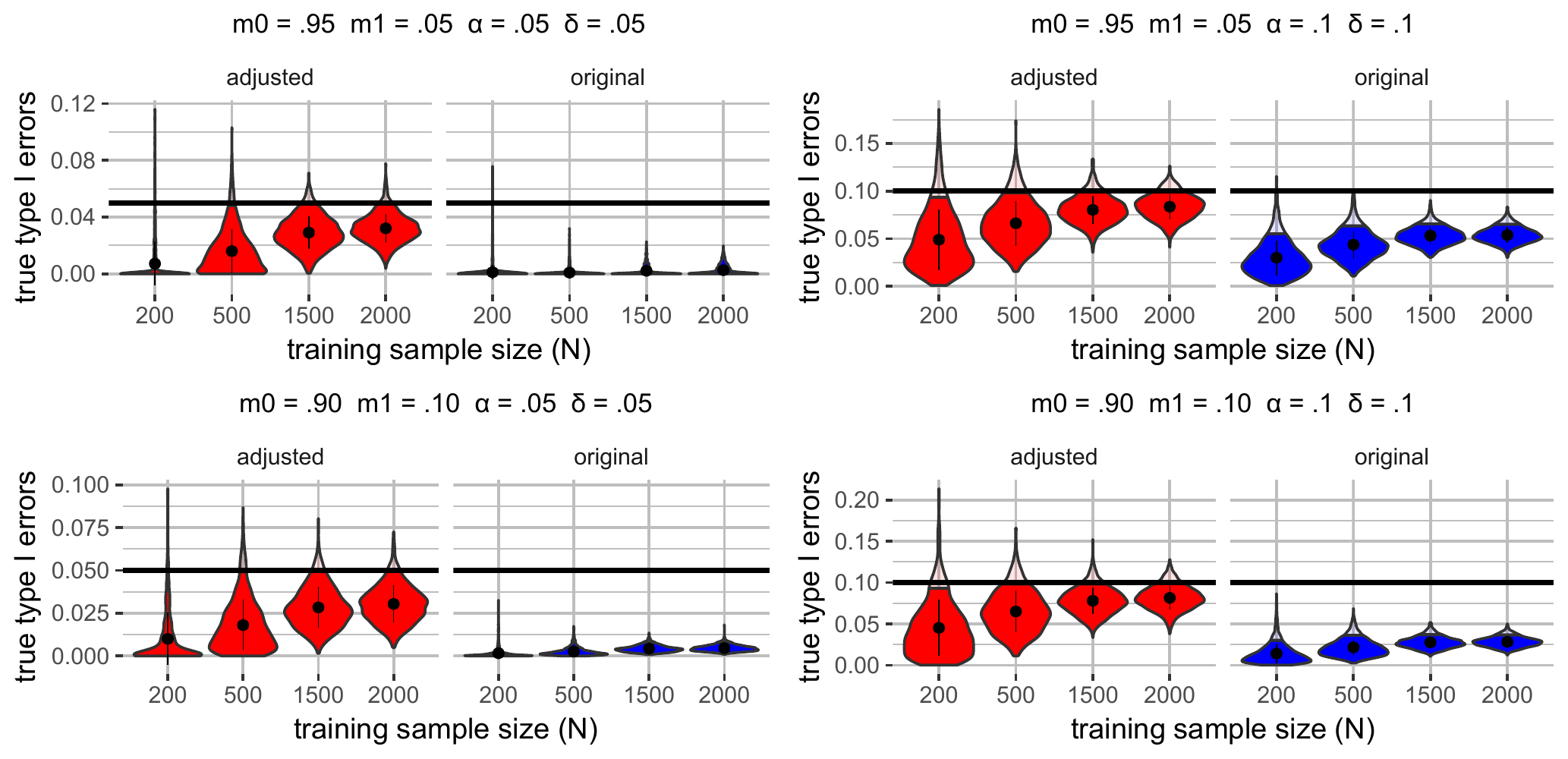}
    \caption{Violin plots for (approximate) true type I errors of Simulation \ref{sim:gmm}. }\label{fig:sim_gmm_epssilon_1}
\end{center}
\end{figure}

\begin{figure}
\begin{center}
    \includegraphics[width = \textwidth, height = 8cm]{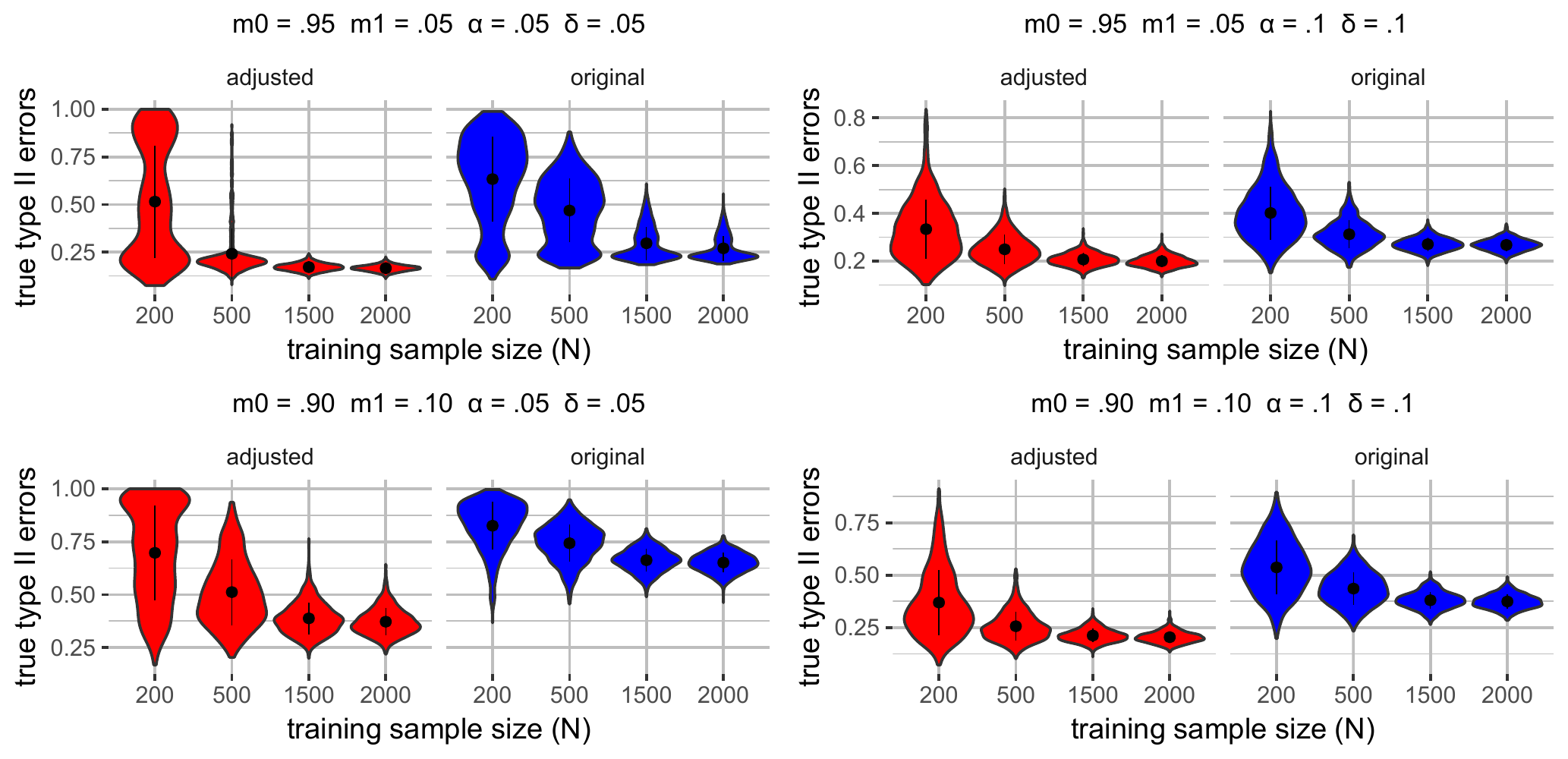}
    \caption{Violin plots for (approximate) true type II errors of Simulation \ref{sim:gmm}. }
\end{center}
\end{figure}

\begin{figure}
\begin{center}
    \includegraphics[width = \textwidth, height = 8cm]{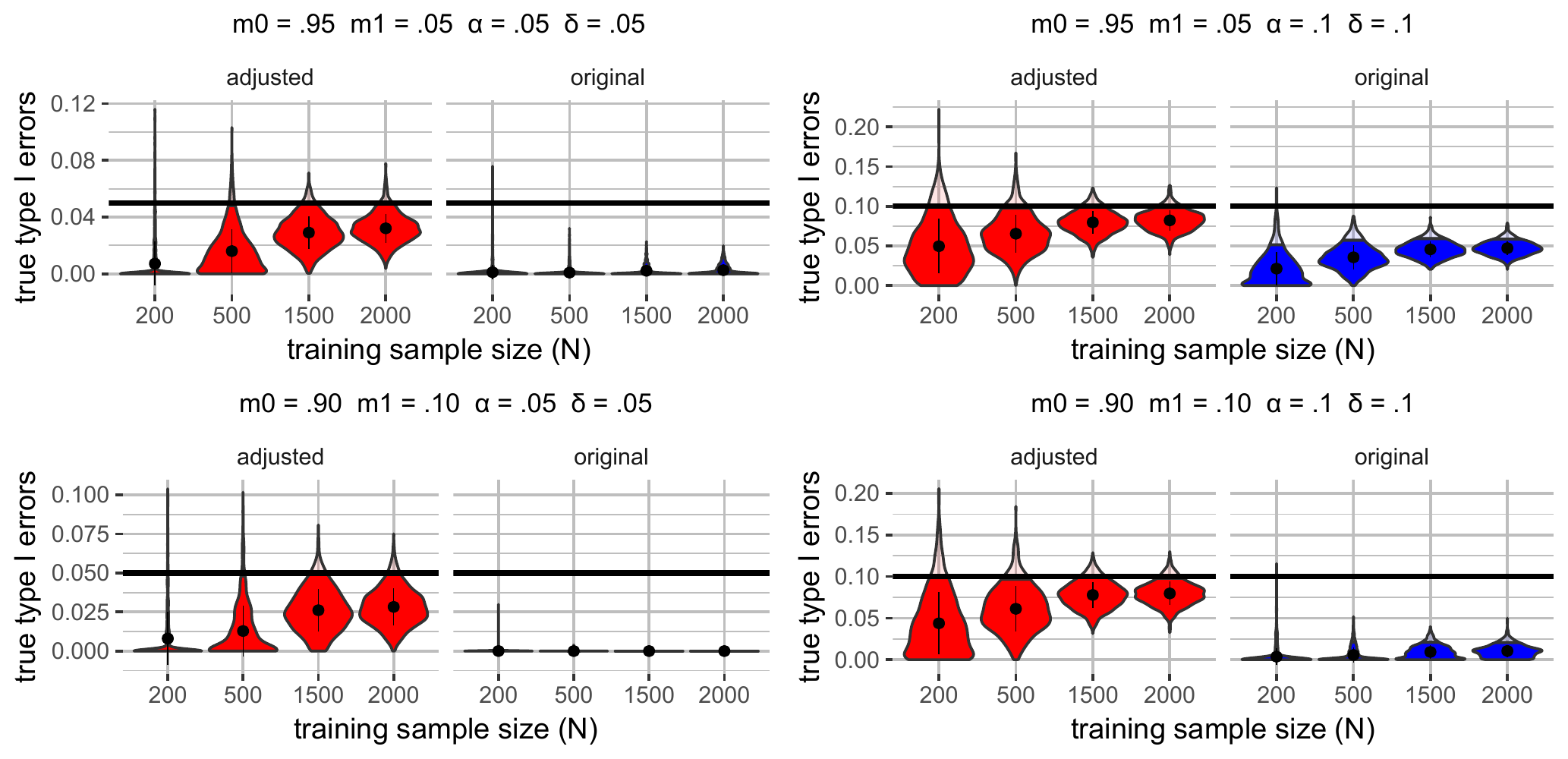}
    \caption{Violin plots for (approximate) true type I errors of Simulation \ref{sim:two_circle}. }
\end{center}
\end{figure}

\begin{figure}
\begin{center}
    \includegraphics[width = \textwidth, height = 8cm]{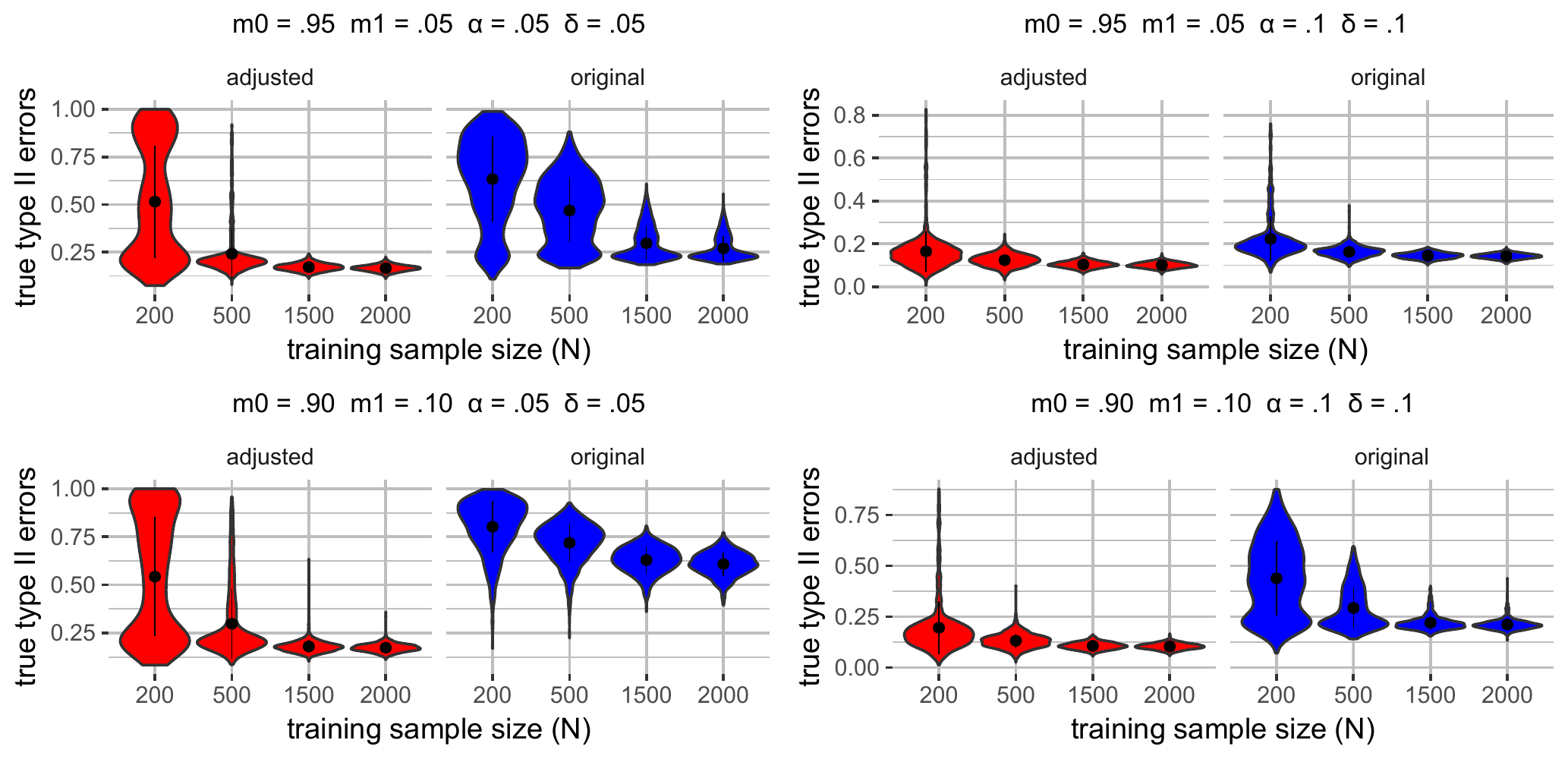}
    \caption{Violin plots for (approximate) true type II errors of Simulation \ref{sim:two_circle}. }
\end{center}
\end{figure}

\begin{figure}
\begin{center}
    \includegraphics[width = \textwidth, height = 8cm]{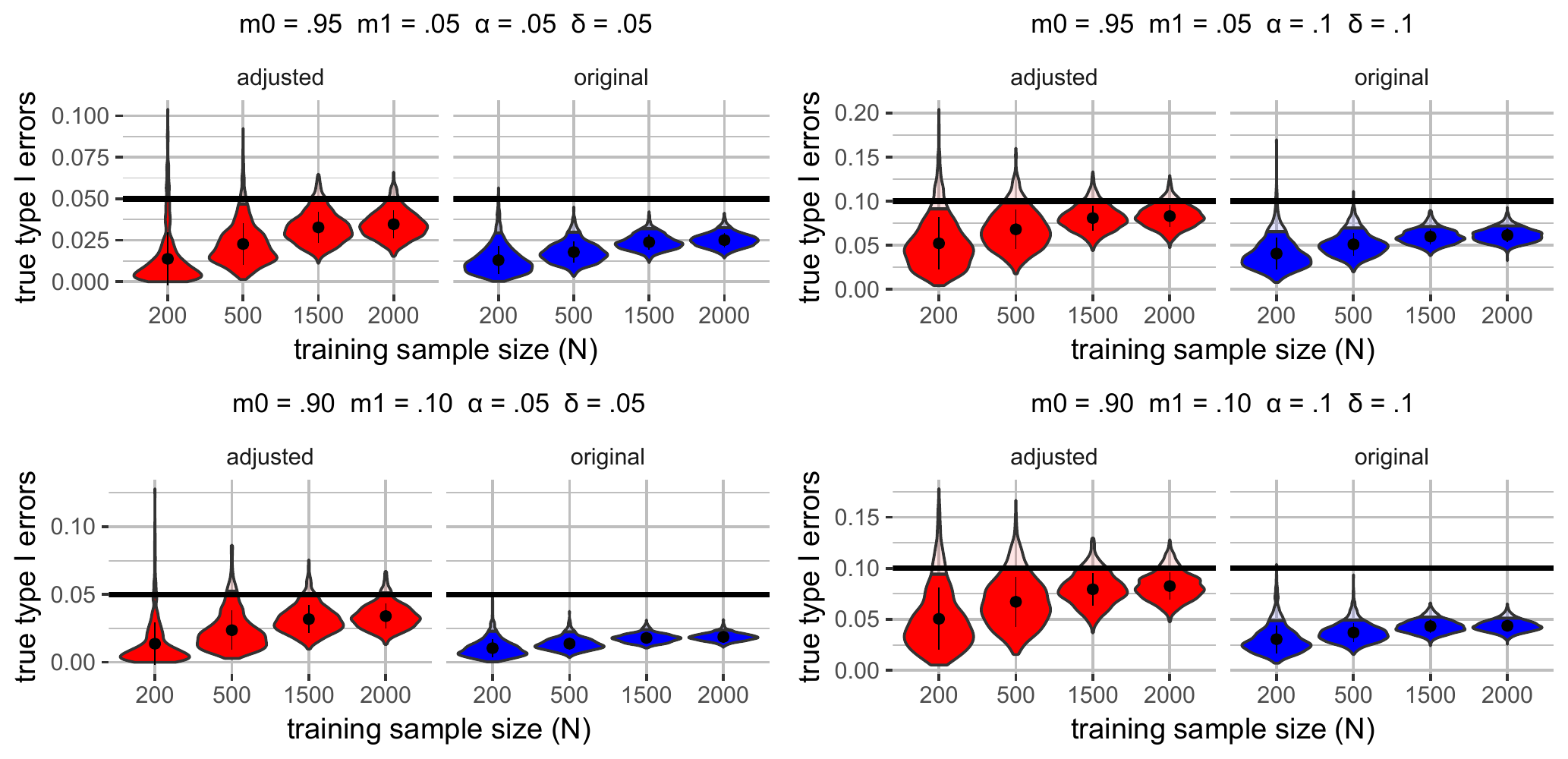}
    \caption{Violin plots for (approximate) true type I errors of Simulation \ref{sim:two_t}. }
\end{center}
\end{figure}

\begin{figure}
\begin{center}
    \includegraphics[width = \textwidth, height = 8cm]{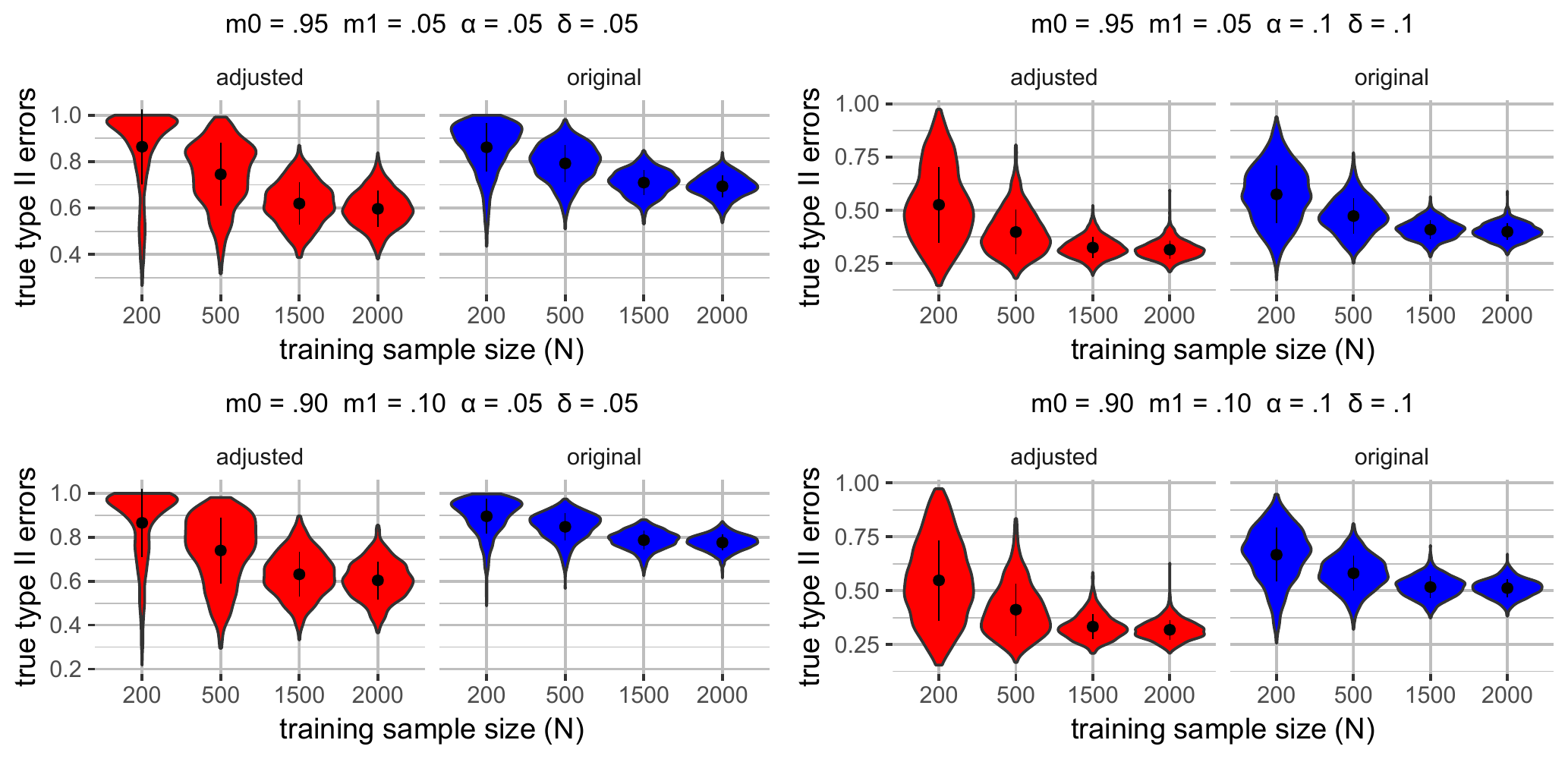}
    \caption{Violin plots for (approximate) true type II errors of Simulation \ref{sim:two_t}. } \label{fig:sim:two_t_type_II}
\end{center}
\end{figure}

%\begin{figure}
%\begin{center}
%    \includegraphics[width = \textwidth, height = 7cm]{Figures/email_spam_4.pdf}
%    \caption{Violin plots for estimated true type I errors of the email spam data. }
%\end{center}
%\end{figure}
%
%\begin{figure}
%\begin{center}
%    \includegraphics[width = \textwidth, height = 8cm]{Figures/email_2.pdf}
%    \caption{Violin plots for estimated true type II errors of th email spam data.} \label{fig:email_spam1_epsilon}
%\end{center}
%\end{figure}
%
%

\section{Extra Lemmas}
\begin{lemma}\label{lemma:corollary_for_assumption_1}
Under Assumption \ref{assumption:mixture}, for any measurable function $T:\R^d \rightarrow \R$ and arbitrary number $z \in \R$, we have
\begin{align*}
    \tilde{F}_0^T(z) = m_0F_0^T(z) + (1-m_0)F_1^T(z)\, \text{ }\text{ and }\text{ }
    \tilde{F}_1^T(z) = m_1F_0^T(z) + (1-m_1)F_1^T(z)\,.
\end{align*}
Furthermore,
\begin{align*}
    \E\tilde{X}^0 = m_0\E X^0 + (1-m_0)\E X^1\, \text{ }\text{ and }\text{ }
    \E\tilde{X}^1 = m_1\E X^0 + (1-m_1)\E X^1\,.
\end{align*}
\end{lemma}
\begin{proof}
The first two equations can be proved in the similar way. So we will only show the first equation. By Assumption \ref{assumption:mixture}, for any Borel set $A$,
\begin{align*}
    \tilde{P}_0(T^{-1}(A)) &= m_0P_0(T^{-1}(A)) + (1-m_0)P_1(T^{-1}(A))\,.
\end{align*}
Then, select $A = (-\infty, z]$ and the result follows.

Similarly, the proof of last two equations are similar in nature. So we are going to show $\E\tilde{X}^0 = m_0\E X^0 + (1-m_0)\E X^1$. Note that by Assumption \ref{assumption:mixture},
\begin{align*}
    \E\tilde{X}^0 &= \int_0^\infty(1 - \tilde{P}_0(X \leq x))dx - \int_{-\infty}^0\tilde{P}_0(X\leq x)dx\\
    &=m_0\left(\int_0^\infty(1 - P_0(X \leq x))dx - \int_{-\infty}^0P_0(X\leq x)dx\right) \\
    &+(1 - m_0)\left(\int_0^\infty(1 - P_1(X \leq x))dx - \int_{-\infty}^0P_1(X\leq x)dx\right) \\
    &= m_0\E X^0 + (1-m_0)\E X^1\,.
\end{align*}
\end{proof}
\begin{lemma}\label{lemma:existence_of_alpha_k_delta}
For any $k \in \{1,\ldots,n\}$ and $\delta\in(0,1)$, a unique $\alpha_{k,\delta}$ exists. Moreover, under Assumption \ref{assumption:sample_size}, $k_* = \min\{k\in\{1, \ldots, n\} : \alpha_{k,\delta} \leq \alpha\}$. 
\end{lemma}
\begin{proof}
Let $h_k(x) = \sum_{j=k}^n{n \choose k}x^{n-j}(1 - x)^j$ for any $k \in \{1,\ldots,n\}$. Then, one can show, for $k \leq n - 1$ and $x \in (0,1)$,
\begin{align*}
    h'_k(x) &= \sum_{j=k}^{n-1}(n-j){n \choose j}x^{n-j-1}(1-x)^j - \sum_{j=k}^nj {n \choose j}x^{n-j}(1 - x)^{j-1} \\
    &= n\sum_{i=k + 1}^{n}{n \choose i - 1}x^{n-i}(1-x)^{i-1} - n\sum_{j=k}^n{n \choose j - 1}x^{n-j}(1-x)^{j-1} \\
    &= - n{n \choose k - 1}x^{n-k}(1-x)^{k-1}\,,
\end{align*}
which is negative. Thus, $h_k(x)$ is strictly decreasing on $(0,1)$ for $k \leq n - 1$. Furthermore, $h_n(x) = (1 - x)^n$ which is also strictly decreasing on $(0,1)$. Since for any $k$, $h_k(0) = 1$ and $h_k(1) = 0$, there exists a unique $\alpha_{k,\delta}$ such that $h_k(\alpha_{k,\delta}) = \delta$.

Recall that $k_*$ is defined as the smallest $k$ such that $h_k(\alpha) \leq \delta$. Meanwhile, by monotonicity, for any $k$, the inequality $h_k(\alpha) \leq \delta$ is equivalent to $\alpha_{k,\delta} \leq \alpha$. Assumption \ref{assumption:sample_size} guarantees the existence of $k$ such that $h_k(\alpha)$. Therefore it also guarantees the existence of $k$ such that $\alpha_{k,\delta} \leq \alpha$. Then, for any $\delta$, $\{k\in\{1, \ldots, n\}:h_k(\alpha) \leq \delta\} = \{k\in\{1, \ldots, n\}:\alpha_{k,\delta} \leq \alpha\}$. Then, $k_* = \min \{k\in\{1, \ldots, n\}:\alpha_{k,\delta} \leq \alpha\}$.
\end{proof}

\begin{lemma}\label{lemma:transform_dkw}
Given a random variable $X \in \R^d$ with probability measure $P$ and a deterministic measurable function $T : \R^d \rightarrow \R$, define probability measure $P^T(B) = P(T(X) \in B)$ for any Borel set $B$. Furthermore, denote the distribution functions of $P$ and $P^T$ as $F$ and $F^T$, respectively. Let $X_1, X_2, \ldots, X_n \sim X$ be i.i.d. random variables. Moreover, let $\hat{F}^T(z) = \frac{1}{n}\sum_{j=1}^n\1\{T(X_j) \leq z\}$ for any $z \in \R$. Then, for any $\xi > 0$
\begin{align*}
    P\left(\sup_{z \in \R}\left|\hat{F}^T(z) - F^T(z)\right| > \xi\right) \leq 2e^{-2n\xi^2}\,.
\end{align*}
\end{lemma}
\begin{proof}

Note that $X_1, X_2, \ldots, X_n$ are i.i.d. random variables, then so are $T(X_1), T(X_2), \ldots, T(X_n)$. Denote $T_j = T(X_j)$, then $T_j$ has the probability measure $P^T$. Note that the Dvoretsky-Kiefer-Wolfowitz inequality says,
\begin{align*}
    P^T\left(\sup_{z \in \R}\left|\frac{1}{n}\sum_{j=1}^n\1\{T_j \leq z\} - F^T(z)\right| > \xi\right) \leq 2e^{-2n\xi^2} \,.
\end{align*}
Then, it suffices to show the left hand side of above inequality equals $P\left(\sup_{z \in \R}\left|\hat{F}^T(z) - F^T(z)\right| > \xi\right)$. Towards that, denote 
\begin{align*}
    f_n(x_1, x_2,\ldots,x_n) = \1\left\{\sup_{z \in \R}\left|\frac{1}{n}\sum_{j=1}^n\1\{T(x_j) \leq z\} - F^T(z)\right| > \xi\right\}\,,
\end{align*}
and
\begin{align*}
    f_0(t_1, t_2,\ldots,t_n) = \1\left\{\sup_{z \in \R}\left|\frac{1}{n}\sum_{j=1}^n\1\{t_j \leq z\} - F^T(z)\right| > \xi\right\} \,.
\end{align*}
By Fubini's theorem, it holds that
\begin{align*}
     P\left(\sup_{z \in \R}\left|\hat{F}^T(z) - F^T(z)\right| > \xi\right) = \E_1\E_2\ldots\E_nf_n(X_1, X_2,\ldots,X_n)\,,
\end{align*}
and
\begin{align*}
     P^T\left(\sup_{z \in \R}\left|\frac{1}{n}\sum_{j=1}^n\1\{T_j \leq z\} - F^T(z)\right| > \xi\right) = \E^T_1\E^T_2\ldots\E^T_nf_0(T_1,T_2,\ldots,T_n)\,,
\end{align*}
where $\E_j$ and $\E_j^T$ are the expectations taken with respect to $X_j$ and $T_j$ under the probability measures $P$ and $P^T$, respectively. Thus, it suffices to show
\begin{align*}
    \E_1\E_2\ldots\E_nf_n(X_1, X_2,\ldots,X_n) = \E^T_1\E^T_2\ldots\E^T_nf_0(T_1,T_2,\ldots,T_n)\,,
\end{align*}
and we will show this by induction. Denote
\begin{align*}
    f_l(x_1, x_2,\ldots,x_{l},t_{l+1},t_{l+2},\ldots,t_n) = \1\left\{\sup_{z \in \R}\left|\frac{1}{n}\left(\sum_{j=1}^{l}\1\{T(x_j) \leq z\} + \sum_{j = l+1}^{n}\1\{t_j \leq z\}\right) - F^T(z)\right| > \xi\right\}\,,
\end{align*}
for any $l \in \{1,2,\ldots,n-1\}$ and $A_{n-1}(x_1, x_2,\ldots,x_{n-1}) = \left\{t_n :  f_{n-1}(x_1, x_2,\ldots,x_{n-1},t_n) = 1\right\}$. Then, for any fixed values of $x_1,x_2,\ldots,x_{n-1}$,
\begin{align*}
    \E_n f_n(x_1,x_2,\ldots,x_{n-1}, X_n) &= P(T(X_n) \in A_{n-1}(x_1,x_2,\ldots,x_{n-1})) \\
    &=P^T(A_{n-1}(x_1,x_2,\ldots,x_{n-1})) \\
    &= \E_n^T f_{n-1}(x_1,x_2,\ldots,x_{n-1},T_n) \,,
\end{align*}
and thus,
\begin{align*}
    \E_1\E_2\ldots\E_nf_n(X_1, X_2,\ldots,X_n) = \E^T_1\E^T_2\ldots\E_{n-1}\E^T_nf_{n-1}(X_1,X_2,\ldots,X_{n-1},T_n) \,.
\end{align*} 
Now, assume that for some $l \in \{2,3,\ldots,n\}$,
\begin{multline*}
    \E_1\E_2\ldots\E_nf_n(X_1, X_2,\ldots,X_n) \\
    = \E_1\E_2\ldots\E_{l-1}\E^T_l\E^T_{l+1}\ldots\E^T_nf_{l-1}(X_1,X_2,\ldots,X_{l-1},T_l,T_{l+1},\ldots,T_n) \,.
\end{multline*}
Therefore, for any fixed values of $x_1,x_2,\ldots,x_{l-2}$, denote 
\begin{align*}
    A_{l-2}(x_1,x_2,\ldots,x_{l-2}) = \{t_{l-1} : \E^T_l\E^T_{l+1}\ldots\E^T_nf_{l-2}(x_1,x_2,\ldots,x_{t-2},t_{l-1},T_{l},\ldots,T_n)=1\} \,,
\end{align*}
we can have
\begin{multline*}
\E_{l-1}\E^T_l\E^T_{l+1}\ldots\E^T_nf_{l-1}(x_1,x_2,\ldots,x_{l-2},X_{l-1},T_l,T_{l+1},\ldots,T_n) \\ = P(T(X_{l-1}) \in  A_{l-2}(x_1,x_2,\ldots,x_{l-2})) = P^T(A_{l-2}(x_1,x_2,\ldots,x_{l-2})) \,,
\end{multline*}
and thus,
\begin{multline*}
    \E_{l-1}\E^T_l\E^T_{l+1}\ldots\E^T_nf_{l-1}(x_1,x_2,\ldots,x_{l-2},X_{l-1},T_l,T_{l+1},\ldots,T_n) \\= \E_{l-1}^T\E^T_l\ldots\E^T_nf_{l-2}(x_1,x_2,\ldots,x_{t-2},T_{l-1},\ldots,T_n)\,.
\end{multline*}
Therefore, by the assumption, we have
\begin{multline*}
    \E_1\E_2\ldots\E_nf_n(X_1, X_2,\ldots,X_n) \\
    = \E_1\E_2\ldots\E_{l-2}\E^T_{l-1}\E^T_{l}\ldots\E^T_nf_{l-2}(X_1,X_2,\ldots,X_{l-2},T_{l-1},T_{l},\ldots,T_n)\,.
\end{multline*}
We conclude the proof by induction.
\end{proof}

\section{Proofs}
\begin{proof}[Lemma \ref{lemma:general_gap}]
Let's focus on the event of the statement of Assumption \ref{assumption:separability}, whose complement holds with probability at most $\delta_1(n_{\text{b}})$. Meanwhile, by Lemma \ref{lemma:corollary_for_assumption_1}, for any $z \in \R$,
\begin{align*}
    \tilde{F}^{\hat{T}}_0\left(z\right) - \tilde{F}^{\hat{T}}_1\left(z\right) &= \left[m_0F^{\hat{T}}_0(z) + (1-m_0)F^{\hat{T}}_1(z)\right] - \left[m_1F^{\hat{T}}_0(z) + (1-m_1)F^{\hat{T}}_1(z)\right] \\
    &= (m_0 - m_1)\left(F^{\hat{T}}_0(z) - F^{\hat{T}}_1(z)\right) \,.
\end{align*}
Furthermore, for any classifier $\phi(X) = \1\{\hat{T}(X) > z\}$
\begin{align*}
    \tilde{R}_0(\phi) - R_0(\phi) &= \left(1 - \tilde{F}^{\hat{T}}_0\left(z\right)\right) - \left(1 - F^{\hat{T}}_0\left(z\right)\right) \\
    &= F^{\hat{T}}_0\left(z\right)-m_0F^{\hat{T}}_0(z) - (1-m_0)F^{\hat{T}}_1(z)\\
    &= (1-m_0)\left(F^{\hat{T}}_0(z) - F^{\hat{T}}_1(z)\right)\,,
\end{align*}
which is positive by Assumption \ref{assumption:separability}. Now, let $D(z) = \tilde{R}_0(\phi) - R_0(\phi) > 0$ and therefore $R_0(\hat{\phi}_{k_*}) > \alpha - D(t_{(k_*)})$ is equivalent to $\tilde{R}_0(\hat{\phi}_{k_*}) > \alpha$, whose probability is $\delta$ by Proposition \ref{prop:umbrella}. To this end, we have shown 
\begin{align*}
    \p\left(R_0(\hat{\phi}_{k_*}) > \alpha - D(t_{(k_*)})\right) \leq \delta + \delta_1(n_{\text{b}})\,.
\end{align*}
\end{proof}

\begin{proof}[Proof of Lemma 
\ref{lemma:alpha_k_delta}]
By Lemma \ref{lemma:existence_of_alpha_k_delta}, the set $\{k\in\{1, \ldots, n\}:\alpha_{k,\delta} \leq \alpha\}$ is non-empty. Then, the set $\{k\in\{1, \ldots, n\}:\alpha_{k,\delta} - \hat{D}^+(t_{(k)}) \leq \alpha\}$ is non-empty since $\hat{D}^+(t_{(k)})$ is non-negative. Then $k^* = \min \{k\in\{1, \ldots, n\}:\alpha_{k,\delta} - \hat{D}^+(t_{(k)}) \leq \alpha\}$ exists. Note that $k_* = \min\{k\in\{1, \ldots, n\}:\alpha_{k,\delta}  \leq \alpha\}$ by Lemma \ref{lemma:existence_of_alpha_k_delta}. Since $\{k\in\{1, \ldots, n\}:\alpha_{k,\delta}  \leq \alpha\}$ is a subset of $\{k\in\{1, \ldots, n\}:\alpha_{k,\delta} - \hat{D}^+(t_{(k)}) \leq \alpha\}$ by the non-negativeness of $\hat{D}^+$, it can be concluded that $k^* \leq k_*$.
\end{proof}

\begin{proof}[Proof of Lemma \ref{lemma:threshold_of_alg_2}]
Assumption \ref{assumption:mixture} implies $0 \leq M_\#:=\frac{1-m_0^\#}{m_0^\#-m_1^\#} \leq M = \frac{1-m_0}{m_0-m_1}$, and thus, $0 \leq \hat{D}^+_\#(c) \leq \hat{D}^+(c)$. Then, $\{k\in\{1, \ldots, n\}:\alpha_{k,\delta} \leq \alpha\}$, which is non-empty by Assumption \ref{assumption:sample_size}, is a subset of  $\{k\in\{1, \ldots, n\}:\alpha_{k,\delta} - \hat{D}^+_\#(t_{(k)}) \leq \alpha\}$. This implies $k^*_\#$ exists and is smaller than or equal to $k_*$. Furthermore, $\{k\in\{1, \ldots, n\}:\alpha_{k,\delta} - \hat{D}^+_\#(t_{(k)}) \leq \alpha\}$ is also a subset of $\{k\in\{1, \ldots, n\}:\alpha_{k,\delta} - \hat{D}^+(t_{(k)}) \leq \alpha\}$ and thus, $k^*_\# = \min\{k\in\{1, \ldots, n\}:\alpha_{k,\delta} - \hat{D}^+_\#(t_{(k)}) \leq \alpha\}$ is larger than or equal to $k^*$.
\end{proof}

\begin{proof}[Proof of Theorem \ref{thm:adjustment}]\label{proof:adjustment}
Let's focus on the event where statement of both Assumption \ref{assumption:separability} and \ref{assumption:regularity} hold, whose complement has probability less that $\delta_1(n_{\text{b}}) + \delta_2(n_{\text{b}})$. Then, let 
\begin{align*}
\mathcal{B}_{\emph{\text{e}}} = \left\{\sup_{z \in \R}\left|\hat{D}(z) - D(z)\right| \leq 2^{-1}\varepsilon\right\}\,.
\end{align*}
It follows from Lemma \ref{lemma:transform_dkw} that
\begin{align*}
    \p(\mathcal{B}_{\emph{\text{e}}}^c) &\leq \p\left(\sup_{z \in \R}\left|\hat{\tilde{F}}^{\hat{T}}_0(z) - \tilde{F}^{\hat{T}}_0(z)\right| > \frac{\varepsilon}{4}\right) + \p\left(\sup_{z \in \R}\left|\hat{\tilde{F}}^{\hat{T}}_1(z) - \tilde{F}^{\hat{T}}_1(z)\right| > \frac{\varepsilon}{4}\right) \\
    &\leq 2e^{-8^{-1}n^0_{\text{e}}M^{-2}\varepsilon^2} + 2e^{-8^{-1}n^1_{\text{e}}M^{-2}\varepsilon^2}\,.
\end{align*}
Note that since $D(z)$ is non-negative by Lemma \ref{lemma:general_gap}, $\left|\hat{D}^+(z) - D(z)\right| \leq \left|\hat{D}(z) - D(z)\right| \leq 2^{-1}\varepsilon$ on $\mathcal{B}_{\emph{\text{e}}}$. So, one can conclude that on the event $\mathcal{B}_{\emph{\text{e}}}$, $k^*$ is chosen from all $k$ such that $\alpha_{k,\delta} - D(t_{(k)}) \leq \alpha - 2^{-1}\varepsilon$. Furthermore, denote $c_k = \inf\{y:\tilde{F}^{\hat{T}}_0(y)\geq kn^{-1}\}$ and $k_0 = \min\{k\in\{1, \ldots, n\} : \alpha_{k,\delta} - D(c_k) \leq \alpha - 4^{-1}\varepsilon\}$. Note that since $\mathcal{D}_T$ is a closed interval, thus $c_k$ is well-defined. Let $\tilde{F}_{n}^{\hat{T}}$ be the empirical distribution induced by $\mathcal{T}_{\text{t}}$, i.e., for any $z \in \R$,
\begin{align*}
    \tilde{F}_{n}^{\hat{T}}(z) = \frac{1}{n}\sum_{t \in \mathcal{T}_{\emph{\text{t}}}}\1\{t \leq z\}\,.
\end{align*}
Denote $\mathcal{B}_{\emph{\text{t}}} = \{\sup_{z \in \R}|\tilde{F}_{n}^{\hat{T}}(z) - \tilde{F}_{0}^{\hat{T}}(z)| \leq 4^{-1}M^{-1}C^{-1}c\varepsilon \}$. Then, by Lemma \ref{lemma:transform_dkw}, $\p(\mathcal{B}^c_{\emph{\text{t}}}) \leq 2e^{-8^{-1}nM^{-2}C^{-2}c^2\varepsilon^2}$. So, it remains to show the probability of true type I error exceeding $\alpha$ is bounded by $\delta$ on the set $\mathcal{B}_{\emph{\text{t}}} \cap \mathcal{B}_{\emph{\text{e}}}$. Thus, till the end of the proof, we will focus on  the intersection of both sets. Note that we have $\tilde{F}_{n}^{\hat{T}}(t_{(k)}) = kn^{-1}$. Then, Taylor expansion implies
\begin{align*}
    \tilde{F}_{n}^{\hat{T}}(t_{(k)}) - \tilde{F}_{0}^{\hat{T}}(t_{(k)}) = \tilde{F}_{n}^{\hat{T}}(t_{(k)}) - \tilde{F}_{0}^{\hat{T}}(c_k) - \tilde{f}_0^{\hat{T}}(c_k^*)(t_{(k)} - c_k) = - \tilde{f}_0^{\hat{T}}(c_k^*)(t_{(k)} - c_k)\,,
\end{align*}
where $c_k^*$ is bounded by $c_k$ and $t_{(k)}$. Then the above equation implies $\left|t_{(k)}-c_k\right| \leq 4^{-1}M^{-1}C^{-1}\varepsilon$ for any $k$ according to the lower bound provided by Assumption \ref{assumption:sample_size}. Furthermore, $D(t_{(k)}) - D(c_k) = M(\tilde{f}_0^{\hat{T}}(c^{**}_k) - \tilde{f}_1^{\hat{T}}(c^{**}_k))(t_{(k)}-c_k)$ for some $c^{**}_k$ bounded by $c_k$ and $t_{(k)}$. Therefore, Assumption \ref{assumption:sample_size} implies $|D(t_{(k)}) - D(c_k)| \leq 4^{-1}\varepsilon$. Suppose $k^* = k'$, then, 
\begin{align*}
    \alpha_{k',\delta} - D(c_{k'}) \leq \alpha_{k',\delta} - D(t_{(k')}) + 4^{-1}\varepsilon \leq \alpha - 4^{-1}\varepsilon\,,
\end{align*}
and thus $k^* \geq k_0$. Furthermore, we also have
\begin{align*}
    \alpha_{k_0,\delta} - D(t_{(k_0)}) \leq \alpha_{k_0,\delta} - D(c_{k_0}) + 4^{-1}\varepsilon \leq \alpha\,.
\end{align*}
Recall that $D(t_{(k_0)}) = \tilde{R}_0(\hat{\phi}_{k_0}) - R_0(\hat{\phi}_{k_0})$.
Therefore, $R_0(\hat{\phi}_{k_0}) > \alpha$ implies $\tilde{R}_0(\hat{\phi}_{k_0}) > \alpha_{k_0,\delta}$ whose probability is bounded by $\delta$ by Proposition \ref{prop:umbrella}.
\end{proof}

\begin{proof}[Proof of Corollary \ref{cor:adjusted}]
By Lemma \ref{lemma:threshold_of_alg_2}, $k^*_\# \geq k^*$ and thus, $t_{(k^*)} \leq t_{(k^*_\#)}$. Therefore, $R_0(\hat{\phi}_{(k^*)}) \geq R_0(\hat{\phi}_{(k^*_\#)})$. Combined with Theorem \ref{thm:adjustment}, the previous result yields
\begin{multline*}
    \p\left(R_0(\hat{\phi}_{(k^*_\#)}) > \alpha\right) \leq \p\left(R_0(\hat{\phi}_{(k^*)}) > \alpha\right) \\
    \leq 
    \delta + \delta_1(n_{\text{b}}) + \delta_2(n_{\text{b}})+ 2e^{-8^{-1}nM^{-2}C^{-2}c^2\varepsilon^2} + 2e^{-8^{-1}n^0_{\text{e}}M^{-2}\varepsilon^2} + 2e^{-8^{-1}n^1_{\text{e}}M^{-2}\varepsilon^2}\,.
\end{multline*}
\end{proof}

\newpage

\end{document}